\documentclass[twoside,11pt]{article}

\usepackage{blindtext}

\usepackage{jmlr2e}

\usepackage[utf8]{inputenc} 
\usepackage[T1]{fontenc}    
\usepackage{hyperref}

\usepackage{amsmath}
\usepackage{amssymb}
\usepackage{bm}
\usepackage{comment}
\usepackage{graphicx}
\usepackage{subcaption}
\usepackage{xcolor}
\usepackage{tikz}
\usetikzlibrary{shadows}
\usepackage{url}            
\usepackage{booktabs}       
\usepackage{amsfonts}       
\usepackage{nicefrac}      
\usepackage{microtype}      
\usepackage{mathtools}
\usepackage{dsfont}
\usepackage{mathrsfs} 
\usepackage{multirow}
\usepackage{tikz, tikz-cd}
\usetikzlibrary{decorations.pathmorphing}
\usepackage{lastpage}

\newtheorem{assumption}{Assumption}

\def\barzetak{(\ft{t}(\xbtrk{k}) - \ytrk{k})}
\def\barzetal{(\ft{t}(\xbtrk{k}) - \ytrk{l})}
\def\rkhsh{\mathcal{H}}

\def\Rad{\text{Rad}}
\def\Fbase{\mathbb{V}}
\def\Fa{\mathcal{F}^{\dag}}

\def\joint{\mathcal{J}}
\def\mubase{\mu_{\text{base}}}

\def\Risk{\mathcal{R}}
\def\Lossm{L}

\newcommand{\ball}[2]{\mathcal{B}(#1; #2)}

\def\Lossb{\mathcal{L}}
\def\rhoz{\rho_{\zb}}
\def\rhoa{\rho_{a}}
\def\rhoW{\rho_{W}}
\def\rhob{\rho_{b}}
\def\Pcal{\mathcal{P}}

\newcommand{\Atmbt}[1]{A^{\mb}_{#1}}
\newcommand{\Htmbt}[1]{H^{\mb}_{#1}}

\def\falt{g}

\def\lambb{\lambdab}
\def\muhid{\nu}

\newcommand{\mumbtrit}[1]{\mu^{\mb}_{#1, \triangle}}
\newcommand{\tilmumbtrit}[1]{\tilde{\mu}^{\mb}_{#1, \triangle}}
\newcommand{\mudtrit}[1]{\mu_{#1, \blacktriangle}}
\newcommand{\mumbdtrit}[1]{\mu^{\mb}_{#1, \blacktriangle}}
\newcommand{\tilmumbdtrit}[1]{\tilde{\mu}^{\mb}_{#1, \blacktriangle}}
\newcommand{\Thetabmbtrit}[1]{\Theta^{\mb, \triangle}_{#1}}
\newcommand{\Thetabdtrit}[1]{\Theta^{\blacktriangle}_{#1}}
\newcommand{\Thetabmbdtrit}[1]{\Theta^{\mb, \blacktriangle}_{#1}}
\newcommand{\Adtrit}[1]{A^{\blacktriangle}_{#1}}
\newcommand{\Ubdtrit}[1]{\Ub^{\blacktriangle}_{#1}}

\newcommand{\Utritk}[2]{U^{\triangle}_{#1, #2}}
\newcommand{\Udtritk}[2]{U^{\blacktriangle}_{#1, #2}}

\newcommand{\Vbdtrit}[1]{\Vb^{\blacktriangle}_{#1}}

\newcommand{\Vdtritk}[2]{V^{\blacktriangle}_{#1, #2}}
\newcommand{\ftritk}[2]{f^{\triangle}_{#1, #2}}

\newcommand{\Ambtrit}[1]{A^{\mb, \triangle}_{#1}}
\newcommand{\Ambdtrit}[1]{A^{\mb, \blacktriangle}_{#1}}
\newcommand{\Ubmbtrit}[1]{\Ub^{\mb, \triangle}_{#1}}
\newcommand{\Ubmbdtrit}[1]{\Ub^{\mb, \blacktriangle}_{#1}}
\newcommand{\Umbtrit}[1]{U^{\mb, \triangle}_{#1}}
\newcommand{\Umbdtrit}[1]{U^{\mb, \blacktriangle}_{#1}}
\newcommand{\Umbtritk}[2]{U^{\mb, \triangle}_{#1, #2}}
\newcommand{\Umbdtritk}[2]{U^{\mb, \blacktriangle}_{#1, #2}}

\newcommand{\Vbmbdtrit}[1]{\Vb^{\mb, \blacktriangle}_{#1}}

\newcommand{\Vmbdtrit}[1]{V^{\mb, \blacktriangle}_{#1}}

\newcommand{\Vmbdtritk}[2]{V^{\mb, \blacktriangle}_{#1, #2}}
\newcommand{\fmbtritk}[2]{f^{\mb, \triangle}_{#1, #2}}

\newcommand{\mut}[1]{\mu_{#1}}
\newcommand{\ft}[1]{f_{#1}}
\newcommand{\xbg}[1]{\xb'_{#1}}
\def\Tpinv{T^{+}}
\def\hatTpinv{\hat{T}^{+}}
\def\Ran{\text{Ran}}
\def\Gg{G'}
\def\ng{n'}
\newcommand{\atmax}[1]{a^{#1}_{\max}}
\def\Pbb{\mathbb{P}}

\def\zerob{\boldsymbol{0}}

\def\Ggmbmax{{G'}^{m_1}_{\max}}
\def\NTK{\mathcal{K}}
\def\NTKmat{K}
\newcommand{\NTKt}[1]{\NTK_{#1}}
\newcommand{\NTKta}[1]{\NTK_{a, #1}}
\newcommand{\NTKtW}[1]{\NTK_{W, #1}}
\newcommand{\NTKmatt}[1]{\NTKmat_{#1}}

\newcommand{\NTKmatWt}[1]{\NTKmat_{W, #1}}

\def\CFM{transport map}

\newcommand{\Lossbt}[1]{\Lossb_{#1}}
\newcommand{\fmbt}[1]{f^{\mb}_{#1}}
\newcommand{\mumbt}[1]{\mu^{\mb}_{#1}}
\newcommand{\mumbkt}[2]{\mu^{\mb_{#1}}_{#2}}

\newcommand{\zetak}[2]{\big ( #2 - \ytrk{#1} \big )}

\def\rhoz{\rho_{\zb}}
\def\rhoa{\rho_{a}}
\def\rhoW{\rho_{W}}
\def\rhob{\rho_{b}}
\def\Pcal{\mathcal{P}}
\def\Hilb{\mathcal{H}}

\newcommand{\summ}[2]{\sum_{#1=1}^{#2}}
\newcommand{\Thetabt}[1]{\Theta_{#1}}

\newcommand{\Thetabtmb}[1]{\Theta^{\mb}_{#1}}
\newcommand{\At}[1]{A_{#1}}
\newcommand{\Ht}[1]{H_{#1}}

\newcommand{\Ct}[1]{C_{#1}}
\newcommand{\Lambt}[1]{\Lambdab_{#1}}

\newcommand{\Atah}[1]{A_{#1}(a, h)}
\newcommand{\Htah}[1]{H_{#1}(a, h)}
\newcommand{\Ctal}[1]{C_{#1}(a, \lambb)}
\newcommand{\Lambtal}[1]{\Lambdab_{#1}(a, \lambb)}
\def\betaA{\beta_a}
\def\betaB{\beta_b}
\newcommand{\sigbig}[1]{\sigma \big ( #1 \big )}
\newcommand{\sigtbig}[1]{\sigma_2 \big ( #1 \big )}
\newcommand{\sigobig}[1]{\sigma_1 \big ( #1 \big )}

\newcommand{\sigtpbig}[1]{{\sigma_2}' \big ( #1 \big )}

\def\falt{g}
\newcommand{\xtildetrk}[1]{\xib_{#1}}
\def\lambb{\lambdab}
\def\muhid{\nu}

\newcommand{\muhidt}[1]{\nu_{#1}}

\def\Ghalf{G^{\frac{1}{2}}}
\def\Gmhalf{(G^{+})^{\frac{1}{2}}}
\def\Gm{G^{+}}
\newcommand{\mutrit}[1]{\mu_{#1, \triangle}}
\newcommand{\Thetabtrit}[1]{\Theta^{\triangle}_{#1}}
\newcommand{\Atrit}[1]{A^{\triangle}_{#1}}
\newcommand{\Ubtrit}[1]{\Ub^{\triangle}_{#1}}

\def\Xtil{\boldsymbol{\Xi}}
\newcommand{\Projj}[1]{\text{P}_{#1}}
\newcommand{\Proj}[2]{\text{P}_{#1}({#2})}
\newcommand{\Projorth}[2]{\text{P}^{\perp}_{#1}({#2})}

\def\Gg{G'}
\def\Ggmb{{G'}^{m_1}}
\def\ng{n'}
\def\Pbb{\mathbb{P}}

\def\zerob{\boldsymbol{0}}

\def\Rbb{{\mathbb{R}}}
\def\Nbb{{\mathbb{N}}}

\def\domX{\mathcal{X}}

\def\varn{\mathscr{C}}
\newcommand{\varnp}[4]{\varn^{(#1)}_{#2, #3, #4}}
\newcommand{\varninfty}[3]{\varn_{#1, #2, #3}}

\def\xb{\boldsymbol{x}}
\def\yb{\boldsymbol{y}}
\def\zb{\boldsymbol{z}}

\def\Zb{\boldsymbol{Z}}

\def\Ub{\boldsymbol{U}}
\def\Vb{\boldsymbol{V}}
\def\thetab{\boldsymbol{\theta}}

\def\lambdab{\boldsymbol{\lambda}}
\def\Lambdab{\boldsymbol{\Lambda}}

\def\xib{\boldsymbol{\xi}}

\def\ub{\boldsymbol{u}}
\def\wb{\boldsymbol{w}}

\def\ab{\boldsymbol{a}}
\def\bb{\boldsymbol{b}}
\def\vb{\boldsymbol{v}}
\def\ub{\boldsymbol{u}}

\def\mb{\boldsymbol{m}}

\def\EE{\mathbb{E}}
\def\Con{\mathcal{C}}

\def\ev{\mathtt{e}}
\def\hatev{\hat{\ev}}
\def\evtr{\mathtt{e}_{\triangle}}

\def\hatevtr{\hat{\mathtt{e}}_{\triangle}}
\def\hatevtrpf{(\hat{\mathtt{e}}_{\triangle})_{\#}}

\def\Qfun{\mathcal{Q}}
\def\Gfun{\mathcal{G}}
\def\Gmat{G}
\def\Gmb{\Gmat^{m_1}}
\def\lambmin{\lambda_{\min}}

\def\barzetak{(\ft{t}(\xbtrk{k}) - \ytrk{k})}
\def\barzetal{(\ft{t}(\xbtrk{l}) - \ytrk{l})}
\def\rkhsh{\mathcal{H}}

\def\Rad{\text{Rad}}
\def\Fbase{\mathcal{U}}
\def\Fa{\mathcal{F}}
\newcommand{\FaUsigc}[3]{\Fa_{#1, #2, #3}}
\newcommand{\FbpUsig}[4]{\Fa_{#2, #3, #4, #1}}
\newcommand{\FbpUsigc}[5]{\Fa_{#2, #3, #4, #1, #5}}
\newcommand{\FbUsig}[3]{\Fa_{#1, #2, #3}}
\newcommand{\FbUsigc}[4]{\Fa_{#1, #2, #3, #4}}

\def\joint{\mathcal{J}}
\def\mubase{\mu_{\text{base}}}

\def\Risk{\mathcal{R}}
\def\Lossm{L}

\def\Lossb{\mathcal{L}}
\def\rhoz{\rho_{\zb}}
\def\rhoa{\rho_{a}}
\def\rhoW{\rho_{W}}
\def\rhob{\rho_{b}}
\def\Pcal{\mathcal{P}}
\def\Hilb{\mathcal{H}}
\def\Hilbtri{\mathcal{H}_{\triangle}}

\def\falt{g}

\def\xtilde{\xib}
\def\lambb{\lambdab}
\def\muhid{\nu}

\def\Ghalf{\Gmat^{\frac{1}{2}}}
\def\Gmhalf{(\Gmat^{+})^{\frac{1}{2}}}
\def\Gm{G^{+}}

\def\Tpinv{T^{+}}

\def\hatTpinv{\hat{T}^{+}}
\def\Ran{\text{Ran}}
\def\Gfunmb{\Gfun^{m_1}}
\def\Gg{G'}
\def\Gtrte{G_{\blacktriangle}}
\def\Gmbtrte{G^{m_1}_{\blacktriangle}}
\def\ng{n'}
\def\evte{\mathtt{e}_{\blacktriangle}}
\def\hatevte{\hat{\mathtt{e}}_{\blacktriangle}}
\def\Pbb{\mathbb{P}}

\def\zerob{\boldsymbol{0}}

\newcommand{\etat}[1]{\eta_t}
\newcommand{\tildeetat}[1]{\tilde{\eta}_t}
\newcommand{\ait}[1]{a_{i, #1}}
\newcommand{\bit}[1]{b_{i, #1}}
\newcommand{\hit}[1]{h_{i, #1}}

\newcommand{\hti}[2]{h_{#2, #1}}

\newcommand{\Wijt}[1]{W_{i, j, #1}}

\DeclareMathOperator{\supp}{supp}

\def\ptl{P-$3$L }
\def\ptbl{P-$\boldsymbol{3}$L }
  
\def\EE{\mathbb{E}}

\def\Risk{\mathcal{R}}
\def\Lossm{L}

\newcommand{\xbtrk}[1]{\xb_{#1}}
\newcommand{\ytrk}[1]{y_{#1}}

\newcommand{\fmu}[1]{f(\hspace{1.5pt}\cdot\hspace{1.5pt}; #1)}

\firstpageno{1}

\jmlrheading{27}{2026}{1-\pageref{LastPage}}{10/22; Revised
3/26}{4/26}{22-1232}{Zhengdao Chen, Eric Vanden-Eijnden, Joan Bruna}
\ShortHeadings{Functional-Space MF Theory of \ptl NN}{Chen, Vanden-Eijnden and Bruna}

\begin{document}

\title{A Functional-Space Mean-Field Theory of Partially-Trained Three-Layer Neural Networks}

\author{\name Zhengdao Chen\thanks{Corresponding author; at New York University when the first version of this manuscript was written.} \email zhengdao.c3@gmail.com \\
      \addr Google Research\\
      Mountain View, CA 94043
      \AND
      \name Eric Vanden-Eijnden \email eve2@cims.nyu.edu \\
      \addr Courant Institute, New York University\\
      New York, NY 10003
      \AND
      \name Joan Bruna \email bruna@cims.nyu.edu \\
      \addr Courant Institute, New York University\\
      New York, NY 10003}

\editor{Quanquan Gu}

\maketitle

\begin{abstract}
To understand the training dynamics of neural networks, prior studies have considered the mean-field limit of two-layer neural networks as the width tends to infinity, establishing theoretical guarantees for its convergence under gradient flow training as well as approximation and generalization capabilities. In this work, we study the infinite-width limit of a type of three-layer neural network where the first-layer weights are randomly sampled and untrained. To rigorously define the limiting model, we extend the mean-field theory by lifting the representation of neurons from Euclidean to functional spaces. This allows us to establish the mean-field training dynamics as a functional gradient flow with a time-varying kernel that remains positive-definite under suitable assumptions, thus proving a linear-rate convergence of its training loss. Furthermore, we define novel function spaces that contain the solutions obtained through the mean-field training dynamics and prove Rademacher complexity bounds for these spaces. Notably, our analysis applies to a range of scaling choices of the model, resulting in two distinct regimes of the mean-field limit that both exhibit feature learning through training.
\end{abstract}


\section{Introduction}
Despite involving a non-convex optimization problem, the training of neural networks (NNs) can often be solved in practice via simple algorithms such as gradient descent (GD) and its variants. To understand this, prior studies have obtained insights by examining the training dynamics of NNs when their layers are sufficiently wide. In particular, a line of works has considered two-layer ($2$L, a.k.a. one-hidden-layer or shallow) NNs in the \emph{mean-field (MF)} scaling \citep{mei2018mean, chizat2018global, rotskoff2018parameters, sirignano2020mean_lln}. On an input space $\domX \subseteq \Rbb^d$, a (scalar-valued) $2$L NN defines a function that maps any $\xb \in \domX$ to 
\begin{equation}
\label{eq:shallow}
    \frac{1}{m} \summ{i}{m} a_i \sigbig{\wb_i^{\intercal} \cdot \xb}~,
\end{equation}
where $m$ is the \emph{width} of the hidden layer, $\sigma: \Rbb \to \Rbb$ is the (nonlinear) \emph{activation function}, and the weight parameters of the first and second layers are contained in $W = [W_{i, j}]_{i \in [m], j \in [d]} = [\wb_1, ..., \wb_m]^{\intercal} \in \Rbb^{m \times d}$ and $\ab = [a_i]_{i \in [m]} \in \Rbb^m$, respectively, which are optimized during training. With the ``$1/m$'' scaling factor in \eqref{eq:shallow} inspired by the MF theory of interacting particle systems \citep{mckean1966markov, braun1977vlasov}, the model admits an integral representation and attains an \emph{infinite-width MF limit} as $m \to \infty$ in the form of 
\begin{equation}
\label{eq:integral_rep_shallow}
    \int_{\Rbb \times \Rbb^d} a \sigbig{\wb^{\intercal} \cdot \xb} \mu(da, d\wb)~,
\end{equation}
where $\mu$ is a probability measure on $\Rbb \times \Rbb^d$. The gradient flow (GF) training dynamics (i.e., GD with an infinitesimal step size) of the model's parameters corresponds to an evolution of $\mu$ under a Wasserstein GF \citep{ambrosio2008gradient} in the space of probability measures, which, in the MF limit, is known to converge to global minimizers of the loss under suitable conditions \citep{nitanda2017stochastic, chizat2018global, rotskoff2018parameters, mei2018mean, wojtowytsch2020convergence}.
Moreover, generalization and approximation guarantees can also be obtained for functions that exhibit an integral representation like \eqref{eq:integral_rep_shallow} \citep{bach2017breaking, ma2022barron}, thus establishing a solid theoretical framework for MF $2$L NNs that covers optimization, approximation, and generalization. Nonetheless, the theory is still limited in two major aspects: (1) an extension of the theory to deeper NNs is not apparent; (2) no convergence rate of the training loss is known in general settings, making it challenging to derive theoretical guarantees for finite training time (see discussions in Section~\ref{sec:related}).

In this work, we consider a type of \emph{partially-trained three-layer (P-$3$L) NN} defined as:
\begin{equation}
\label{eq:p3l}
    \begin{split}
        f^{\mb}_{\alpha}(\xb; \ab, W) 
        =&~ \frac{1}{m_2} \sum_{i=1}^{m_2} a_i \sigma_{2} \big ( h_i(\xb; W) \big )~,  \\
    \hspace{-10pt} \forall i \in [m_2] \quad : \quad h_i(\xb; W) =&~ \frac{1}{m_1^{\alpha}} \sum_{j=1}^{m_1} W_{ij} \sigma_{1} \big (\zb_j^{\intercal} \cdot \xb \big )~, \\
    \end{split}
\end{equation}
where the pair
$(m_1, m_2) =: \mb$ denotes the widths of the first and second hidden layers, $\sigma_1$ and $\sigma_2: \Rbb \to \Rbb$ are the activation functions of the first and second hidden layers, and $\alpha$ is a scaling exponent whose important role will be discussed later. The matrix $W = [W_{ij}]_{i \in [m_2], j \in [m_1]} \in \Rbb^{m_2 \times m_1}$ and the vector $\ab = [a_i]_{i \in [m_2]} \in \Rbb^{m_2}$ contain the weight parameters of the middle and output layers, respectively, and are both trained by GD. (For simplicity, we do not include bias terms in the model in the main theoretical analyses; in Appendix~\ref{app:bias} we describe a generalized version of the P-$3$L model with the bias term included in the second hidden layer.)
The input-layer parameters, $\zb_1, ..., \zb_{m_1} \in \Rbb^d$, are sampled randomly at initialization and untrained, hence the term ``partially-trained''.
For each $i \in [m_2]$, we refer to the function $h_i$ as the \emph{pre-activation function} (a.k.a. \emph{feature map}) represented by the $i$th neuron in the second hidden layer. We will often drop the dependency on $\ab$ and $W$ in $f^{\mb}$ and $h_i$ for notational simplicity.

When $\alpha = 0$, if $m_1 = m_2$ and the activation functions are $1$-homogeneous (e.g., identity or the ReLU function), \eqref{eq:p3l} under i.i.d. random initialization of the parameters is equivalent to a three-layer NN under the \textit{Neural Tangent Kernel} (\textit{NTK}; \citealt{jacot2018neural}) parameterization. In particular, as the widths tend to infinity, the model approaches a limit where the training dynamics is described by a functional GF with respect to a fixed kernel function --- the NTK. Guided by this observation, prior works have proved linear-rate convergence guarantees of the training loss \citep{du2019provably, du2019deep, allen2019convergence, zou2020gradient, oymak2020moderate, chen2021how} as well as generalization bounds \citep{arora2019fine, cao2019generalization, e2020comparative} for different kinds of NNs when the widths are sufficiently large. However, this simplified analysis arises from the large parameter scaling (in other words, a small $\alpha$), a regime where neurons in wide networks barely move during training, resulting in a lack of \emph{feature learning} \citep{chizat2019lazy, woodworth2020kernel}. For this reason, the NTK analysis does not explain the ability of NNs to perform representation learning through training, whose benefit has been shown by theoretical and empirical studies such as \cite{wei2019regularization, geiger2019disentangling, ghorbani2019limitations, ghorbani2020neural, lee2020infinite}.

Alternatively, \citet{chen2022on} consider the \ptl NN model with $\alpha = 1 / 2$ and show that when both $m_1$ and $m_2$ are large but finite, not only the model exhibits feature learning but also its training loss converges to zero at a linear rate in a regression setting. An intriguing question then is whether any well-defined infinite-width limit exists for this model. Note that if $m_1$ is fixed while $m_2$ tends to infinity, the model amounts to a $2$L NN in the MF scaling on top of a fixed embedding map, and hence an infinite-width limit can be derived analogously to that of $2$L NNs. However, this approach is no longer valid when $m_1$ also grows to infinity, and a new theory is needed to define the limiting model.

In this work, we develop a novel \emph{functional-space} MF theory for the infinite-width limit of the \ptl model with $\alpha \geq 1 / 2$. This allows us to examine the training dynamics in the infinite-width limit rigorously, which can be written as a functional GF with a \emph{time-varying} kernel, and prove a linear-rate convergence guarantee of the training loss. We see distinct behaviors of the infinite-width limit when $\alpha = 1 / 2$ versus $\alpha > 1/2$, and for both regimes, we characterize the space of functions corresponding to the MF model and prove bounds on their Rademacher complexity.

\subsection{Related works}
\label{sec:related}
\paragraph{Convergence rate of training dynamics of MF $2$L NN.} A number of studies have established the rate of convergence of the training of $2$L NN in the MF scaling, but typically only 1) under strong assumptions, 2) with modifications to the learning algorithm, or 3) for special tasks. For example, \citet{javanmard2020analysis} prove the linear-rate convergence of $2$L NN under GD under the assumption of displacement convexity, which is often too strong. \citet{hu2019mean, nitanda2022convex, chizat2022mfld} prove that mean-field Langevin dynamics on $2$L NN can converge exponentially to global minimizers if the entropic regularization is strong enough.
\citet{rotskoff2019global, wei2019regularization, nitanda2021particle, chizat2022sparse, oko2022particle} propose other modifications to the GD algorithm under which the training loss of MF $2$L NN converge at an exponential or polynomial rate.
\citet{li2020twobeyond} prove that a type of $2$L NNs trained by truncated GD in a student-teacher setup with Gaussian inputs learns the target function in a polynomial number of iterations. In contrast with these works, we will study the training of \ptl NNs in general $L_2$ regression tasks via vanilla GF without additional noise or regularization. On the side of negative results, \citet{wojtowytsch2020can} prove that if we train a $2$L NN to fit a Lipschitz target function under \emph{population} loss, the convergence rate cannot beat the curse of dimensionality. In comparison, we are interested in the empirical risk minimization (ERM) setting, where the loss function is evaluated on finitely many training data. The work of \citet{chen2022on} proves a linear-rate convergence guarantee for the $L_2$ training loss of the model defined by \eqref{eq:p3l} when $\alpha = 1 / 2$, which holds non-asymptotically when width is large. Our current work first establishes the limit of this model as $m_1$ and $m_2$ \emph{jointly} tend to infinity for both the  $\alpha = 1 / 2$ and the $\alpha > 1 / 2$ settings. Then, we prove a similar linear-rate convergence rate guarantee for the limiting model by analyzing its training dynamics as a functional GF with a time-varying kernel function.

\paragraph{MF theory of multi-layer NNs.}
The generalization of the MF limit from $2$L to deeper NNs is an intriguing and non-trivial task, and we refer the readers to \citet[Section 4.3]{sirignano2022mean_deep} for an exposition of the main challenges. Several works have made notable progress in this direction: \citet{nguyen2019mean} derives a MF limit of multi-layer NNs based on a symmetry among the neurons; by modeling the paths of weights, \citet{araujo2019mean} obtain a similar type of limit when the first and last layers are untrained; \citet{sirignano2022mean_deep} consider an alternative regime where the widths of the hidden layers tend to infinity sequentially. Notably, \citet{nguyen2020rigorous}, \citet{pham2020iclr} and \citet{fang2021modeling} derive MF limits of multi-layer NNs by defining neurons as feature maps on the input domain, opening up a perspective that inspires the function-space MF theory that we develop. However, we note some limitations of these highly interesting results: 
\begin{itemize}
    \item Even though \citet{nguyen2020rigorous, pham2020iclr, fang2021modeling} have proved global convergence results of the training dynamics, they rely on either diversity assumptions on the neurons (further discussed below) or certain re-parametrization and regularization. Moreover, no convergence rate guarantee has been derived. 
    \item There is a lack of theoretical characterization of the space of functions corresponding to these multi-layer MF NNs. Relevant to this point, \citet{e2020deepbanach} study a type of multi-layer models called \emph{neural trees} and propose a corresponding function space that generalizes the Barron space of $2$L NNs. However, the neural tree models form a much larger model class than NNs.
    \item The multi-layer NNs studied by these works all adopt the ``$1 / $width'' scaling in each layer (i.e., setting $\alpha = 1$ in \eqref{eq:p3l}), whose limitation we further discuss in the next paragraph. 
\end{itemize}
Besides fully-connected NNs, a few studies have also derived the MF limits of deep ResNets \citep{lu2020mean, ma2022barron, ding2022deepresnet}, whose behavior is nevertheless quite different from NNs with large widths in all layers. Finally, \citet{korolev2021two} develops an approximation theory of $2$L NN on Banach space inputs but does not study its training.

\paragraph{Scaling choices of wide multi-layer NNs.} Under the $\alpha = 1$ (a.k.a. the \emph{classical MF}) scaling, the neurons lose diversity if the widths tend to infinity and the parameters are sampled i.i.d. at initialization \citep{nguyen2020rigorous}, which calls for a reconsideration of how the model should be scaled based on the width \citep{luo2021phase, zhou2022empirical}. In particular, \citet{yang2021feature} propose an alternative \emph{maximum-update ($\mu P$)} scaling such that the infinite-width limit under i.i.d. initialization exhibits both feature learning \citep{ba2022high} and a diversity of the neurons' features, and it is connected to the asymptotic limit of approximate message passing \citep{bayati2011dynamics} and the dynamical mean-field theory from statistical physics \citep{bordelon2022self}. However, neither convergence guarantees nor the associated function spaces have been derived for this limit.

\subsection{Our contributions} 
\label{sec:contributions}
In this work, we derive an infinite-width MF-type limit of the \ptl NN model trained for $L_2$ regression by GF. 
By characterizing the neurons in its second hidden layer via the functions they represent on the input domain,
we define the limit as a probability measure on function spaces, and hence the name \emph{functional-space MF limit}. In particular,
\begin{itemize}
    \item We prove its existence for both the ``1/width'' scaling (corresponding to $\alpha = 1$ in \eqref{eq:p3l}) and the $\mu P$ scaling of \citet{yang2021feature} (corresponding to $\alpha = 1 / 2$) under i.i.d. initialization. A key to the proof is establishing its connection with the MF limit of a corresponding non-parametric $2$L NN on $\Rbb^n$ (where $n$ is the size of the training set), which is different between the two scaling regimes.
    \item We prove that in the MF limit, training loss can converge to zero at a linear rate via an insight that the training dynamics follows a functional GF under a \emph{time-varying} (thus allowing feature learning) kernel that remains \emph{positive definite}.
    \item We derive complexity measures that characterize the functions learned by the MF \ptl NNs and prove bounds on the Rademacher complexity of the function spaces associated with these complexity measures.
    \item We perform numerical experiments on two synthetic tasks to illustrate 1) the existence of the infinite-width limit, 2) distinct behaviors between the scaling choices of $\alpha = 1 / 2$ vs $\alpha = 1$, and 3) differences with the NTK model, $2$L NN and fully-trained $3$L NN.
\end{itemize}
In summary, the functional-space MF theory allows us to rigorously study \ptl NNs in the infinite-width limit  and establish its novel and interesting properties.

\begin{table}[]
\begin{tabular}{|ll|l|l|}
\hline
Main theoretical contributions &  & Results    & Examples of $\sigma_2$          \\ \hline
Convergence to MF limit  & & Theorem~\ref{thm:mflim} & tanh               \\ \hline
Linear-rate loss decay of MF &\hspace{-19pt}dynamics                                   & Theorem~\ref{prop:gc}                & tanh, ReLU-like*, linear \\ \hline
\multicolumn{1}{|l|}{}  & $\alpha > 1/2$ & Corollary~\ref{cor:rad_>1/2_hilb} & (leaky) ReLU, linear    \\ \cline{2-4} 
\multicolumn{1}{|l|}{\begin{tabular}[c]{@{}l@{}}Rademacher complexity bound \\ \end{tabular}} & $\alpha \geq 1/2$ & Corollary~\ref{cor:rad_1/2_hilb} & tanh, ReLU-like*, linear \\ \hline
\end{tabular}
\caption{Summary of main theoretical results and examples of the activation function $\sigma_2$ that satisfy the assumptions therein. *: ``ReLU-like'' includes ReLU, leaky ReLU, ELU, GELU, SiLU, softplus and Swish.}
\label{tab:assumptions}
\end{table}
\section{Problem setup}
In this work, we focus on the supervised $L_2$ regression setup. Let $\domX$ be the input space which is a compact subset of $\Rbb^d$, $\mathcal{Y} \subseteq [-1, 1]$ be the output space, and $\mathcal{D}$ be an underlying joint distribution on $\domX \times \mathcal{Y}$. The goal is to find a function $f$ that achieves a low \emph{population risk} $\Risk_{\mathcal{D}}(f)$, defined as
\begin{equation}
    \label{eq:risk}
    \Risk_{\mathcal{D}}(f) = \frac{1}{2} \EE_{(\xb, y) \sim \mathcal{D}} \left [ (f(\xb) - y)^2 \right ]~.
\end{equation}
In practice, instead of the true distribution $\mathcal{D}$, we are typically given a training data set consisting of $n$ i.i.d. samples from $\mathcal{D}$, $S = \{(\xbtrk{1}, \ytrk{1}), ..., (\xbtrk{n}, \ytrk{n})\} \sim \mathcal{D}^n$. Then, the strategy is to find a function that minimizes the \emph{empirical risk} as a proxy for \eqref{eq:risk}, defined as
\begin{equation}
    \widehat{\Risk}_{S}(f) = \frac{1}{2 n} \sum_{k=1}^n (f(\xbtrk{k}) - \ytrk{k})^2~.
\end{equation}
\noindent To find such a desired function, we parameterize the function by a \ptl NN and optimize its parameters using the empirical risk as the loss function:
\begin{equation}
    \Lossm(\ab, W) = \widehat{\Risk}_{S}(f^{\mb}(\cdot~; \ab, W)) = \frac{1}{2n} \sum_{k=1}^n (f^{\mb}(\xbtrk{k}; \ab, W) - \ytrk{k})^2~.
\end{equation}
with $f^{\mb}$ defined in~\eqref{eq:p3l}.
For simplicity, we do not consider any regularization term. The optimization problem is solved numerically by 
a combination of random initialization and GD training. First, we initialize each $a_i$, $W_{ij}$ and $\zb_j$ with values $\ait{0}$, $\Wijt{0}$ and $\zb_{j, 0}$ which are drawn randomly and independently from distributions $\rhoa$, $\rhoW$ and $\rhoz$, respectively. Next, for $t \geq 0$, we \emph{fix} the value of each $\zb_{j}$ while evolving each $\ait{t}$ and $\Wijt{t}$ by GD with respect to the loss function $\Lossm$. In this work, we limit our scope to studying the continuous-time version of GD, often called gradient flow (GF). Thus, if we use $\betaA \geq 0$ to represent the learning rate of $\ab_t = [a_{i, t}]_{i \in [m_2]}$ (relative to $W_t = [\Wijt{t}]_{i \in [m_2], j \in [m_1]}$) and rescale the learning rate of $\ab_t$ by $m_2$ and that of the $W_t$ by $m_2 m_1^{2\alpha-1}$ (see Remark~\ref{rmk:lr}), then each $\ait{t}$ and $\Wijt{t}$ evolve in time according to
\begin{equation}
    \frac{d}{dt}\ait{t} = - \betaA m_2 \frac{\partial \Lossm}{\partial \ait{t}}(\ab_t, W_t) = -\frac{\betaA}{n} \sum_{k=1}^n \zetak{k}{\fmbt{t}(\xbtrk{k})} \sigma_2 \big (\hit{t}(\xbtrk{k}) \big ) ~, ~\label{eq:dotati}
\end{equation}
\begin{equation}
\label{eq:dotWti}
    \begin{split}
        \frac{d}{dt} \Wijt{t} =& - m_2 m_1^{2\alpha-1} \frac{\partial \Lossm}{\partial \Wijt{t}}(\ab_t, W_t) \\ =& -\frac{\ait{t}}{n m_1^{1-\alpha}} \sum_{k=1}^n \zetak{k}{\fmbt{t}(\xbtrk{k})}  {\sigma_2}' \big (\hit{t}(\xbtrk{k}) \big ) \sigma_1 \big ( \zb_j^{\intercal} \cdot \xbtrk{k} \big )~.   
    \end{split}
\end{equation}
We write $\fmbt{t} = f^{\mb}_{\alpha}(\ab_t, W_t)$ for the output function and, for each $i \in [m_1]$, $\hit{t}(\xb) = \frac{1}{m_1^{\alpha}} \sum_{j=1}^{m_1} W_{ij} \sigma_{1} \big (\zb_j^{\intercal} \cdot \xb \big )$ for the pre-activation function of the $i$th neuron in the second hidden layer at time $t$. Induced by \eqref{eq:dotati} and \eqref{eq:dotWti}, the latter evolves in time according to
\begin{equation}
\label{eq:dothti}
    \begin{split}
        \frac{d}{dt} \hti{t}{i}(\xb) 
        =& \frac{\ait{t}}{n} \sum_{k=1}^n \zetak{k}{\fmbt{t}(\xbtrk{k})}  \sigtpbig{\hit{t}(\xbtrk{k})} \Gfunmb(\xbtrk{k}, \xb)~,
    \end{split}
\end{equation}
where given $\xb, \xb' \in \domX$, we define $\Gfunmb(\xb, \xb') = \frac{1}{m_1} \summ{j}{m_1} \sigma_1 \big ( \zb_j^{\intercal} \cdot \xb \big ) \sigma_1 \big ( \zb_j^{\intercal} \cdot \xb' \big )$. The evolution of the output function can then be expressed as
\begin{equation}
\label{eq:dotfmbt}
    \frac{d}{dt} \fmbt{t}(\xb) = \frac{1}{n} \sum_{k=1}^n \zetak{k}{\fmbt{t}(\xbtrk{k})} \NTKt{t}^{\mb}(\xbtrk{k}, \xb) ~,
\end{equation}
where for $\xb, \xb' \in \domX$, we define $\NTKt{t}^{\mb}(\xb, \xb') = \betaA \NTKt{a, t}^{\mb}(\xb, \xb') + \NTKt{W, t}^{\mb}(\xb, \xb')$, with
\begin{equation}
\label{eq:KaW}
\begin{split}
        \NTKt{a, t}^{\mb}(\xb, \xb') =&~ \frac{1}{m_2} \sum_{i=1}^{m_2} \sigma_2(\hti{t}{i}(\xb)) \sigma_2(\hti{t}{i}(\xb'))~, \\
    \NTKt{W, t}^{\mb}(\xb, \xb') =&~ \bigg ( \frac{1}{m_2} \sum_{i=1}^{m_2} (\ait{t})^2 {\sigma_2}' \big (\hti{t}{i}(\xb) \big ) {\sigma_2}'\big (\hti{t}{i}(\xb') \big ) \bigg ) \Gfunmb(\xb, \xb')~.
\end{split}
\end{equation}
$\NTKt{a, t}^{\mb}$ and $\NTKt{W, t}^{\mb}$ can be viewed as representing the contributions from the movements of $\ab_t$ and $W_t$ to the loss decay, respectively.
\begin{remark}
\label{rmk:lr}
    The choice to rescale the learning rates by $m_2$ in \eqref{eq:dotati} and $m_2 m_1^{2\alpha-1}$ in \eqref{eq:dotWti} is consistent with prior literature for both the $\alpha=1$ \citep{pham2020iclr, araujo2019mean, fang2021modeling, sirignano2022mean_deep} and the $\alpha = 1 / 2$ case \citep{yang2021feature}. With this rescaling, the magnitudes of both $\NTKt{a, t}^{\mb}$ and $\NTKt{W, t}^{\mb}$ 
    stay constant as the widths grow, suggesting a meaningful infinite-width limit that belongs to the feature learning regime. Furthermore, when $\alpha = 1 / 2$ and the activation functions are $1$-homogeneous, the training dynamics above is equivalent to that of a Xavier-initialized model up to reparameterization (see Appendix~\ref{app:xavier} as well as \citealt[Appendix C]{chen2022on}). We refer the readers to the work of \citet{yang2021feature} for further discussions on the interplay between learning rates and feature learning in deep NNs.
\end{remark}

A main question to be addressed in this work is whether the dynamics of $\fmbt{t}$ through training admits a limit as $m_1, m_2 \to \infty$, and if so, what properties of the limiting dynamics can be deduced. To this end, we establish a functional-space MF theory in the next section.

\subsection{Additional notations}
\label{sec:notations}
We will use bold, lower-case letters to denote finite-dimensional vectors, e.g., $\xb = [x_1, ..., x_d]$. For a matrix $\Gmat \in \Rbb^{n \times n}$, we use $\Gmat_{k, l}$ to denote its entry at the $(k, l)$-th position and $\Gmat_{k, :}$ to denote its $k$-th row as an $n$-dimensional vector. We define $\Gmat_{\min} = \min_{k \in [n]} \Gmat_{k, k}$ and let $\lambmin(\Gmat)$ denote the smallest eigenvalues of $\Gmat$ when it is symmetric. We write $\text{Id}_n$ for the $n \times n$ identity matrix. 

We let $\Con = \Con(\domX, \Rbb)$ denote the space of continuous functions on $\domX$ equipped with the Borel sigma-algebra (which coincides with the cylindrical sigma-algebra since $\domX$ is compact; see e.g. \citealt[Section 2]{applebaum2010cylindrical}). For any measurable space $\Omega$, we let $\Pcal(\Omega)$ denote the set of all probability measures on $\Omega$. If $T$ is a measurable map between measurable spaces $\Omega$ and $\Omega'$, we let $T_{\#}$ denote the push-forward map between $\Pcal(\Omega)$ and $\Pcal(\Omega')$.

For a Banach space $\Fbase$ and $c > 0$, we let $\ball{\Fbase}{c} \coloneqq \{ \| u \|_{\Fbase}: u \in \Fbase \}$ denote the centered ball in $\Fbase$ with radius $c$.
If $T$ is a map between spaces $\mathcal{U}$ and $\mathcal{V}$, we let $\hat{T}$ denote the map from $\Rbb \times \mathcal{U}$ to $\Rbb \times \mathcal{V}$ defined as $\hat{T}(a, u) = [a, T(u)]$, and refer to it as the \emph{lifted} version of $T$.

Suppose $p \in \mathbb{N}_+$ and $\{ \xbg{1}, ..., \xbg{p} \}$ be a subset of $\domX$. We let $\Gfun[\xbg{1}, ..., \xbg{p}]$ and $\Gfunmb[\xbg{1}, ..., \xbg{p}]$ denote the $p \times p$ matrices defined by $(\Gfun[\xbg{1}, ..., \xbg{p}])_{k, l} = \Gfun(\xbg{k}, \xbg{l})$ and $(\Gfunmb[\xbg{1}, ..., \xbg{p}])_{k, l} = \Gfunmb(\xbg{k}, \xbg{l})$ for all $k, l \in [p]$, respectively.
We define a \emph{finite-dimensional evaluation map} $\ev_{\xbg{1}, ..., \xbg{p}}: \Con \to \Rbb^k$ that maps any continuous function $f$ on $\domX$ to $[f(\xbg{1}), ..., f(\xbg{p})] \in \mathbb{R}^{p}$. Its lifted version, $\hat{\ev}_{\xbg{1}, ..., \xbg{p}}$, thus maps any $(a, f) \in \mathbb{R} \times \Con$ to $[a, \ev_{\xbg{1}, ..., \xbg{p}}(f)]^{\intercal} \in \mathbb{R} \times \mathbb{R}^{p}$. We also introduce the following shorthands for finite-dimensional evaluations with respect to the training data: $\evtr = \ev_{\xbtrk{1}, ..., \xbtrk{n}}$,  $\hatevtr = \hat{\ev}_{\xbtrk{1}, ..., \xbtrk{n}}$.

\section{Towards a Mean-Field Theory on Functional Space}
\label{sec:mf}
Suppose first that we fix $m_1$ while letting $m_2$ tend to infinity. Then, by the mean-field theory of a $2$L NN, the limit can be described via a probability measure on $\Rbb \times \Rbb^{m_1}$. Namely, if we consider the \emph{empirical measure} on the parameter space, $\frac{1}{m_2} \summ{i}{m_2} \delta_{\ait{t}}(da) \delta_{\wb_{i, t}}(d \wb) \in \Pcal(\Rbb \times \Rbb^{m_1})$, it converges weakly at each time to a MF measure as $m_2 \to \infty$, which evolves in time according to a Wasserstein GF in the space $\Pcal(\Rbb \times \Rbb^{m_1})$. 
When $m_1$ also tends to infinity, however, the space on which the probability measure is defined also grows in its dimension, and hence a more general theory is called for.

In this work, we propose a MF theory on \emph{functional space} instead of finite-dimensional parameter spaces.
We begin by observing that, regardless of $m_1$, the pre-activation of each neuron in the second hidden layer, $h_{i, t}$, always represents a continuous function on the input space $\domX$ as long as $\sigma_1$ is continuous, and its evolution as a function during training is fully given by \eqref{eq:dothti}. Thus, without directly tracking the individual weight parameters, 
we can instead track
the evolution of the following \emph{empirical measure} on the product space between $\Rbb$ and the space of continuous functions,
$\Con \coloneqq \Con(\domX, \Rbb)$:
\begin{equation}
    \mumbt{t}(da, dh) = \frac{1}{m_2} \delta_{\ait{t}}(da) \delta_{\hit{t}}(d h)~.
\end{equation}
Notice that we can write $\fmbt{t} = f(\hspace{1pt}\cdot \hspace{2pt};\mumbt{t})$, where for any $\mu \in \Pcal(\Rbb \times \Con)$, we define
\begin{equation}
    f(\xb; \mu) \coloneqq \int_{\Rbb \times \Con} a \sigtbig{h(\xb)} \mu(da, dh)~,~ \quad \forall \xb \in \domX~.
\end{equation}
In other words, the empirical measure $\mumbt{t}$ completely determines the output function. 

To see the connection between the GF dynamics of the weights in Euclidean space and the dynamics of $\mumbt{t}$ in $\Pcal(\Rbb \times \Con)$, notice that any solution to \eqref{eq:dotati} and \eqref{eq:dothti} can be written as $[a_{i, t}, h_{i, t}] = \Thetabtmb{t}(a_{i, 0}, h_{i, 0})$, where for $t \geq 0$, $\Thetabtmb{t}: \Rbb \times \Con \to \Rbb \times \Con$ is a measurable map that can be decomposed as $\Thetabtmb{t}(a, h) = [\Atmbt{t}(a, h), \Htmbt{t}(a, h)]$, with $\Atmbt{t}: \Rbb \times \Con \to \Rbb$ and $\Htmbt{t}: \Rbb \times \Con \to \Con$ satisfying the following equations:
\begin{align}
    \frac{d}{dt}\Atmbt{t}(a, h) =&~ \frac{\betaA}{n} \summ{k}{n} \zetak{k}{\fmbt{t}(\xbtrk{k})} \sigtbig{ \Htmbt{t}(a, h)(\xbtrk{k}) }~, \label{eq:dotAtmb_Con} \\
    \frac{d}{dt}\Htmbt{t}(a, h) =&~ \frac{1}{n}\Atmbt{t}(a, h) \summ{k}{n} \zetak{k}{\fmbt{t}(\xbtrk{k})} \sigtpbig{\Htmbt{t}(a, h)(\xbtrk{k})} \Gfunmb(\xbtrk{k}, \cdot)~, \label{eq:dotHtmb_Con}
\end{align}
together with the initial conditions
\begin{equation}
\label{eq:flow_init_m}
    \Atmbt{0}(a, h) = a~, \quad \Htmbt{0}(a, h) = h ~.
\end{equation}
Hence, the dynamics of $\mumbt{t}$ is given as the push-forward of $\mumbt{0}$ by the time-varying \emph{\CFM} $\Thetabtmb{t}$ on $\Rbb \times \Con$, i.e., $\mumbt{t} = (\Thetabtmb{t})_{\#} \mumbt{0}$, where $\Thetabtmb{t}$ plays an analogous role as the characteristic flow map for the transport equation describing interacting particle systems \citep{braun1977vlasov, rotskoff2022cpam}.

Based on the function-space picture, we are now able to sketch out a candidate for the MF limit. Suppose (and we will prove later) that $\mumbt{0}$ converges to a limit $\mut{0}$ as $m_1$ and $m_2$ tend to infinity. 
For $t\geq 0$, 
analogously to $\fmbt{t}$, $\mumbt{t}$ and $\Thetabtmb{t}$ in the finite-width case, we look for $\ft{t}: \domX \to \Rbb$, $\mut{t} \in \Pcal(\Con)$ and $\Thetabt{t}: \Rbb \times \Con \to \Rbb \times \Con$ which satisfy
\begin{align}
    \ft{t} =&~ \fmu{\mut{t}} \label{eq:ft_mut} \\
    \mut{t} =&~ (\Thetabt{t})_{\#} \mut{0} \label{eq:mut_Thetat} \\
    \Thetabt{t}(a, h) =&~ [\At{t}(a, h), \Ht{t}(a, h)] \label{eq:Thetat_At_Ht}
\end{align}
with $\At{t}: \Rbb \times \Con \to \Rbb$ and $\Ht{t}: \Rbb \times \Con \to \Con$ evolving in time according to
\begin{align}
    \frac{d}{dt}\Atah{t} =&~ \frac{\betaA}{n} \summ{k}{n} \zetak{k}{\ft{t}(\xbtrk{k})} \sigtbig{ \Htah{t}(\xbtrk{k}) }~, \label{eq:dotAt_Con} \\
    \frac{d}{dt}\Htah{t} =&~ \frac{1}{n}\Atah{t} \summ{k}{n} \zetak{k}{\ft{t}(\xbtrk{k})} \sigtpbig{\Htah{t}(\xbtrk{k})} \Gfun(\xbtrk{k}, \cdot)~, \label{eq:dotHt_Con}
\end{align}
together with the initial conditions
\begin{equation}
    \At{0}(a, h) = a~, \quad \Ht{0}(a, h) = h~. \label{eq:A0H0_Con}
\end{equation}
Here, we define
\begin{equation}
\label{eq:Gfun}
    \Gfun(\xb, \xb') = \lim_{m_1 \to \infty} \Gfunmb(\xb, \xb') = \int_{\Rbb^d} \sigma_1 \big ( \zb^{\intercal} \cdot \xb \big ) \sigma_1 \big ( \zb^{\intercal} \cdot \xb' \big ) \rhoz(d\zb)~,
\end{equation}
where the limit holds almost surely by the strong law of large numbers (LLN). In the literature, $\Gfun$ sometimes bears the name of the \emph{random feature kernel} or \emph{conjugate kernel} \citep{neal1996bayesian, rahimi2008random}.

\begin{remark}
The equations \eqref{eq:dotAt_Con} - \eqref{eq:A0H0_Con} define a measure-valued nonlinear transport partial differential equation (PDE) of McKean-Vlasov type \citep{mckean1966markov, braun1977vlasov}. But unlike in the MF models of interacting particle systems or two-layer NNs, in our case, the evolving object $\mut{t}$ is a probability measure on a functional space which is in principle infinite-dimensional. Hence, results in prior literature on the existence of the McKean-Vlasov MF limit and the LLN do not immediately apply. 
\end{remark}
To rigorously show that the above indeed defines the MF limit, we want to prove that: (i) At $t = 0$, $\mut{0}$ exists as the limit of $\mumbt{0}$ as $m_1, m_2 \to \infty$; (ii) There exists a tuple of $\ft{t}, \mut{t}$ and $\Thetabt{t}$ that satisfy the system of equations in  \eqref{eq:ft_mut} - \eqref{eq:A0H0_Con}; and (iii) $\forall t > 0$, $\mut{t}$ is the limit of $\mumbt{t}$ as $m_1, m_2 \to \infty$. 
As we will demonstrate, choosing $\alpha > 1/2$ versus $\alpha = 1 / 2$ results in qualitatively distinct behaviors of the MF limit. Hence, in the next three sections, we will first prove (i) and (ii) for the (easier) case of $\alpha > 1/2$ in Section~\ref{sec:4}, then (i) and (ii) for the $\alpha = 1 / 2$ case in Section~\ref{sec:5}, and finally (iii) for both cases together in Section~\ref{sec:6}.
The main proof structure for the $\alpha > 1/2$ case is illustrated in the diagram of Figure~\ref{fig:diagram_alpha>1/2}.

\begin{figure}
\begin{minipage}{1 \linewidth}
    \centering
    \begin{tikzcd}[column sep=huge, row sep=huge]
\hspace{5pt} \mumbt{0} \hspace{5pt} \arrow[d, squiggly, "\substack{\text{Finite NN GF}~ \\\\  \eqref{eq:dotAtmb_Con} - \eqref{eq:flow_init_m}}"'] \arrow{r}[swap]{\text{Lem. }\ref{lem:lln_0_1}}{m \to \infty}
& \hspace{5pt} \mut{0} \hspace{5pt} \arrow[d,squiggly,"\substack{\text{MF WGF}~ \\\\  \eqref{eq:dotAt_Con} - \eqref{eq:A0H0_Con}}"', "~\text{Lem. } \ref{lem:mf_exist_1}"] \arrow[r,shift left=1.5pt,harpoon,"(\Tpinv)_{\#}"]
& \hspace{5pt} \nu_0 \hspace{5pt}\arrow[l,shorten=0pt,shift left=1.5pt,harpoon,"T_{\#}"] \arrow[d,squiggly,"(\Psi_t)_{\#}"]  \\
\hspace{5pt} \mumbt{t} \hspace{5pt} \arrow{r}[swap]{\text{Lem.~\ref{lem:lln_t_gen}}}{m \to \infty}
& \hspace{5pt} \mut{t} \hspace{5pt} \arrow[r,shift left=1.5pt,harpoon,"(\Tpinv)_{\#}"]
& \hspace{5pt} \nu_t \hspace{5pt} \arrow[l,shorten=0pt,shift left=1.5pt,harpoon,"T_{\#}"]
\end{tikzcd}
\end{minipage}
\caption{Structure of the analysis in Sections~\ref{sec:4} and \ref{sec:6} for proving the $\alpha=1/2$ case of Theorem~\ref{thm:mflim}. To show that the GF training dynamics of P-$3$L NNs converges to a well-defined MF limit as the width $m$ goes to infinity, we (i) prove the convergence holds at $t = 0$ by the LLN (Lemma~\ref{lem:lln_0_1}); (ii) prove the MF dynamics exists as a Wasserstein GF through an equivalence with $n$-dimensional $2$L NN models (Lemma~\ref{lem:mf_exist_1}); and (iii) use a propagation-of-chaos argument to prove that the convergence as $m \to \infty$ holds at all finite time (Lemma~\ref{lem:lln_t_gen}).}
    \label{fig:diagram_alpha>1/2}
\end{figure}
\section{Neurons in Reproducing Kernel Hilbert Space: $\alpha > 1/2$}
\label{sec:4}
\subsection{MF limit at $t = 0$}
Suppose that $\alpha > 1/2$ and the parameters are randomly sampled i.i.d. at initialization. Then as $m_1 \to \infty$, by the LLN, we see that for any $i \in [m_2]$ and any $\xb \in \domX$, $\hit{0}(\xb)$ converges to zero almost surely. Thus, $\mumbt{0}$ converges to the measure $\mut{0}(da, dh) = \rhoa(da) \delta_{\zerob}(dh)$, where $\delta_{\zerob}$ is the singular measure at the constant-zero function. 
More concretely, under the following assumptions on $\sigma_1$, $\sigma_2$, $\rhoa$, $\rhoW$ and $\rhoz$, we can prove that $\mumbt{0}$ converges in $1$-Wasserstein distance to $\mut{0}$ under all finite-dimensional evaluations (as defined in Section~\ref{sec:notations}):
\begin{assumption}
\label{ass:sigma_diff}
$\sigma_1$ is continuous. $\sigma_2$ is differentiable and its derivative ${\sigma_2}'$ is bounded and Lipschitz-continuous: $\exists \mathtt{L}_{\sigma_2}, \mathtt{L}_{{\sigma_2}'} > 0$ such that $\forall u \in \Rbb$, $|{\sigma_2}'(u)| \leq \mathtt{L}_{\sigma_2}$ and ${\sigma_2}'$ is $\mathtt{L}_{{\sigma_2}'}$-Lipschitz.
\end{assumption}
\vspace{-2pt}
\begin{assumption}
\label{ass:init}
$\rhoW = \mathcal{N}(0, 1)$, $\rhoa$ is compactly-supported and symmetric with respect to zero, and
$\rhoz$ is sub-Gaussian. 
\end{assumption}
\vspace{-2pt}
\begin{lemma}[LLN at $t = 0$, $\alpha > 1/2$]
\label{lem:lln_0_1} If $\alpha > 1/2$ and Assumptions~\ref{ass:sigma_diff} and \ref{ass:init} hold,
then $\mumbt{0}$ converges weakly in all finite-dimensional evaluations to $\mut{0} = \rhoa \times \delta_{\zerob}$ almost surely.

Concretely, for any finite subset $\{ \xbg{1}, ..., \xbg{k} \} \subseteq \domX$, if we write $\mu_{t, \xbg{1}, ..., \xbg{k}} \coloneqq (\hat{\ev}_{\xbg{1}, ..., \xbg{k}})_{\#} \mut{t}$ and $\mu^{\mb}_{t, \xbg{1}, ..., \xbg{k}} \coloneqq (\hat{\ev}_{\xbg{1}, ..., \xbg{k}})_{\#} \mumbt{t}$, then 
$\mumbt{0, \xbg{1}, ..., \xbg{k}}$ converges weakly to $\mut{0, \xbg{1}, ..., \xbg{k}}$ almost surely. 
Moreover,
$\forall \epsilon > 0$, $\exists R_1, R_2 > 0$ (depending on $\epsilon$ and the set $\{ \xbg{1}, ..., \xbg{k} \}$) such that
\begin{equation}
\label{eq:exp_small_prob_wass}
    \Pbb \left ( \mathcal{W}_1 \left (\mumbt{0, \xbg{1}, ..., \xbg{k}}, \mut{0, \xbg{1}, ..., \xbg{k}} \right ) > \epsilon \right ) < O \left ( e^{-R_1 m_1} + e^{-R_2 m_2} \right )~.
\end{equation}
\end{lemma}
\noindent This lemma is proved in Appendix~\ref{app:pf_lln_0_1}. Note that the almost-sure convergence is a consequence of \eqref{eq:exp_small_prob_wass} and the Borel-Cantelli Lemma.

\subsection{$2$L NN on Hilbert Space} 
\label{sec:equiv_alpha>1/2}
For $t \geq 0$, we see that within the space of functions on $\domX$, the right-hand side of \eqref{eq:dotHt_Con} belongs to the linear span of $\{ \Gfun(\xbtrk{k}, \cdot)\}_{k \in [n]}$, which we write as $\Hilbtri \coloneqq \{ h_{\lambb}: \lambb \in \Rbb^n \}$ by defining
\begin{equation}
\label{eq:h_lamb}
    h_{\lambb} \coloneqq \summ{k}{n} \lambda_{k} \Gfun(\xbtrk{k}, \cdot)~.
\end{equation}
In fact, $\Hilbtri$ is a finite-dimensional subspace of a larger Hilbert space, $\Hilb$, which is the \emph{reproducing kernel Hilbert Space (RKHS)} on $\domX$ associated with the kernel function $\Gfun$.\footnote{We verify in Appendix~\ref{app:pf_Gpsd} that $\Gfun$ is positive semi-definite and hence a valid kernel function for RKHS.} Thus, for all $t \geq 0$, the measure $\mut{t}$ is supported on $\Rbb \times \Hilbtri \subseteq \Rbb \times \Hilb$ only, and hence $\Thetabt{t}, \At{t}, \Ht{t}$ need only to be defined on $\Rbb \times \Hilbtri$.

Interestingly, this allows us to interpret the model as a generalized MF $2$L model where the first-layer parameters belong to a Hilbert space instead of the Euclidean space $\Rbb^d$, and it could be categorized as a \emph{functional nonparametric model} \citep{ferraty2006nonparametric}. 
Moreover, its training dynamics corresponds to a Wasserstein gradient flow in $\Pcal(\Rbb \times \Hilb)$.
Specifically, similarly to the Euclidean case \citep{chizat2018global}, the Fr\`{e}chet derivative of the loss can be defined as, for $a \in \mathbb{R}$, $h \in \Hilb$,
\begin{equation}
\label{eq:frechet}
    \Lossb'_{\mu}(a, h) =  \frac{a}{n} \summ{k}{n} \zetak{k}{\fmu{\mu}(\xbtrk{k})} \sigtbig {h(\xbtrk{k})}~.
\end{equation}
Recall the reproducing property of $\Hilb$ as an RKHS: $\forall h \in \Hilb, \forall \xb \in \domX$, $h(\xb) = \langle h, \Gfun(\xb, \cdot) \rangle_{\Hilb}$, where $\langle \cdot , \cdot \rangle_{\Hilb}$ is the inner product on $\Hilb$. Hence, we derive that $\nabla_h \left ( h(\xbtrk{k}) \right ) = \Gfun(\xbtrk{k}, \cdot) \in \Hilb$,
and thus \eqref{eq:dotAt_Con} and \eqref{eq:dotHt_Con} can be equivalently written as
\begin{equation}
\label{eq:WGF_Theta}
    \frac{d}{dt} \Thetabt{t}(a, h) = - \nabla \mathcal{L}'_{\mut{t}}(\Thetabt{t}(a, h))~,
\end{equation}
where for $\mu \in \Pcal(\Rbb \times \Hilb)$, $a \in \Rbb$ and $h \in \Hilb$,
\begin{equation}
    \nabla \mathcal{L}'_{\mu}(a, h) = \begin{bmatrix}
    \frac{1}{n} \summ{k}{n} \zetak{k}{\fmu{\mu}(\xbtrk{k})} \sigtbig {h(\xbtrk{k})} \\
    \frac{a}{n} \summ{k}{n} \zetak{k}{\fmu{\mu}(\xbtrk{k})} \sigtpbig {h(\xbtrk{k})} \Gfun(\xbtrk{k}, \cdot)
    \end{bmatrix}
\end{equation}
is the gradient of the Fr\`{e}chet derivative \eqref{eq:frechet}.
Thus, \eqref{eq:WGF_Theta} together with \eqref{eq:mut_Thetat} expresses a Wasserstein gradient flow in $\Pcal(\Rbb \times \Hilb)$, which is well-defined owing to the inner product structure of the Hilbert space $\Hilb$.\footnote{This is in contrast with considering $2$L NNs with inputs from general Banach spaces, e.g., \citet{korolev2021two}. }

\subsection{Existence via equivalence to MF $2$L NN on $\Rbb^n$}
\label{sec:equiv_dimn_alpha>1/2}
To show the existence of $\mut{t}$ as defined above, we rely on its equivalence with an alternative MF $2$L model on a \emph{finite-dimensional} Euclidean space that shares an isometry with $\Hilbtri$. Let $\Gmat \in \Rbb^{n \times n}$ be the matrix defined by $\Gmat_{k, l} = \Gfun(\xbtrk{k}, \xbtrk{l})$ for $k, l \in [n]$. First,
we define a pair of linear maps $T: \Rbb^n \to \Hilbtri$ and $\Tpinv: \Con \to \Ran(\Gmat)$ as
\begin{align}
    (T (\lambb))(\xb) \coloneqq&~ h_{\Gmhalf \cdot \lambb}(\xb)~, \label{eq:T} \\
    \Tpinv (h) \coloneqq&~ \Gmhalf \cdot \evtr(h)~, \label{eq:Tpinv}
\end{align}
where $\Gmat^{+}$ denotes the Moore-Penrose pseudo-inverse of $\Gmat$ and \eqref{eq:T} uses the notation of \eqref{eq:h_lamb}.
We see that $T$ is a bijective map from $\Ran(\Gmat) \subseteq \Rbb^n$ to $\Hilbtri$, with $\Tpinv$ being its inverse map when restricted on $\Hilbtri$.
In fact, there is
\begin{equation}
\| T (\lambb) \|_{\Hilb}^2 = \lambb^{\intercal} \cdot (\Gmhalf)^{\intercal} \cdot \Gmat \cdot \Gmhalf \cdot \lambb = \| \Proj{\Ran(\Gmat)}{\lambb} \|_2^2~,
\end{equation}
and hence $T$ is an isometry between $\Ran(\Gmat)$ and $\Hilbtri$.
Moreover, for all 
$\lambb \in \Rbb^n$, it holds that $h_{\lambb}(\xb) = \summ{k}{n} \lambda_{k} \Gfun(\xbtrk{k}, \xb) = (\Tpinv (h_{\lambb}))^{\intercal} \cdot \Xtil(\xb)$,
where we define $\Xtil: \domX \to \Rbb^n$ as
\begin{equation}
    \Xtil(\xb) \coloneqq \summ{k}{n} \Gfun(\xbtrk{k}, \xb) (\Gmhalf)_{k, :}~, ~ \quad \forall \xb \in \domX~.
\end{equation}
Since the image of $\Rbb^n$ under $T$ is $\Hilbtri$, this implies that $\forall h \in \Hilbtri$,
\begin{equation}
\label{eq:hx}
    h(\xb) = (\Tpinv (h))^{\intercal} \cdot \Xtil(\xb)~.
\end{equation}
Since $\mut{t}$ is supported within $\Rbb \times \Hilbtri$, we then obtain that 
\begin{equation}
\label{eq:ft_gt_Hilb}
    \begin{split}
        \ft{t}(\xb) 
        =&~ \int_{\Rbb \times \Hilbtri} a \sigtbig{(\Tpinv (h))^{\intercal} \cdot \Xtil(\xb)} \mut{t}(da, dh)~.
    \end{split}
\end{equation}
Let us define 
\begin{equation}
\label{eq:g_alpha>1}
    g(\xtilde; \nu) \coloneqq \int_{\Rbb \times \Rbb^n} a \sigbig{ \lambb^{\intercal} \cdot \xib} \muhid(da, d \lambb)~,~\quad \forall \xib \in \Rbb^n~,
\end{equation}
for $\nu \in \Pcal(\Rbb \times \Rbb^n)$, which is equivalent to a MF $2$L NN on $\Rbb^n$. Then, \eqref{eq:ft_gt_Hilb} can be written as
\begin{equation}
\label{eq:ft_gt_Hilb_2}
    \ft{t}(\xb) = g(\Xtil(\xb); \nu_t) =: g_t(\Xtil(\xb))~,
\end{equation}
where we define $\muhidt{t} = (\hatTpinv)_{\#} \mut{t}$, with $\hatTpinv: \Rbb \times \Hilbtri \to \Rbb \times \Ran(\Gmat)$ defined by $\hatTpinv(a, h) = [a, \Tpinv(h)]$ being the lifted version of $\hatTpinv$.
Notably, $g_t$ can be viewed as a MF $2$L model on $\Rbb^n$ trained on an alternative set of training data, $\{(\xtildetrk{k}, \ytrk{k}) \}_{k \in [n]}$ with $\xtildetrk{k} = \Xtil(\xbtrk{k}) = (\Ghalf)_{k, :}$,
and the evolution of $\muhidt{t}$ follows a Wasserstein GF in $\Pcal(\Rbb \times \Rbb^n)$. In other words, in the MF limit, the \ptl NN model becomes equivalent to a (non-parametric) MF $2$L model applied to the input transformed by $\Xtil$. In particular, since $\muhidt{0} = (\hatTpinv)_{\#} \mut{0} = \rhoa \times (\delta_0)^n$, we will refer to the latter model as the \emph{dim-$n$ MF $2$L NN with $0$-initialization}.

Thus, we can show the existence of $\mut{t}$ by constructing it from $\muhidt{t}$, whose existence as a Wasserstein GF on finite-dimensional Euclidean space is well-known \citep{chizat2018global, sirignano2020mean_lln}. Specifically, $\muhidt{t}$ can be expressed as the push-forward of a time-varying transport map $\Psi_t$ on $\Rbb \times \Rbb^n$, $\muhidt{t} = (\Psi_t)_{\#}\muhidt{0}$, where for $a \in \Rbb$ and $\lambb \in \Rbb^n$, $\Psi_t(a, \lambb) = [\Ctal{t}, \Lambtal{t}]$ with $\Ct{t}: \Rbb \times \Rbb^n \to \Rbb$ and $\Lambtal{t}: \Rbb \times \Rbb^n \to \Rbb^n$, satisfying
\begin{align}
    \frac{d}{dt} \Ctal{t} =&~ \frac{1}{n} \summ{k}{n} \zetak{k}{g_t(\xtildetrk{k})} \sigtbig{ \Lambt{t}(a, \lambb)^{\intercal} \cdot \xtildetrk{k} }~, \label{eq:dotCt_Hilb}\\
    \frac{d}{dt} \Lambtal{t}
    =&~ \frac{1}{n} \Ctal{t} \summ{k}{n} \zetak{k}{g_t(\xtildetrk{k})} \sigtpbig{\Lambtal{t}^{\intercal} \cdot \xtildetrk{k}} \xtildetrk{k} ~, \label{eq:dotLambt_Hilb}
\end{align}
together with the initial conditions $\Ctal{0} = a$ and $\Lambtal{0} = \lambb$.
Then, if we define $\At{t}: \Rbb \times \Hilbtri \to \Rbb$ and $\Ht{t}: \Rbb \times \Hilbtri \to \Hilbtri$ through
\begin{align}
    \At{t}(a, h) =&~ \Ct{t}(a, \Tpinv(h))~, \label{eq:AtfromCt_Hilb} \\
        \Ht{t}(a, h) =&~ T \circ \Lambt{t}(a, \Tpinv(h))~, \label{eq:HtfromLambt_Hilb}
\end{align}
it can be verified that they satisfy \eqref{eq:dotAt_Con} - \eqref{eq:A0H0_Con}. In other words, the linear maps $T$ and $\Tpinv$ are commutative with the GF dynamics, as illustrated in Figure~\ref{fig:cd_>1/2}. We therefore conclude that: 
\begin{lemma}[Existence of MF dynamics, $\alpha > 1/2$]
\label{lem:mf_exist_1}
Suppose $\alpha > 1/2$ and Assumptions~\ref{ass:sigma_diff} and \ref{ass:init} hold. $\forall t \geq 0$ and $\forall \mut{0} \in \Pcal(\Rbb \times \Hilbtri)$, $\exists \mut{t} \in \Pcal(\Rbb \times \Hilbtri)$ and $\Thetabt{t}: \Rbb \times \Hilbtri \to \Rbb \times \Hilbtri$ such that $\mut{t} = (\Thetabt{t})_{\#} \mut{0}$, where $\mut{0} = \rhoa \times \delta_{\zerob}$ and $\Thetabt{t} = [\At{t}, \Ht{t}]$ satisfy \eqref{eq:dotAt_Con} - \eqref{eq:A0H0_Con}.
In particular,
\begin{equation}
\label{eq:f_ndim_1}
    \ft{t}(\xb) =~ \falt_{t}(\Xtil(\xb)) =~ \int_{\Rbb \times \Rbb^n} a \sigtbig{ \lambb^{\intercal} \cdot \Xtil(\xb)} \muhidt{t}(da, d \lambb)~,
\end{equation}
where $\muhidt{t} = (\hatTpinv)_{\#} \mut{t} = (\Psi_t)_{\#}\muhidt{0}$ with $\muhidt{0} = \rhoa \times (\delta_0)^n$.
\end{lemma}

\begin{figure}
\hspace{15pt}
\begin{minipage}{0.48 \linewidth}
    \caption{In the case of $\alpha > 1/2$, the linear isometric maps between $\Hilbtri$ and $\text{Ran}(\Gmat)$ ($T$ / $\Tpinv$, horizontal) commute with the flow maps induced by the GF dynamics (vertical).}
    \label{fig:cd_>1/2}
\end{minipage}
\hspace{10pt}
\begin{minipage}{0.5 \linewidth}
\vspace{-5pt}
    \centering
        \begin{tikzcd}
\hspace{16pt} h \hspace{16pt} \arrow{d}[swap]{\Theta_t(a, \vspace{2pt}\cdot\vspace{2pt})} \arrow[r,shift left=1.5pt,harpoon,"\Tpinv"]
& \hspace{16pt} \lambb \hspace{18pt}\arrow[l,shorten=0pt,shift left=1.5pt,harpoon,"T"] \arrow{d}{\Psi_t(a, \vspace{1.5pt} \cdot \vspace{1.5pt})} & \hspace{-50pt} = \Tpinv(h) \\
H_t(a, h) \arrow[r,shift left=1.5pt,harpoon,"\Tpinv"]
& \Lambdab_t(a, \lambb) \arrow[l,shift left=1.5pt,harpoon,"T"] &
\end{tikzcd}
    \end{minipage}
\end{figure}

\section{Neurons as Continuous Functions: $\alpha = 1/2$}
\label{sec:5}
\subsection{MF limit at $t = 0$}
When $\alpha = 1 / 2$, even at $t = 0$, the limiting MF measure $\mut{0}$ is no longer supported within $\Rbb \times \Hilb$. In particular, the probability measure $\frac{1}{m_2} \sum_{i=1}^{m_2} \delta_{\hit{t}}(dh)$ approaches the sample path distribution of a Gaussian process with covariance function $\Gfun$ \citep{yang2021feature}, which, almost surely, does not belong to $\Hilb$ \citep{gross1967abstract}. This suggests that we need to consider the neurons as elements from a larger functional space than $\Hilb$. 
Fortunately, we show that $\Con$ suffices as such a choice:
\begin{lemma}[LLN at $t = 0$, $\alpha = 1 / 2$]
\label{lem:lln_0_1/2} Suppose Assumptions~\ref{ass:sigma_diff} and \ref{ass:init} hold.
When $\alpha = 1 / 2$, there exists a probability measure $\mut{0} = \rhoa \times \mathcal{GP}(0, \Gfun)$ on $\Rbb \times \Con$ such that $\mumbt{0}$ converges weakly in all finite-dimensional projections to $\mut{0} \in \Pcal(\Rbb \times \Con)$ almost surely. Here, $\mathcal{GP}(0, \Gfun)$ denotes the law of the sample paths of a Gaussian process with mean zero and covariance function $\Gfun$. 
This means that for any finite subset $\{ \xbg{1}, ..., \xbg{k} \} \subseteq \domX$, 
$\mumbt{0, \xbg{1}, ..., \xbg{k}}$ converges weakly to $\mut{0, \xbg{1}, ..., \xbg{k}}$ almost surely, and moreover,
$\forall \epsilon > 0$, $\exists R_1, R_2 > 0$ such that
\eqref{eq:exp_small_prob_wass} holds similarly.
\end{lemma}
The proof is given in Appendix~\ref{app:pf_lln_0_1/2} and has two main steps: 1) proving that $\mathcal{GP}(0, \Gfun)$ is supported in $\Con$ using the Kolmogorov-Chentsov continuity theorem, and 2) the convergence of $\mumbt{0, \xbg{1}, ..., \xbg{k}}$ to $\mut{0, \xbg{1}, ..., \xbg{k}}$ in the $1$-Wasserstein distance.

\subsection{Modified MF $2$L NN on $\Rbb^n$}
\label{sec:np_=1/2}
Like in the $\alpha > 1/2$ case, we may consider the measure $\muhidt{t} = (\Psi_t)_{\#} \muhidt{0} = (\Psi_t)_{\#}(\hatTpinv)_{\#} \mut{0}$, which also satisfies a Wasserstein GF in $\Pcal(\Rbb \times \Rbb^n)$ except for having a different initial condition, $\muhidt{0} = \rhoa \times \mathcal{N}(0, \text{Id}_n)$.
Nonetheless, as $\mut{t}$ is not supported within $\Hilbtri$, the equality \eqref{eq:ft_gt_Hilb} does not necessarily hold for all $\xb \in \domX$, and hence we no longer have a full equivalence between $\ft{t}$ and $g_t(\Xtil(\cdot))$ on all of $\domX$. The commutative relation in Figure~\ref{fig:cd_>1/2} breaks down because $T$ is not injective onto $\Con$.

Meanwhile, the two functions are still equal on the training set.
The reason is that, if $h \in \Con$ satisfies $\evtr(h) \in \Ran(G)$, then by \eqref{eq:projj}, it holds for all $k \in [n]$ that $(\Projj{\Hilbtri}(h))(\xbtrk{k}) = (G \cdot \Gm \cdot \evtr(h))_k = h(\xbtrk{k})$.
Hence, \eqref{eq:ft_gt_Hilb} holds for $\xb = \xbtrk{k}$ as long as $\mutrit{t}$ has zero mass outside of $\Rbb \times \Ran(G)$, where we introduce the shorthand
$\mutrit{t} \coloneqq \mu_{t, \xbtrk{1}, ..., \xbtrk{n}}$. This condition is indeed satisfied because at $t = 0$ it is guaranteed by Lemma~\ref{lem:lln_0_1/2} (since the sample paths of $\mathcal{N}(0, G)$ belong almost surely to $\Ran(G)$), and at any $t > 0$ it continues to hold because \eqref{eq:dotHt_Con} implies that $\Htah{t} - h \in \Hilbtri$.

Moreover, we observe from \eqref{eq:dotHt_Con} that $\frac{d}{dt} \Htah{t}$ is fully determined by the values that $\Htah{t}$ takes on the training set, or equivalently,
by its projection onto $\Hilbtri$, $\Proj{\Hilbtri}{\Htah{t}}$, where we define for any function $h \in \Con$ that
\begin{equation}
\label{eq:projj}
\begin{split}
    \Proj{\Hilbtri}{h} = T \circ \Tpinv (h) = \summ{k}{n} \left ( \Gm \cdot \evtr(h) \right )_k \Gfun(\xbtrk{k}, \cdot)~, 
\end{split}
\end{equation}
Thus, to show the existence of $\mut{t}$, we can construct it from $\muhidt{t}$ by defining $\At{t}$ through \eqref{eq:AtfromCt_Hilb} and $\Ht{t}$ alternatively through
\begin{equation}
\label{eq:HtfromLambt_Con}
    \begin{split}
        \Ht{t}(a, h) =&~ h + T \left ( \Lambt{t}(a, \Tpinv(h)) - \Tpinv(h) \right ) \\
        =&~ T \circ \Lambt{t}(a, \Tpinv(h)) + \Projorth{\Hilbtri}{h}~,
    \end{split}
\end{equation}
where
we define $\Projorth{\Hilbtri}{h} = h - \Proj{\Hilbtri}{h}$ for any $h \in \Con$. 
We can then verify that \eqref{eq:dotAt_Con} - \eqref{eq:A0H0_Con} are satisfied (details in Appendix~\ref{app:pf_mf_exist_1/2}), and the relations among $\Thetabt{t}$, $\Lambt{t}$, $T$, $\Tpinv$, and $\Projj{\Hilbtri}$ are illustrated in Figure~\ref{fig:cd_=1/2}.
Analogous to Lemma~\ref{lem:mf_exist_1}, the result can be summarized as follows:
\begin{figure}
\hspace{15pt}
\begin{minipage}{0.32 \linewidth}
    \caption{In the $\alpha = 1 / 2$ case, the commutative relation in Figure~\ref{fig:cd_>1/2} holds after we project from $\Con$ to $\Hilbtri$ via $\Projj{\Hilbtri}$. Note that $\Tpinv \circ \Projj{\Hilbtri} = \Tpinv$ on $\Con$.}
    \label{fig:cd_=1/2}
\end{minipage}
\hspace{20pt}
\begin{minipage}{0.6 \linewidth}
\vspace{-10pt}
    \centering
         \begin{tikzcd}[row sep=large]
h \arrow{d}[swap]{\Theta_t(a, \hspace{2pt} \cdot \hspace{2pt})} \arrow[r,shift left=1.5pt,harpoon,"\Projj{\Hilbtri}"] & \hspace{5pt} \Projj{\Hilbtri}(h) \hspace{5pt} \arrow{d}[swap]{\Theta_t(a, \hspace{2pt} \cdot \hspace{2pt})} \arrow[r,shift left=1.5pt,harpoon,"\Tpinv"]
& \hspace{5pt} \lambb \hspace{5pt}\arrow[l,shorten=0pt,shift left=1.5pt,harpoon,"T"] \arrow{d}[swap]{\Psi_t(a, \hspace{2pt}\cdot\hspace{2pt})} & \\ 
H_t(a, h) \arrow[r,shift left=1.5pt,harpoon,"\Projj{\Hilbtri}"] & \Projj{\Hilbtri}(H_t(a, h))  \arrow[r,shift left=1.5pt,harpoon,"\Tpinv"]
&  \Lambdab_t(a, \lambb) \arrow[l,shorten=0pt,shift left=1.5pt,harpoon,"T"] &
\end{tikzcd}
    \end{minipage}
\end{figure}
\begin{lemma}[Existence of MF dynamics, $\alpha = 1 / 2$]
\label{lem:mf_exist_1/2}
Suppose $\alpha = 1 / 2$ and Assumptions~\ref{ass:sigma_diff} holds. $\forall t \geq 0$ and $\forall \mut{0} \in \Pcal(\Rbb \times \Con)$ such that $\mutrit{0}$ is supported within $\Rbb \times \Ran(\Gmat)$, $\exists \mut{t} \in \Pcal(\Rbb \times \Con)$ and $\Thetabt{t}: \Rbb \times \Con \to \Rbb \times \Con$ such that $\mut{t} = (\Thetabt{t})_{\#} \mut{0}$, where $\Thetabt{t} = [\At{t}, \Ht{t}]$ satisfies \eqref{eq:dotAt_Con} - \eqref{eq:A0H0_Con}. (In our case of interest where $\mut{0}$ is the limit of $\mumbt{0}$ as $m_1$, $m_2 \to \infty$, Lemma~\ref{lem:lln_0_1/2} guarantees that the assumption on the support of $\mutrit{0}$ holds under Assumption~\ref{ass:init}.)

In particular, 
$\forall \xb \in \domX$, it holds that
\begin{equation}
\begin{split}
    \label{eq:f_ndim}
    \ft{t}(\xb) 
    =&~ 
    \int_{\Rbb \times \Rbb^n} a \EE_{Z \sim \mathcal{N}(0, 1)} \left [ \sigtbig{\tau(\xb) Z + \summ{k}{n} (\Gmhalf \cdot \lambb)_{k} \Gfun(\xbtrk{k}, \xb)} \right ] \muhidt{t}(da, d\lambb) \\
    =&~ 
    \int_{\Rbb \times \Rbb^n} a \EE_{Z \sim \mathcal{N}(0, 1)} \left [ \sigtbig{\tau(\xb) Z + \lambb^{\intercal} \cdot \Xtil(\xb) } \right ] \muhidt{t}(da, d\lambb)~,
\end{split}
\end{equation}
where $\tau(\xb) \coloneqq \sqrt{\Gfun(\xb, \xb) - \sum_{k, l=1}^n \Gfun(\xb, \xbtrk{k}) \Gfun(\xb, \xbtrk{l}) (\Gm)_{k, l}} \geq 0$ and $\muhidt{t} = (\hatTpinv)_{\#} \mut{t} = (\Psi_t)_{\#}\muhidt{0}$ with 
$\muhidt{0} =  (\hatTpinv)_{\#} \mut{0}$.
\end{lemma}
\noindent The full proof of the lemma is given in Appendix~\ref{app:pf_mf_exist_1/2}. 

To understand \eqref{eq:f_ndim}, we see that $\forall k \in [n]$, $\tau(\xbtrk{k}) = 0$, and hence \eqref{eq:f_ndim} agrees with \eqref{eq:f_ndim_1} on the training data. Outside the training set, the term $\tau(\xb) Z$ in \eqref{eq:f_ndim} is a consequence of the additional term $\Projorth{\Hilbtri}{h}$ in \eqref{eq:HtfromLambt_Con} compared with \eqref{eq:HtfromLambt_Hilb}. Heuristically, it can be interpreted as adding an \emph{input-dependent smoothing} to the activation of each neuron. 

\section{Mean-Field Limit}
\label{sec:6}
So far, we have shown the existence of $\mut{t}$ as a dynamics in the space of $\Pcal(\Rbb \times \Con)$, which can be restricted to $\Pcal(\Rbb \times \Hilbtri)$ when $\alpha > 1/2$. Next, we examine the convergence of $\fmbt{t}$ to $\ft{t}$ as $m_1, m_2$ tend to infinity. While Lemmas~\ref{lem:lln_0_1} and \ref{lem:lln_0_1/2} establish the convergence at $t = 0$ under random initialization, for $t > 0$, since the training dynamics introduce nonlinear interactions among the neurons, further arguments are necessary. 

Classical studies of interacting particle systems rely on a propagation-of-chaos argument to bound the deviation between the finite-size system and its infinite-width limit through evolution using the Lipschitz-continuity of the evolution map \citep{braun1977vlasov}. This approach has been adapted for showing that $2$L NNs converge to the MF limit when the widths tend to infinity \citep{mei2018mean, rotskoff2018parameters, rotskoff2022cpam, chizat2018global, sirignano2020mean_lln}. Here, we want to adopt a similar approach but face the challenge that the probability measures are defined on functional spaces rather than Euclidean space. 

To circumvent this challenge, we again leverage the fact that the system can be determined from a transport dynamics of probability measures on finite-dimensional space, $\muhidt{t} = (\hatTpinv)_{\#} \mut{t} \in \Pcal(\Hilbtri)$, or equivalently, $\mutrit{t} \in \Pcal(\Rbb^n)$. First, we will prove the convergence of $\mumbtrit{t}$ to $\mutrit{t}$ as $m_1$ and $m_2$ tend to infinity. Specifically, if the activation function $\sigma_2$ is additionally assumed to be bounded, we can prove an upper bound on their $1$-Wasserstein distance at time $t$ based on their $1$-Wasserstein distance at time $0$, which has been controlled by Lemmas~\ref{lem:lln_0_1} and \ref{lem:lln_0_1/2}. Compared to propagation-of-chaos results for $2$L NNs, an additional complication is caused by the finiteness of $m_1$, which introduces an extra term involving the deviation of $\Gmb \coloneqq \Gfunmb[\xbtrk{1}, ..., \xbtrk{n}]$ from $G$. The results are as follows:
\begin{assumption}
\label{ass:sigma_bdd}
$\sigma_2$ is bounded. Specifically, $\exists \mathtt{M}_{\sigma_2} > 0$ such that $\forall u \in \Rbb, |\sigma_2(u)| \leq \mathtt{M}_{\sigma_2}$.
\end{assumption}
\begin{lemma}[Propagation-of-chaos, I]
\label{lem:lln_t_tr}
Suppose Assumptions~\ref{ass:sigma_diff} and \ref{ass:sigma_bdd} hold and $\mut{t}$ exists.
Then for any $t \geq 0$, $\exists C_1(t) > 0$ such that
\begin{equation}
\label{eq:poc1}
    \begin{split}
        \mathcal{W}_1(\mutrit{t}, \mumbtrit{t}) \leq C_1(t) \left ( \mathcal{W}_1(\mutrit{0}, \mumbtrit{0}) + \| G - \Gmb \|_{2} \right )~.
    \end{split}
\end{equation}
\end{lemma}
\noindent This lemma is proved in Appendix~\ref{app:pf_lln_t_tr}. Next, we can extend the bound to \emph{any} finite-dimensional evaluations (not just on the training set) of $\mumbt{t}$ and $\mut{t}$. Let $\{ \xbg{1}, ..., \xbg{\ng} \}$ be any finite subset of $\domX$. We define the shorthands $\mudtrit{t} \coloneqq (\hat{\ev}_{\xbg{1}, ..., \xbg{\ng}})_{\#}\mut{t}$, $\mumbdtrit{t} \coloneqq (\hat{\ev}_{\xbg{1}, ..., \xbg{\ng}})_{\#} \mumbt{t}$, $\Gtrte \coloneqq \Gfun[\xbtrk{1}, ..., \xbtrk{n}, \xbg{1}, ..., \xbg{\ng}]$ and $\Gmbtrte \coloneqq \Gfunmb[\xbtrk{1}, ..., \xbtrk{n}, \xbg{1}, ..., \xbg{\ng}]$. Then, we have:
\begin{lemma}[Propagation-of-chaos, II]
\label{lem:lln_t_gen}
Suppose Assumptions~\ref{ass:sigma_diff} and \ref{ass:sigma_bdd} hold and $\mut{t}$ exists.
For any $t \geq 0$, $\exists C_2(t) > 0$ such that
\begin{equation}
    \begin{split}
        \mathcal{W}_1(\mudtrit{t}, \mumbdtrit{t}) 
        \leq~ C_2(t) \left ( \mathcal{W}_1(\mudtrit{0}, \mumbdtrit{0}) + \mathcal{W}_1(\mutrit{0}, \mumbtrit{0}) + \| \Gtrte - \Gmbtrte \|_{2} \right )~.
    \end{split}
\end{equation}
\end{lemma}
\noindent This lemma is proved in Appendix~\ref{app:pf_lln_t_gen}. As in standard propagation-of-chaos results based on Gr\"onwall's inequality, the constants $C_1(t)$ and $C_2(t)$ in the above lemmas grow exponentially in $t$. It remains open whether one can reduce the dependence on $t$ by further exploiting the properties of the MF dynamics, perhaps by leveraging ideas considered by \citet{chen2020dynamical, pham2021fluctuation, glasgow2025propagation}.

Thus, combining Lemmas~\ref{lem:lln_0_1}, \ref{lem:lln_0_1/2}, and \ref{lem:lln_t_gen} as well as concentration bounds of $\Gfunmb$ (Lemma~\ref{lem:Gm_dev}), we see that $\forall t \geq 0, \epsilon > 0$, $\exists R_3(t), R_4(t) > 0$ (dependent on $\epsilon$, $\{ \xbtrk{1}, ..., \xbtrk{n}\}$ and $\{ \xbg{1}, ..., \xbg{n'}\}$) such that
\begin{equation}
\label{eq:exp_small_prob_wass_t}
    \Pbb \left ( \mathcal{W}_1 \left (\mudtrit{t}, \mumbdtrit{t} \right ) > \epsilon \right ) < O(e^{-R_3(t) m_1} + e^{-R_4(t) m_2})~,
\end{equation}
which implies an almost sure convergence through the Borel-Cantelli lemma (analogous to the last step in the proof of Lemma~\ref{lem:lln_0_1/2} in Appendix~\ref{app:pf_lln_0_1/2}).
Finally, choosing $n' = 1$ and $\xb'_1 = \xb$, we are able to prove our main result on the MF limit:
\begin{theorem}[MF limit]
\label{thm:mflim}
Suppose Assumptions~\ref{ass:sigma_diff}, \ref{ass:init} and \ref{ass:sigma_bdd} hold. Then $\mut{t} = (\Thetabt{t})_{\#} \mut{0}$ exists, where $\Thetabt{t} = [\At{t}, \Ht{t}]$ satisfy \eqref{eq:dotAt_Con} - \eqref{eq:A0H0_Con}, and $\mut{0} = \rhoa \times \chi$, where $\chi = \delta_{\zerob}$ if $\alpha > 1/2$ or $\mathcal{GP}(0, \Gfun)$ if $\alpha = 1 / 2$. 
Moreover,
$\forall \xb \in \domX, t \geq 0$, $\fmbt{t}(\xb)$ converges almost surely as $m_1, m_2 \to \infty$ to $\ft{t}(\xb)$, which can be characterized by \eqref{eq:f_ndim_1} and \eqref{eq:f_ndim} when $\alpha > 1/2$ and $\alpha = 1 / 2$, respectively.
\end{theorem}
\begin{remark}
Because the bounds obtained in Lemmas~\ref{lem:lln_t_tr} and \ref{lem:lln_t_gen} are non-asymptotic in $m_1$ and $m_2$, the MF limit established in Theorem~\ref{thm:mflim} does not require $m_1$ and $m_2$ to tend to infinity either in a particular order or under certain asymptotic relations.
\end{remark}

\section{Convergence Guarantee of the Mean-Field Dynamics}
\label{sec:gc}
In this section, we further investigate the evolution of $\ft{t}$ over time as a function on $\domX$.
At initialization ($t=0$), 
there is $\ft{0}(\xb) = \left ( \int_{\mathbb{R}} a \rhoa(da) \right ) \left ( \int_{\Con} \sigma(h(\xb)) \chi(dh) \right )$ for any $\xb \in \domX$. Hence,
if $\rhoa$ is symmetric with respect to zero (Assumption~\ref{ass:init}), 
then for either choice of $\alpha$, $\ft{0}$ is the zero function on $\domX$. 

For $t \geq 0$, the evolution of the measure $\mut{t}$ induces a dynamics of $\ft{t}$ that can be expressed as a \emph{functional gradient flow}:
\begin{equation}
\label{eq:dotft}
        \frac{d}{dt}\ft{t}(\xb)
        = \frac{1}{n} \sum_{k=1}^n \zetak{k}{\ft{t}(\xbtrk{k})} \NTKt{t}(\xbtrk{k}, \xb) ~,
\end{equation}
where $\forall \xb, \xb' \in \domX$, we define the kernel function $\NTKt{t}(\xb, \xb') = \betaA \NTKta{t}(\xb, \xb') + \NTKtW{t}(\xb, \xb')$, with
\begin{align*}
    \NTKta{t}(\xb, \xb') 
        =& \int_{\mathbb{R} \times \Con} \sigma_1 \left ( h(\xb) \right ) \sigma_1 \left ( h(\xb') \right ) \mut{t}(da, dh)~,\\
    \NTKtW{t}(\xb, \xb') 
        =&~ \Qfun_t(\xb, \xb') ~ \Gfun(\xb, \xb')~, \\
        \text{ where } \quad \Qfun_t(\xb, \xb') = & \int_{\mathbb{R} \times \Con} a^2 {\sigma_2}' \left ( h(\xb) \right ) {\sigma_2}' \left ( h(\xb') \right ) \mut{t}(da, dh)~.
\end{align*}
The dynamics \eqref{eq:dotft} can be viewed as an infinite-width analog of \eqref{eq:dotfmbt} as $m_1, m_2 \to \infty$, which is now well-defined through the functional-space MF theory developed in the previous section.

In the NTK regime (equivalent to $\alpha = 0$ if $\sigma$ is $1$-homogeneous), the corresponding kernel function is \emph{static} during training, which leads to a linearized training dynamics
and excludes the possibility of feature learning (\citealt{chizat2019lazy}; also see numerical results in Section~\ref{sec:exp}).
In contrast, when $\alpha \geq 1 / 2$, the kernel function $\NTKt{t}$ changes over time as $\mut{t}$ evolves during training. 
Inevitably, this complicates the convergence analyses compared to the NTK model, but we will show below that a linear-rate convergence guarantee can still be derived through a fine-grained analysis of the kernel function.

\subsection{Linear-rate convergence with a time-varying kernel}
To analyze the convergence rate of the training loss, we define a $n \times n$ kernel matrix $\NTKmatt{t}$ associated with the kernel function $\NTKt{t}$ by $(\NTKmatt{t})_{k, l} \coloneqq \NTKt{t}(\xbtrk{k}, \xbtrk{l})$ for $k, l \in [n]$. Similarly, we define $n \times n$ matrices $\NTKmatWt{t}$ and $Q_t$ associated with $\NTKtW{t}$ and $\Qfun_t$. It is easy to see that these matrices are all symmetric and positive semi-definite. Then, from \eqref{eq:dotft}, the decay rate of the training loss can be computed as
\begin{equation}
\begin{split}
    \frac{d}{dt} \Lossbt{t} =&~ - \frac{1}{n^2} \sum_{k, l=1}^n \barzetak \barzetal (\NTKmatt{t})_{k, l} ~.
\end{split}
\end{equation}
Using the definition $\Lossbt{t}$, we obtain the following bound by focusing on the contributions from the movement of $W_t$ (and hence it does not depend on the learning rate of the last layer, $\betaA$):
\begin{equation}
\label{eq:rate_lower_bdd}
     \frac{d}{dt} \Lossbt{t} \leq - \frac{2}{n^2} \lambda_{\min}(\NTKmatt{t}) \Lossbt{t} ~ \leq - \frac{2}{n^2} \lambda_{\min}(\NTKmatWt{t}) \Lossbt{t} ~.
\end{equation}

\noindent Thus, if $\lambda_{\min}(\NTKmatWt{t})$ has a positive lower bound throughout training, \eqref{eq:rate_lower_bdd} establishes a  Polyak-\L ojasiewicz (PL) condition \citep{Polyak1963GradientMF, lojasiewicz1963topological}, through which one can prove that $\Lossbt{t}$ converges to zero at a linear rate. Under the NTK limit mentioned above, since the kernel remains fixed during training, it suffices to prove that the kernel matrix is positive definite at initialization, which indeed holds in various settings \citep{du2019deep, du2019provably}.
When $\alpha \geq 1 / 2$, the kernel moves substantially during training, and thus a uniform-in-time lower bound on $\lambda_{\min}(\NTKmatWt{t})$ is much less trivial. Nonetheless, we notice that the matrix $\NTKmatWt{t}$ can be written as the Hadamard (i.e. entry-wise) product of two matrices that are both positive semi-definite, $Q_t$ and $G$. Thus, to show the positive-definiteness of $\NTKmatWt{t}$, we can take advantage of Oppenheim's inequality \citep{markham1986oppenheim} to write
\begin{equation}
\label{eq:oppenheim}
    \det(\NTKmatWt{t}) \geq \Big ( \prod_{k=1}^n (Q_t)_{k, k} \Big ) \det(G)~.
\end{equation}
On one hand, $G$ is independent of $t$ and often guaranteed to be positive \emph{definite}, such as under the following assumptions on $\rhoz$, $\sigma_1$ and the training data \citep{du2019deep, du2019provably}:
\begin{assumption}
\label{ass:relu_anp}
$\rhoz$ is $d$-dimensional standard Gaussian and $\sigma_1$ is either 1) analytic and non-polynomial or 2) the ReLU function.
\end{assumption}
\begin{assumption}
\label{ass:not_aligned}
The training set $\{ \xbtrk{1}, ..., \xbtrk{n} \}$ does not contain any pair of aligned vectors.
\end{assumption}
Thus, $\NTKmatWt{t}$ is guaranteed to be positive definite as long as the diagonal entries of $Q_t$ have a positive lower bound uniformly throughout training. For the latter to be established, we require that the activation $\sigma_2$ satisfies:
\begin{assumption}
\label{ass:itvl}
There exists an open interval $I = (I_l, I_r) \subseteq \mathbb{R}$ on which $\sigma_2$ is differentiable and $|{\sigma_2}'|$ is lower-bounded by some $\mathtt{K}_{\sigma_2} > 0$. If $\alpha > 1/2$, we need to further assume that $0 \in I$.
\end{assumption}
\noindent The first part of this assumption is satisfied by most activation functions in practice, such as ReLU and $\tanh$. The additional assumption $0 \in I$ is needed for the case $\alpha > 1/2$ due to the bias term in the second hidden layer being omitted. If the bias term is added and randomly sampled from $\rho_b \in \Pcal(\Rbb)$ at initialization, then this assumption can be replaced by $\rho_b(I) > 0$.

Assumption~\ref{ass:itvl} is needed to ensure that, heuristically, the back-propagated gradients are not fully vanishing due to the multiplicative factors involving the terms $\{ \sigma'_2(h(\xbtrk{k})) \}_{k \in [n]}$. To show that this property holds true throughout training, we need a fine-grained analysis of the neurons' dynamics, specifically, bounding the speed of the movement of the second-hidden-layer neurons by the decay rate of the training loss (Lemma~\ref{lem:ddtAH_int_upper}).

Together, we prove that the training loss converges to zero at a linear rate \emph{without} requiring the kernel to be frozen during training:
\begin{theorem}[Linear-rate convergence of training loss]
\label{prop:gc}
Suppose that Assumptions~\ref{ass:sigma_diff}, \ref{ass:init}, \ref{ass:relu_anp}, \ref{ass:not_aligned} and \ref{ass:itvl} hold. Then there exist $\hat{a}$ and $r > 0$ such that if 
$\rhoa \left ( [\hat{a}, \infty) \right ) > 0$, 
then it holds that $\forall t \geq 0$,
\begin{equation}
    \Lossbt{t} \leq \Lossbt{0} e^{-r \hat{a}^2 \lambmin(G) t}~,
\end{equation}
where $r$ depends on $\rhoa \left ( [\hat{a}, \infty) \right )$, $I_r - I_l$, $\| \Gfun \|_{\infty}$, $G_{\min}$, $\mathtt{M}_{\sigma_2}$, $\mathtt{L}_{\sigma_2}$ and $\mathtt{K}_{\sigma_2}$.
\end{theorem}
\noindent This theorem is proved in Appendix~\ref{app:pf_gc}. Note that the condition on $\rhoa$ is satisfied if, for example, it is a uniform distribution on a wide-enough interval. While a non-asymptotic version of this result for the case $\alpha = 1 / 2$ and $\betaA=0$ has been given in \citet{chen2022on}, the analysis here provides asymptotic results for the broader settings and novel insights via the functional GF formulation.

\section{Complexity Measures and Function Spaces}
For $2$L NNs in the MF scaling, prior works have characterized the functions they represent via the Barron norm (a.k.a. variation norm or $\mathcal{R}$-norm) as a complexity measure of functions, which in turn defines the Barron space as the space of functions with finite Barron norms \citep{bach2017breaking, ma2018priori}. In this section, we will similarly introduce function spaces corresponding to functions learned by the MF training dynamics of P-$3$L NNs through new complexity measures of functions. We will see that when $\alpha > 1/2$, the function space can be viewed as a straightforward extension of the Barron space; whereas to incorporate the $\alpha = 1 / 2$ setting, we will define a novel complexity measure based on Wasserstein-type distances between distributions of functions.

For simplicity of presentation, we will concentrate here on the easier case where $\betaA = 0$ and leave the $\betaA > 0$ case to the Appendix.

\subsection{Barron norm generalized ($\alpha > 1/2$)}
\label{sec:space_alpha>1/2}
We consider the following type of function spaces as a generalization of the Barron space. Let $\Fbase$ be a vector space of real-valued functions on $\domX$ equipped with a norm $\| \cdot \|_{\Fbase}$ and $\sigma: \Rbb \to \Rbb$ be an activation function of choice. Then, for any function $f$ on $\domX$, we can define:
\begin{equation}
\label{eq:varn}
    \varn_{\Fbase, \sigma}(f) 
    =~ \inf_{\mu} \int_{\mathbb{R} \times \Fbase} |a| \| h \|_{\Fbase} \mu(da, dh)~,
\end{equation}
where the infimum is taken over all $\mu \in \Pcal(\Rbb \times \Fbase)$ such that $f(\xb) = \int_{\mathbb{R} \times \Fbase} a \sigma(h(\xb)) \mu(da, dh)$.
For $c \geq 0$, we then use $\FaUsigc{\Fbase}{\sigma}{c}$ to denote the space of all functions $f$ on $\domX$ such that $\varn_{\Fbase, \sigma}(f) \leq c$. We further define $\Fa_{\Fbase, \sigma} = \cup_{c > 0} \FaUsigc{\Fbase}{\sigma}{c}$.

\begin{example}
If we choose $\Fbase = \mathbb{R}^d$ (identified with the space of linear functions on $\mathbb{R}^d$), then $\Fa_{\mathbb{R}^d, \sigma}$ coincides with the Barron space associated with activation function $\sigma$ \citep{ma2022barron} and is equivalent to the ``$\mathcal{F}_1$'' space from \citet{bach2017breaking} when $\sigma$ is $1$-homogeneous. 
\end{example}
Meanwhile, choosing $\Fbase = \Hilb$ and $\sigma = \sigma_2$ defines a function space that contains the $\ft{t}$ from Lemma~\ref{lem:mf_exist_1} for the $\alpha > 1/2$ case:
\begin{proposition}
\label{prop:omega_t_bound}
    If $\alpha > 1/2$ and Assumptions~\ref{ass:sigma_diff} and \ref{ass:init} are satisfied (for Lemma~\ref{lem:mf_exist_1} to hold), then $f_t \in \Fa_{\Hilb, \sigma_2}$, $\forall t \geq 0$.
\end{proposition}
This result is a consequence of the 
the following lemma that bounds the evolution of the flow maps by the loss trajectory:
\begin{lemma}
\label{lem:AH_movement_bound} 
Under the conditions for Proposition~\ref{prop:omega_t_bound}, 
it holds that
\begin{equation}
\label{eq:AH_movement_bound_1}
    \int_{\Rbb \times \Con}  \left \| \Ht{t}(a, h) - h\right \|_{\Hilb} \mut{0}(da, dh) \leq \int_0^t \Big (-\frac{d}{ds} \mathcal{L}_s \Big )^{1/2} ds~.
\end{equation}
Moreover, when $\betaA = 0$ and
Assumptions~\ref{ass:sigma_diff} -- \ref{ass:sigma_bdd} hold, we also have 
\begin{equation}
\label{eq:AH_movement_bound_2}
    \sup_{(a, h) \in \supp(\mu_0)}\left \| \Ht{t}(a, h) - h\right \|_{\Hilb} \leq \sqrt{2} (\| \Gfun \|_{\infty})^{1/2} a_{\max} \mathtt{L}_{\sigma_2} \int_0^t (\mathcal{L}_s )^{1/2} ds~,
\end{equation}
where $\| \Gfun \|_{\infty} \coloneqq \max_{\xb \in \domX} |\Gfun(\xb, \xb)|$ and $a_{\max} \coloneqq \max_{a \in \supp(\rhoa)} |a|$.
\end{lemma}
The proof of Lemma~\ref{lem:AH_movement_bound} with an extension of \eqref{eq:AH_movement_bound_2} to the $\betaA > 0$ case is given in Appendix~\ref{app:pf_lem_AH_movement_bound}.

\subsubsection{Rademacher complexity} When $\sigma$ is $1$-homogeneous (e.g. ReLU or linear), we can control the Rademacher complexity of $\FaUsigc{\Fbase}{\sigma}{c}$ by that of the unit ball in $\Fbase$ via the following lemma, which is proved in Appendix~\ref{app:pf_rad_1}:
\begin{lemma}
\label{lem:rad_1}
If $\sigma$ is $1$-homogeneous and $\mathtt{L}_{\sigma}$-Lipschitz, then
$\Rad_n(\FaUsigc{\Fbase}{\sigma}{c}) \leq c \mathtt{L}_{\sigma} \Rad_n(\ball{\Fbase}{1})$.
\end{lemma}
Hence, via standard Rademacher complexity bounds of RKHS (e.g., \citealt[Theorem 6.12]{mohri2018foundations}), we obtain the following as a corollary:
\begin{corollary}
\label{cor:rad_>1/2_hilb}
If $\sigma$ is $1$-homogeneous and $\mathtt{L}_{\sigma}$-Lipschitz, then 
\begin{equation}
    \Rad_n(\FaUsigc{\rkhsh}{\sigma}{c}) \leq c \mathtt{L}_{\sigma} (\| \Gfun \|_{\infty})^{1/2} / {\sqrt{n}}~.
\end{equation}
\end{corollary}

\subsection{Complexity Measure via Transport Distance in Function Space ($\alpha \geq 1/2$)}
\label{sec:1/2_gen}
While the complexity measure \eqref{eq:varn} is suitable for characterizing the functions obtained by the MF training dynamics when $\alpha > 1/2$, it falls short in the case of $\alpha = 1 / 2$: there is no guarantee that $\varn_{\rkhsh, \sigma_2}(\ft{t}) < \infty$ since the $\mut{t}$ is no longer supported within $\rkhsh$ even at $t = 0$.\footnote{Although by choosing $\Fbase$ to be $\Con$ with a suitable norm, we could easily show that $\varn_{\Con, \sigma}(\ft{t}) < \infty$ at finite $t \geq 0$, it will be difficult to derive Rademacher complexity bounds since $\Con$ is too large to avoid the ``curse of dimensionality''. } We need an alternative complexity measure that is more tailored to the dynamics. 

We recall from \eqref{eq:HtfromLambt_Con} that for $(a, h)$ in the support of $\mut{0}$, even though neither $h$ nor $\Htah{t}$ necessarily belongs to $\rkhsh$, their difference, $(\Htah{t} - h)$, always does. In other words, $\mut{t}$ is obtained as the push-forward of $\mut{0}$ via a flow map whose \emph{displacement} is everywhere bounded in $\rkhsh$. Therefore, we can let our space include all functions representable as $\fmu{\mu}$ for which $\mu$ is within a certain distance from $\mut{0}$, where this distance is measured by an optimal-transport-type metric between distribution of functions, as we will introduce below.

We start from a general setup where $\Fbase$ and $\mathcal{V}$ are two Banach spaces with norms $\| \cdot \|_{\Fbase}$ and $\| \cdot \|_{\mathcal{V}}$ such that $\Fbase \subseteq \mathcal{V}$, and we define an optimal-transport-type extended metric between probability measures on $\mathbb{R} \times \mathcal{V}$ as follows\footnote{The definition that follows is tailored specifically to the simpler case of $\betaA = 0$; for the case where $\betaA > 0$, the more general definition is given in Appendix~\ref{app:pf_rad_1/2}.}. Let $\mu, \mu'$ be two probability measures on $\mathbb{R} \times \mathcal{V}$, and let $\joint(\mu, \mu')$ denote the space of probability measures on $\mathbb{R} \times \mathcal{V} \times \mathcal{V}$ that satisfy $\int_{\mathcal{V}} \pi(\cdot, \cdot, dh') = \mu$ and $\int_{\mathcal{V}} \pi(\cdot, dh, \cdot) = \mu'$.
Then, inspired by the Wasserstein metrics\footnote{Wasserstein metrics are parameterized by an exponent $p \in [1, +\infty]$, and the definition \eqref{eq:Winfty} corresponds to the case $p = \infty$. An analogous definition for $p \in [1, \infty)$ is given in Appendix~\ref{app:Wp}.} between probability measures on metric spaces, 
we define
\begin{equation}
\label{eq:Winfty}
    \mathcal{W}_{\infty}(\mu, \mu'; \Fbase, \mathcal{V}) \coloneqq \inf_{\pi \in \joint(\mu, \mu')} \hspace{10pt} \text{ess} \hspace{-19pt} \sup_{\pi(da, dh, dh') \hspace{15pt}} \| h - h' \|_{\Fbase}~.
\end{equation}
Note that since the right-hand-side may not be finite, this is an \emph{extended} metric on $\Pcal(\Rbb \times \mathcal{V})$.

Let us now focus on the case where $\mathcal{V} = \Con$. Specifically, let $\mubase$ be any \emph{base} probability measure on $\mathbb{R} \times \Con$. Then, for any function $f$ on $\domX$, we define:
\begin{equation}
\label{eq:varn_dag}
    \varninfty{\Fbase}{\sigma}{\mubase}(f) :=~ \inf_{\mu} \mathcal{W}_{\infty}(\mu, \mubase; \Fbase, \Con)~,
\end{equation}
where the infimum is taken over all $\mu \in \Pcal(\Rbb \times \Con)$ such that $f(\xb) = \int_{\mathbb{R} \times \Con} a \sigma(h(\xb)) \mu(da, dh)$. 
As in the $\alpha > 1/2$ case, for any $c \geq 0$, we use $\FbUsigc{\Fbase}{\sigma}{\mubase}{c}$ to denote the space of all functions $f$ on $\domX$ such that $\varninfty{\Fbase}{\sigma}{\mubase}(f) \leq c$, and we further define $\FbUsig{\Fbase}{\sigma}{\mubase} = \cup_{c > 0} \FbUsigc{\Fbase}{\sigma}{\mubase}{c}$.

\begin{remark}
\label{rmk:scaling}
    A concurrent work by \citet{neumayer2024effect} also proposes an optimal-transport based complexity measure for functions represented by infinite-width $2$L NNs, which is similar to \eqref{eq:varn_dag} (and the generalized version defined in Appendix~\ref{app:pf_rad_1/2}) when we choose $\Fbase$ as the space of linear functions on $\Rbb^d$. In comparison, by allowing more general choices of $\Fbase$, our definition is relevant to more general models including \ptl NNs in the MF limit. 
    
     We refer interested readers to Section 2 of \citet{neumayer2024effect} for a discussion on further theoretical properties of the version defined therein. We focus below on relating our complexity measure to the MF training dynamics and
    deriving Rademacher complexity bounds on the corresponding function space.
\end{remark}
Then, setting $\Fbase = \rkhsh$, $\sigma = \sigma_2$ and $\mubase = \mut{0}$ allows us to define appropriate spaces for the functions $\ft{t}$ obtained by the MF training dynamics when $\alpha \geq 1 / 2$ (note that $\mut{0}$ is different in the two cases of $\alpha > 1 / 2$ and $\alpha =  1 / 2$). In particular, \eqref{eq:AH_movement_bound_2} implies that for any $t \geq 0$, $\ft{t} \in \FbUsig{\rkhsh}{\sigma_2}{\mut{0}}$ with $\varninfty{\rkhsh}{\sigma_2}{\mut{0}}(f_t) \leq \sqrt{2} (\| \Gfun \|_{\infty})^{1/2} a_{\max} \mathtt{L}_{\sigma_2} \int_0^t (\mathcal{L}_s )^{1/2} ds$. We see that the dependence of the right-hand-side of the bound depends on the training set and training time only through the integral $\int_0^t (\mathcal{L}_s )^{1/2} ds$, which is controlled by the decay rate of the training loss. In particular, if the conditions of Theorem~\ref{prop:gc} are satisfied, we have $\int_0^{\infty} (\mathcal{L}_s )^{1/2} ds \leq 2 (\mathcal{L}_0)^{1/2} / (r \hat{a} \lambmin(G))$ (with the same $r$ and $\hat{a}$ as defined therein), which yields a finite bound for all time that depends on the the training set through $1 / \lambmin(G)$. Formally, this leads to the following result:
\begin{corollary}
    Suppose that Assumptions~\ref{ass:sigma_diff} -- \ref{ass:itvl} are satisfied. If $a_{\max} \geq \hat{a}$, then it holds for all $t \geq 0$ that
    \begin{equation}
        \varninfty{\rkhsh}{\sigma_2}{\mut{0}}(f_t) \leq \frac{2 \sqrt{2} (\| \Gfun \|_{\infty})^{1/2} a_{\max} \mathtt{L}_{\sigma_2}}{r \hat{a}^2 \lambmin(G)}~,
    \end{equation}
    where $r$ and $\hat{a}$ have the same definition as in Theorem~\ref{prop:gc}.
\end{corollary}

\subsubsection{Rademacher complexity} The Rademacher complexity of $\FbUsigc{\Fbase}{\sigma}{\mubase}{c}$ can still be controlled by that of the unit ball of $\Fbase$, in fact \emph{without} homogeneity assumptions on $\sigma$ (unlike Lemma~\ref{lem:rad_1}):
\begin{lemma}
\label{lem:rad_1/2}
If $\sigma$ is $\mathtt{L}_{\sigma}$-Lipschitz, then $\forall c > 0$,
\begin{equation}
    \Rad_n(\FbUsigc{\Fbase}{\sigma}{\mubase}{c}) \leq \mathtt{L}_{\sigma} \bigg ( \int_{\mathbb{R} \times \Con} |a| \mubase(da, dh) \bigg ) \Rad_n(\ball{\Fbase}{c})~.
\end{equation}
\end{lemma}
\noindent This lemma is proved in Appendix~\ref{app:pf_rad_1/2}. As a corollary of the Rademacher complexity bounds of RKHS, we therefore derive that:
\begin{corollary}
\label{cor:rad_1/2_hilb}
Under Assumptions~\ref{ass:sigma_diff} and \ref{ass:init}, it holds for all $c > 0$ that
\begin{equation}
    \Rad_n(\FbUsigc{\rkhsh}{\sigma_2}{\mut{0}}{c}) \leq  c \mathtt{L}_{\sigma_2} a_{\max} (\| \Gfun \|_{\infty})^{1/2} / {\sqrt{n}}~.
\end{equation}
\end{corollary}

\section{Numerical experiments}
\label{sec:exp}
We present numerical experiments to complement our theoretical analysis above on \ptl NNs and provide empirical evidence for their large-width limit, the connection with the $n$-dimensional $2$L NNs, the impact of the choice of $\alpha$ as well as the comparison with related NN models (NTK, $2$L NN and fully-trained $3$L NN).

\subsection{Tasks} We consider two synthetic data distributions on $\Rbb^2$ with binary labels and train our models in an $L_2$ regression setting.
Task I is introduced by \citet{chizat2020implicit} for comparing kernel versus feature learning regimes in $2$L NNs. Task II has a data distribution supported on three concentric circles where the labels depend alternatingly on the radius. This task is inspired by prior theoretical results on the advantage of deeper NNs compared to $2$L NNs in approximating and learning radial functions \citep{eldan2016power, safran2022optimization}. We choose $n=18$ and $100$ as the sizes of the training set in the two settings, respectively.

\subsection{Models} We choose three variants of the \ptl NN model: $\textbf{\ptbl ($\boldsymbol{\alpha = 1}$)}$,  $\textbf{\ptbl ($\boldsymbol{\alpha = 1/2}$)}$ and $\textbf{\ptbl (NTK)}$. The first two are defined by \eqref{eq:p3l} with their respective choices of $\alpha$, while the third is a $3$L NN under the NTK parameterization with the input-layer weights untrained.
All three models have the same width in the two hidden layers ($m_1 = m_2 = m$) with various choices of $m$. For comparisons, we also include $2$L NNs ($\textbf{2L}$) and fully-trained $3$L NNs ($\textbf{3L}$) with the same widths. In Figure~\ref{fig:xavier} in Appendix~\ref{app:exp}, we additionally compare \ptl NN with $\alpha = 1 / 2$ versus under Xavier scaling in the case where $\sigma_2$ is ReLU.

To validate the connections between MF \ptl NN and the $n$-dimensional MF $2$L NN (i.e., $g_t(\Xtil(\xb))$) established in Section~\ref{sec:4}, we also consider finite-width realizations of the latter, i.e., $2$L NNs on $\Rbb^n$ trained to fit the same training set under a transformation: $\{(\xtildetrk{k}, \ytrk{k}) \}_{k \in [n]}$. We include two versions, \textbf{dim-$\boldsymbol{n}$ $\boldsymbol{2}$L ($\mathcal{N}$-init)} and \textbf{dim-$\boldsymbol{n}$ $\boldsymbol{2}$L ($\boldsymbol{0}$-init)}, with $\muhidt{0} = \rhoa \times \mathcal{N}(0, \text{Id}_n)$ (corresponding to $\alpha = 1 / 2$) and $\muhidt{0} = \rhoa \times (\delta_0)^n$ (corresponding to $\alpha \geq 1$),
respectively.

We choose $\sigma_1$ as ReLU so that Assumption~\ref{ass:relu_anp} is satisfied and the kernel function $\Gfun$ can be computed analytically.
We choose $\sigma_2$ primarily as tanh (which satisfies Assumptions~\ref{ass:sigma_diff}, \ref{ass:sigma_bdd} and \ref{ass:itvl}) while also including the case where $\sigma_2$ is ReLU for Task II. The bias term in the last hidden layer is included and initialized to be zero, and we set $\betaA = 0$ and $\betaB = 0.5$. All models are trained with full-batch GD. We choose a step size of $0.05$ for the \ptl models and dim-$n$ $2$L models and adjust it for other models when needed to ensure training stability. For each pair of task and model, the experiment is run three times with different random seeds for parameter initialization (held identical across all models). The error curves are averaged over the three runs while the other visualizations are based on the first run.
\subsection{Results}
\label{sec:results}
\begin{figure}[!h]
    \centering
    \includegraphics[scale=0.305]{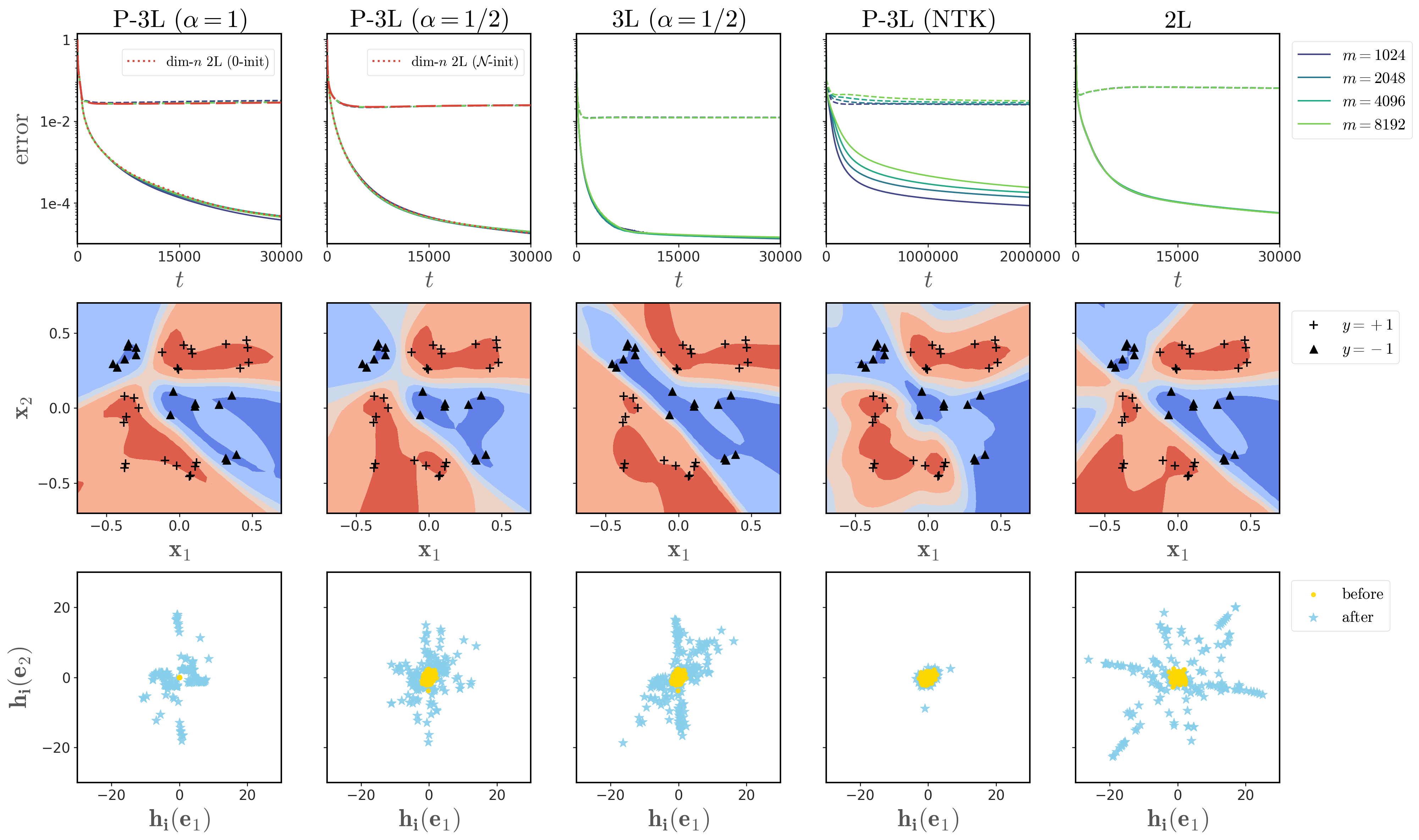}
    \caption{Numerical results on Task I. \textbf{Row $1$}: Curves of training (solid) and testing (dashed) errors for different choices of $m$. In the first two columns, the red curves are the training and testing errors of the respective dim-$n$ $2$L NNs with $m = 8192$. \textbf{Row $2$}: Contour plots of the output function after training with $m = 8192$. \textbf{Row $3$}: Pre-activation values of neurons in the (last) hidden layer evaluated on the two unit vectors in $\Rbb^2$ with $m = 8192$, before (yellow) and after (blue) training.}
    \label{fig:chizat2D}
\end{figure}
\begin{figure}[!h]
    \centering
    \includegraphics[scale=0.305]{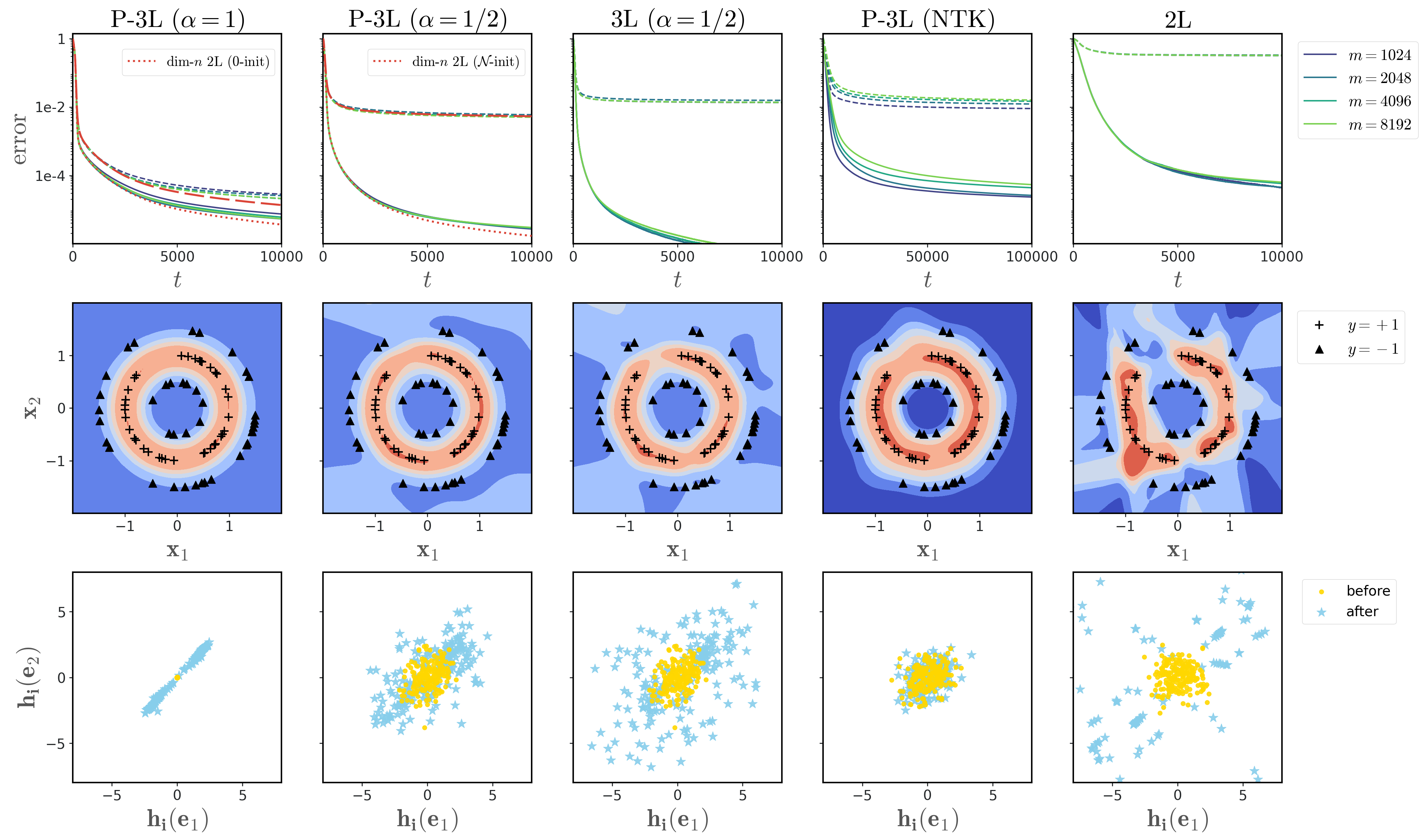}
    \caption{Numerical results of on Task II with $\sigma_2$ chosen to be tanh. The plots are defined in the same way as in Figure~\ref{fig:chizat2D}.}
    \label{fig:radial_relu+tanh}
\end{figure}
\begin{figure}[!h]
    \centering
    \includegraphics[scale=0.305]{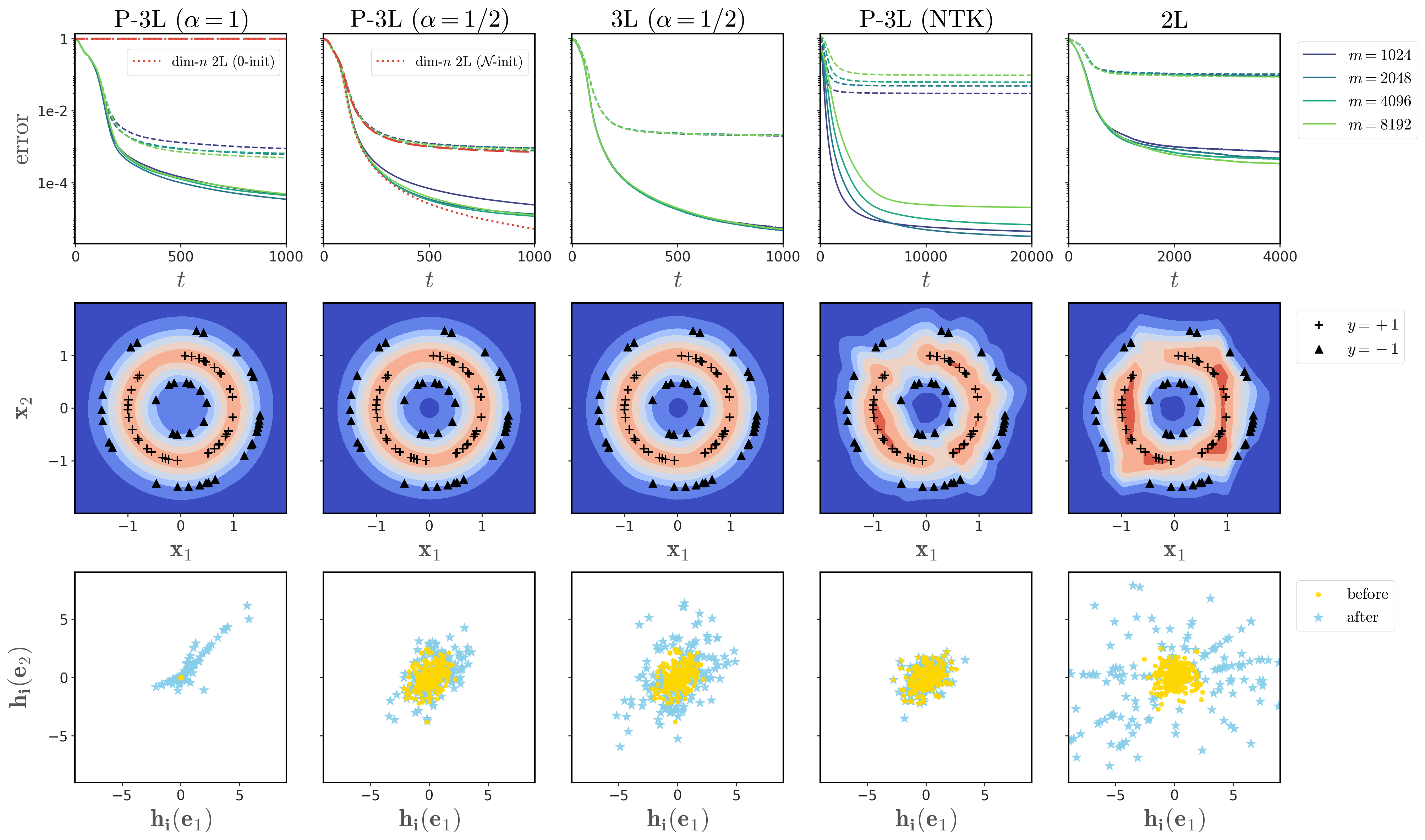}
    \caption{Numerical results of various models on Task II, where we choose $\sigma_2$ to be ReLU. The plots have the same setting as in Figure~\ref{fig:chizat2D}.}
    \label{fig:radial_relu+relu}
\end{figure}
Figures~\ref{fig:chizat2D} and \ref{fig:radial_relu+tanh} show the empirical results on the two tasks when $\sigma_2$ is tanh, and Figure \ref{fig:radial_relu+relu} show the result on Task II when $\sigma_2$ is ReLU.

\paragraph{Large-width asymptotics.} When $\sigma_2$ is tanh (hence satisfying Assumptions~\ref{ass:sigma_diff} and \ref{ass:sigma_bdd}), our theory predicts the existence of an infinite-width MF limit for \ptl NNs with $\alpha = 1$ and $1 / 2$. This is consistent with the first row of Figures~\ref{fig:chizat2D} and \ref{fig:radial_relu+tanh}, where loss curves of both training and testing are nearly uniform across different choices of $m$. In particular, the training curves are close to that of the corresponding $n$-dimensional $2$L NNs, which is consistent with our theoretical result that the two types of models coincide in their infinite-width limits on the training set. We note, though, that the MF theory concerns the ``finite $t$, $m \to \infty$'' limit, whereas if we fix $m$, the discrepancy can increase as $t$ becomes large. 

Meanwhile, when we choose $\sigma_2$ as ReLU, which does not satisfy the regularity assumptions for Theorem~\ref{thm:mflim}, Figure~\ref{fig:radial_relu+relu} shows that 
\textbf{\ptbl ($\boldsymbol{\alpha = 1}$)} no longer shares the same infinite-width limit as that of \textbf{dim-$\boldsymbol{n}$ $\boldsymbol{2}$L ($\boldsymbol{0}$-init)}. In the latter, all neurons in the second hidden layer represent the zero function (i.e., $\muhidt{0}$ is a singular measure at the zero function). Since ReLU is not differentiable at $0$ (and we typically choose the zero subgradient in back-propagation), $\muhidt{t}$ will not evolve at all during training. By contrast, with random initialization breaking the symmetry, a finite-width \ptl NN with $\alpha = 1$ does not suffer from the same lack of gradient signals. We illustrate how this key difference manifests during the early dynamics in Figure~\ref{fig:early}. It shows an example of the infinite-width limit breaking down when the differentiability assumption is not satisfied.

Further comparisons between \ptl NNs and their corresponding $n$-dimensional $2$L NNs in terms of learned functions and pre-activation value distributions are given in Figures~\ref{fig:chizat_tanh_dim-n} -- \ref{fig:radial_relu_dim-n}.

\paragraph{Comparison with NTK parameterization} As expected from prior analyses on lazy learning \citep{chizat2019lazy}, under the NTK parameterization, the second-hidden-layer neurons barely move throughout training in terms of the pre-activation values. This results in qualitative differences in the learned functions as well as higher test errors on Task II. A theoretical comparison between the NTK and our scaling choices for \ptl NNs is beyond the scope of this work, though we refer the interested readers to \citet{wei2019regularization} for an insightful analysis in the context of $2$L NNs.

\paragraph{Comparison with $2$L NN} From Figures~\ref{fig:radial_relu+tanh} and \ref{fig:radial_relu+relu}, we see lower training and test errors achieved on Task II by both the \ptl and the $3$L NNs compared to $2$L NNs, which corroborates the theoretical results on the advantage of three- versus two-layer NNs in terms of both approximating and learning radial functions \citep{eldan2016power, safran2022optimization}.

\paragraph{Training of input layer.} On both tasks, both $\textbf{\ptbl ($\boldsymbol{\alpha = 1/2}$)}$ and $\textbf{$3$L ($\boldsymbol{\alpha = 1/2}$)}$ achieve training losses well below $10^{-4}$, though the latter has a faster decay of training loss with the training of the input-layer weights. Visually, in both models, the second-hidden-layer neurons exhibit significant movements in their pre-activation values through training. The output functions that they learn can be slightly different (e.g., see second row of Figure~\ref{fig:chizat2D}). On Task II (Figures~\ref{fig:radial_relu+tanh} and \ref{fig:radial_relu+relu}), it is worth noting that the \ptl NNs achieve even lower test errors than the $3$L NNs. Interestingly, the $3$L NN example constructed by \citet[Theorem 4.3]{safran2022optimization} which learns the ball indicator function efficiently under GD also has the first-layer weights random and fixed. This suggests that three-layer NNs can exhibit a benefit of depth even when the input-layer weights are \emph{not} trained.

\section{Conclusions and Limitations}
In this work, we defined the infinite-width limit of \ptl NN by rigorously developing a functional-space MF theory. Through this framework, we proved a linear-rate convergence guarantee of the empirical loss for the limiting model. We then characterized the functional spaces explored by the MF dynamics via novel complexity measures based on optimal-transport-type distances between distributions of functions and bounded their Rademacher complexity. 
Our analysis covers two different regimes of scaling the model output by its width ($\alpha > 1/2$ and $\alpha = 1 / 2$), which result in different behaviors through training despite both exhibiting feature learning.

Our theory is still limited in several ways. First, by only focusing on the unregularized setting, we do not have a priori generalization bounds derived. Second, a comparison of the new function spaces with the ones associated with $2$L NNs is still lacking.
Third, the theoretical result on the MF limit needs boundedness and smoothness assumptions on the activation function of the second hidden layer, which is relatively standard in the literature but excludes e.g. the ReLU function. Lastly, the \ptl NN model assumes that the parameters in the first layer are fixed, which is not often seen in practice. 
Despite these shortcomings, the framework developed in this work is a helpful stepping stone for further advances. In particular, we refer the readers to a follow-up work that extends the idea of a functional-space MF theory to cover more general multi-layer NNs where all layers are trainable \citep{chen2024nhl}.

\acks{The authors thank Carles Domingo-Enrich and anonymous reviewers for feedback on the manuscript, and acknowledge support from the Henry McCracken Fellowship, the Isaac Barkey and Ernesto Yhap Fellowship, NSF RI-1816753, NSF CAREER CIF 1845360, NSF CHS-1901091, NSF Scale MoDL DMS 2134216, Capital One and Samsung
Electronics.}

\newpage

\appendix
\section{$\alpha = 1 / 2$ is asymptotically equivalent to Xavier initialization}
\label{app:xavier}
Consider a three-layer NN (with omitted bias terms and $m_1 = m_2 = m$) defined by
\begin{align}
    f(\boldsymbol{x}) =& \sum_{i=1}^{m} \theta^{(a)}_i \sigma_2 \big ( {h}_i(\boldsymbol{x}) \big )~, \\
     \forall i \in [m] \quad : \quad {h}_i(\boldsymbol{x}) =& \sum_{j=1}^{m} \theta^{(W)}_{i, j} \sigma_1 \Big (\sum_{k=1}^d \theta^{(z)}_{j, k} x_k \Big )~,
\end{align}
with weight parameters $\big \{ \theta^{(z)}_{j, k} \big \}_{j \in [m], k \in [d]}$, $\big \{ \theta^{(W)}_{i, j} \big \}_{i, j \in [m]}$ and $\big \{ \theta^{(a)}_i \big \}_{i \in [m]}$ initialized according to Xavier-normal initialization \citep{glorot2010difficulty}, meaning that we sample each $\theta^{(W)}_{i, j}$ i.i.d. from $\mathcal{N}(0, \frac{2}{m + m}) = \mathcal{N}(0, \frac{1}{m})$, each $\theta^{(z)}_{j, k}$ i.i.d. from $\mathcal{N}(0, \frac{2}{m+d})$, and each $\theta^{(a)}_i$ i.i.d. from $\mathcal{N}(0, \frac{2}{m+1})$ at $t=0$. If $m \to \infty$ while $d$ remains fixed, the latter two distributions become approximately $\mathcal{N}(0, \frac{2}{m})$. Thus, by redefining $a_i = \sqrt{m} \theta^{(a)}_i$, $W_{i, j} = \sqrt{m} \theta^{(W)}_{i, j}$ and $z_{j, k} = \sqrt{m} \theta^{(z)}_{j, k}$, we can write
\begin{align}
    f(\boldsymbol{x}) =& \frac{1}{\sqrt{m}} \sum_{i=1}^{m} a_i \sigma_2 \big ( h_i(\boldsymbol{x}) \big )~, \\
     \forall i \in [m] \quad : \quad h_i(\boldsymbol{x}) 
    =& \frac{1}{\sqrt{m}} \sum_{j=1}^{m} W_{i, j} \sigma_1 \Big ( \frac{1}{\sqrt{m}} z_{j}^{\intercal} \boldsymbol{x} \Big ) ~,
\end{align}
and where $a_i, W_{i, j}$ and $z_{j, k}$ are all initialized from normal distributions with variance $O(1)$ as $m \to \infty$. If $\sigma_1$ and $\sigma_2$ are homogeneous (e.g., ReLU or leaky ReLU), the $\frac{1}{\sqrt{m}}$ factors and the activation functions commute, and hence this is equivalent to the definition in \eqref{eq:p3l} under the choice of $\alpha = 1/2$ and $m_1 = m_2 = m$ at initialization.

Furthermore, the equivalence continues to hold into the GD training of the \ptl NN under the learning-rate rescaling of \eqref{eq:dotati} and \eqref{eq:dotWti}. To see this, note that $\frac{\partial f(\xb)}{\partial a_i} = \frac{1}{\sqrt{m}} \frac{\partial f(\xb)}{\partial \theta^{(a)}_{i}}$ and $\frac{\partial f(\boldsymbol{x})}{\partial W_{i, j}} = \frac{1}{\sqrt{m}} \frac{\partial f(\boldsymbol{x})}{\partial \theta^{(W)}_{i, j}}$. Then, since performing GD on $\theta^{(a)}_{i}$ and $\theta^{(W)}_{i, j}$ with step size $\delta$ means updating them according to
\begin{equation}
\begin{split}
    \theta^{(a)}_{i} \leftarrow \theta^{(a)}_{i} - \delta \frac{\partial L}{\partial \theta^{(a)}_{i}}~, \\
    \theta^{(W)}_{i, j} \leftarrow \theta^{(W)}_{i, j} - \delta \frac{\partial L}{\partial \theta^{(W)}_{i, j}}~,
\end{split}
\end{equation}
this is equivalent to updating $W_{i, j}$ according to
\begin{equation}
\begin{split}
    a_{i} \leftarrow & \sqrt{m} \Big ( \theta^{(a)}_{i} - \delta \frac{\partial L}{\partial \theta^{(a)}_{i}} \Big ) = a_{i} - m \delta \frac{\partial L}{\partial a_{i}}~, \\
    W_{i, j} \leftarrow & \sqrt{m} \Big ( \theta^{(W)}_{i, j} - \delta \frac{\partial L}{\partial \theta^{(W)}_{i, j}} \Big ) = W_{i, j} - m \delta \frac{\partial L}{\partial W_{i, j}}~,
\end{split}
\end{equation}
which is equivalent to \eqref{eq:dotati} and \eqref{eq:dotWti} when $\alpha = 1 / 2$, $m_1 = m_2 = m$ and $\beta_a = 1$.

For numerical evidence of this asymptotic equivalence, see Figures~\ref{fig:xavier} and \ref{fig:xavier_vs_1/2}.

\section{Proof that $G$ is positive semi-definite}
\label{app:pf_Gpsd}
It is obvious to see that $G$ is a symmetric function in its two arguments. To show that it satisfies the positive semi-definite condition, consider any $\xbg{1}, ..., \xbg{k} \in \domX$ and $c_1, ..., c_k \in \Rbb$. It holds that
\begin{equation}
\begin{split}
    \summ{i, j}{k}{c_i c_j G(\xbg{i}, \xbg{j})} =&~ \summ{i, j}{k}{c_i c_j \int_{\Rbb^d} \sigma_1(\zb^{\intercal} \cdot \xbg{i}) \sigma_1(\zb^{\intercal} \cdot \xbg{j}) \rhoz(d\zb)} \\
    =&~ \int_{\Rbb^d} \summ{i, j}{k}{c_i c_j \sigma_1(\zb^{\intercal} \cdot \xbg{i}) \sigma_1(\zb^{\intercal} \cdot \xbg{j})} \rhoz(d\zb) \\
    =&~ \int_{\Rbb^d} \bigg ( \summ{i}{k} c_i \sigma_1(\zb^{\intercal} \cdot \xbg{i}) \bigg )^2 \rhoz(d\zb) \geq 0
\end{split}
\end{equation}

\section{Proof of Lemma~\ref{lem:lln_0_1}}
\label{app:pf_lln_0_1}
Let $\{ \xbg{1}, ..., \xbg{\ng} \}$ be any finite subset of $\mathcal{X}$. We write $\Gg = \Gfun[\xbg{1}, ..., \xbg{\ng}]$, $\Ggmb = \Gfunmb[\xbg{1}, ..., \xbg{\ng}]$, $\evte = \ev_{\xbg{1}, ..., \xbg{\ng}}$ and $\hatevte = \hatev_{\xbg{1}, ..., \xbg{\ng}}$.

Recall that when $\alpha > \frac{1}{2}$, $\hatevte(\mut{0}) = \rhoa \times \delta_{\zerob}$.
By the triangle inequality of $1$-Wasserstein distance, there is
\begin{equation}
\label{eq:triangle_1}
\begin{split}
    \mathcal{W}_1(\hatevte(\mumbt{0}), \hatevte(\mut{0})) =&~ \mathcal{W}_1(\hatevte(\mumbt{0}), \rhoa \times \delta_{\zerob}) \\
    \leq &~ \mathcal{W}_1(\rhoa \times \mathcal{N}(0, m_1^{1-2\alpha}\Ggmb), \rhoa \times \delta_{\zerob}) \\
    & ~ + \mathcal{W}_1(\hatevte(\mumbt{0}), \rhoa \times \mathcal{N}(0, m_1^{1-2\alpha}\Ggmb))~.
\end{split}
\end{equation}
First, we examine the first term on the right-hand side.  By the property of Wasserstein distances on product measures (e.g. \citealt[Lemma 3]{mariucci2018wasserstein}), we have
\begin{equation}
\label{eq:first_term_alpha>1}
    \begin{split}
        \mathcal{W}_1(\rhoa \times \mathcal{N}(0, m_1^{1-2\alpha}\Ggmb), \rhoa \times \delta_{\zerob})
        \leq & ~ \mathcal{W}_1(\rhoa, \rhoa) + \mathcal{W}_1(\mathcal{N}(0, m_1^{1-2\alpha}\Ggmb), \delta_{\zerob})) \\
        \leq & ~ \mathcal{W}_1(\mathcal{N}(0, m_1^{1-2\alpha}\Ggmb), \delta_{\zerob})) \\
        \leq & ~ \left ( \EE_{\Zb \in \mathcal{N}(0, \Ggmb)} \left [ \left \| m_1^{1/2-\alpha}\Zb \right \|_2^2 \right ]  \right )^{\frac{1}{2}} \\
    \leq &~ \frac{(\mathrm{Tr}(\Ggmb))^{1/2}}{m_1^{\alpha - 1/2}} \leq \frac{(n' \Ggmbmax)^{1/2} }{m_1^{\alpha-1/2}}~.
    \end{split}
\end{equation}
\noindent For the second term, we see that, when conditioned on $\zb_1, ..., \zb_{m_1}$, $\{ [a^0, h_i^0(\xbg{1}), ..., h_i^0(\xbg{\ng})] \}_{i \in [m_2]}$ is distributed i.i.d. across $i \in [m_2]$ according to $\rhoa \times \mathcal{N}(0, m_1^{1-2\alpha} \Ggmb)$. Hence, when conditioned on $\Ggmb$ (which is measurable with respect to $\zb_1, ..., \zb_{m_1}$), $\hatevte(\mumbt{0})$ has the same distribution as the empirical measure of $m_2$ i.i.d. samples from $\rho^0_a \times \mathcal{N}(0, m_1^{1-2\alpha} \Ggmb)$, which we denote by $\nu_{(m_2)} \in \mathcal{P}(\mathbb{R} \times \mathbb{R}^{n'})$. Therefore, by conditioning on $\Ggmb$, we can leverage concentration inequalities in Wasserstein distance of empirical measures of i.i.d. samples:
\begin{lemma}[Adapted from \citet{fournier2015rate}, Theorem 2]
\label{lem:wconcen}
Given a probability measure $\nu \in \mathcal{P}(\mathbb{R}^d)$, let $\nu_{(m)}$ be the empirical measure of $m$ i.i.d. samples from $\nu$. If $\exists \iota > 1, \exists \gamma > 0$ such that 
\begin{equation}
    \mathcal{E}_{\iota, \gamma}(\nu) := \int_{\mathbb{R}^d} e^{\gamma |\xb|^{\iota}} \nu(d\xb) < \infty~,
\end{equation}
then $\forall m > 1, \forall u > 0$, 
\begin{equation}
    \mathbb{P}(\mathcal{W}_1(\nu_{(m)}, \nu) \geq u) \leq \begin{cases}
    C_1 e^{-C_2 m (u / \log(2 + 1/u))^2} \mathds{1}_{u \leq 1} + C_1 e^{-C_2 m u^{\iota}} \mathds{1}_{u > 1}~,~ & \text{ if } $d = 2$ \\
    C_1 e^{-C_2 m u^d} \mathds{1}_{u \leq 1} + C_1 e^{-C_2 m u^{\iota}} \mathds{1}_{u > 1}~,~ & \text{ if } $d > 2$~,
    \end{cases}
\end{equation}
where $C_1$ and $C_2$ depend only on $d, \iota, \gamma$ and $\mathcal{E}_{\iota, \gamma}(\nu)$.
\end{lemma}
In particular, choosing $\nu = \rhoa \times \mathcal{N}(0, m_1^{1-2\alpha} \Ggmb), \iota = 2, \gamma = \frac{1}{2 \lambda_{\max}(\Ggmb)}$, there is
\begin{equation}
    \begin{split}
        \mathcal{E}_{2, \gamma}(\nu) =&~ \int_{\Rbb} \frac{1}{(2 \pi)^{\frac{\ng}{2}}} \int_{\Rbb^{\ng}} e^{\gamma (a^2 + m_1^{1-2\alpha} \| (\Ggmb)^{\frac{1}{2}} \cdot \ub \|_2^2)} e^{-\| \ub \|_2^2} d \ub \rhoa(da) \\
        \leq &~ e^{\gamma (a^0_{\max})^2} \frac{1}{(2 \pi)^{\frac{\ng}{2}}} \int_{\Rbb^{\ng}} e^{(m_1^{1-2\alpha} \gamma \lambda_{\max}(\Ggmb) - 1) \| \ub \|_2^2} d \ub \\
        \leq &~ e^{\gamma (a^0_{\max})^2} \frac{2^{\frac{\ng}{2}}}{(2 \pi \cdot 2)^{\frac{\ng}{2}}} \int_{\Rbb^{\ng}} e^{- \| \ub \|_2^2 / 2} d \ub \\
        \leq &~ 2^{\frac{\ng}{2}} e^{(a^0_{\max})^2 / (2 \lambda_{\max}(\Ggmb))} < \infty~.
    \end{split}
\end{equation}
Therefore, applying Lemma~\ref{lem:wconcen}, we have $\forall u > 0$, $\exists C_1, C_2 > 0$ such that
\begin{equation}
\begin{split}
    \mathbb{P} \left (\mathcal{W}_1(\hatevte(\mumbt{0}), \rhoa \times \mathcal{N}(0, m_1^{1-2\alpha}\Ggmb)) \geq u ~ \Big |~ \Ggmb \right )
    =&~ \mathbb{P} \left (\mathcal{W}_1(\nu_{(m_2)}, \nu) \geq u ~ \Big |~ \Ggmb \right ) \\
    \leq & C_1 e^{-C_2 u^{\max\{n'+1, 4\}} m_2}~,
\end{split}
\end{equation}
where $C_1$ and $C_2$ depend only on $n'$ and $\lambda_{\max}(\Ggmb)$. Furthermore, if we condition on the event that $\| \Ggmb - \Gg \|_2 < \Delta$ for some $\Delta \in (0, \lambda_{\max}(\Gg)]$, which is measurable with respect to 
$\Ggmb$,
then by choosing $\iota = 2$ and $\gamma = \frac{1}{2 (\lambda_{\max}(\Gg) + \Delta)}$, we have $\mathcal{E}_{\alpha, \gamma}(\nu) \leq 2^{\frac{\ng}{2}} e^{(a^0_{\max})^2 / (2 \lambda_{\max}(\Ggmb))} \leq 2^{\frac{\ng}{2}} e^{(a^0_{\max})^2 / \lambda_{\max}(\Ggmb)} < \infty$. Therefore, $\forall u > 0$, $\exists C_1, C_2 > 0$ depending only on $n'$ and $\lambda_{\max}(\Gg)$ (instead of $\lambda_{\max}(\Ggmb)$) such that,
\begin{equation}
    \mathbb{P} \left (\mathcal{W}_1(\hatevte(\mumbt{0}), \rhoa \times \mathcal{N}(0, m_1^{1-2\alpha}\Ggmb)) \geq u ~ \Big |~ \| \Ggmb - \Gg \|_2 < \Delta \right ) \leq C_1 e^{-C_2 u^{\max\{n'+1, 4\}} m_2}~.
\end{equation}
Thus, choosing $\Delta = \lambda_{\max}(\Gg)$, we know from Lemma~\ref{lem:Gm_dev} that
\begin{equation}
    \Pbb \left ( \| \Ggmb - \Gg \|_2 \geq \Delta \right ) < C_3 (\ng)^2 e^{- C_4 \min \{ \Delta, C_5 \Delta^2 \} m_1}~.
\end{equation}
Fix an $\epsilon > 0$. Conditioned on the event that $\| \Ggmb - \Gg \|_2 < \Delta = \lambda_{\max}(\Gg)$, \eqref{eq:first_term_alpha>1} implies that
\begin{equation}
     \mathcal{W}_1(\rhoa \times \mathcal{N}(0, m_1^{1-2\alpha}\Ggmb), \rhoa \times \delta_{\zerob}) \leq \frac{1}{2} \epsilon~,
\end{equation}
when $m_1 \geq (\frac{8 \ng \Delta}{\epsilon^2})^{1 / (2\alpha - 1)}$.
Thus, putting things together, if $m_1 \geq (\frac{8 \ng \Delta}{\epsilon^2})^{1 / (2\alpha - 1)}$, then
\begin{equation}
\begin{split}
    &~ \Pbb \left ( \mathcal{W}_1 (\hatevte(\mumbt{0}), \hatevte(\mut{0})) > \epsilon \right ) \\
    \leq &~ 
    \Pbb \left ( \mathcal{W}_1(\hatevte(\mumbt{0}), \rhoa \times \delta_{\zerob}) > \epsilon ~ \Big |~ \| \Ggmb - \Gg \|_2 < \Delta \right ) + \Pbb \left ( \| \Ggmb - \Gg \|_2 \geq \Delta \right ) \\
    \leq & \mathbb{P} \left (\mathcal{W}_1(\hatevte(\mumbt{0}), \rhoa \times \mathcal{N}(0, m_1^{1-2\alpha} \Ggmb)) \geq \frac{1}{2} \epsilon ~ \Big |~ \| \Ggmb - \Gg \|_2 < \Delta \right ) + \Pbb \left ( \| \Ggmb - \Gg \|_2 \geq \Delta \right ) \\
    \leq & C_1 e^{-C_2 (\epsilon / 2)^{\max\{n'+1, 4\}} m_2} + C_3 (\ng)^2 e^{- C_4 \min \{ \lambda_{\max}(\Gg), C_5 (\lambda_{\max}(\Gg))^2 \} m_1}~.
\end{split}
\end{equation}
Thus, with any pair of increasing $\Nbb_+$-valued sequences $\{m_{1, k}\}_{k \in \Nbb_+}$ and $\{m_{2, k}\}_{k \in \Nbb_+}$, denoting $\mb_k = (m_{1, k}, m_{2, k})$, there is
\begin{equation}
    \summ{k}{\infty} \Pbb \left ( \mathcal{W}_1 (\hatevte(\mumbkt{k}{0}), \hatevte(\mut{0})) > \epsilon \right ) < \infty~.
\end{equation}
Since this holds for any $\epsilon > 0$, the Borel-Cantelli lemma implies that 
\begin{equation}
    \lim_{k \to \infty} \mathcal{W}_1 (\hatevte(\mumbkt{k}{0}), \hatevte(\mut{0})) = 0~,
\end{equation}
almost surely, and hence $\hatevte(\mumbkt{k}{0})$ converges weakly to $\hatevte(\mut{0})$ almost surely.

\section{ Proof of Lemma~\ref{lem:lln_0_1/2}}
\label{app:pf_lln_0_1/2}
Two parts of Lemma~\ref{lem:lln_0_1/2} need to be proved: the LLN as $m_1, m_2 \to \infty$ and the existence of $\mathcal{GP}(\zerob, \Gfun)$ as a probability measure on $\Con$. \\

\noindent \textbf{Part 1: Convergence as $m_1, m_2 \to \infty$}

Let $\{ \xbg{1}, ..., \xbg{\ng} \}$ be any finite subset of $\mathcal{X}$ and let $\Gg = \Gfun[\xbg{1}, ..., \xbg{\ng}]$, $\Ggmb = \Gfunmb[\xbg{1}, ..., \xbg{\ng}]$, $\evte = \ev_{\xbg{1}, ..., \xbg{\ng}}$ and $\hatevte = \hatev_{\xbg{1}, ..., \xbg{\ng}}$.
Let $\bar{\lambda}_1 = \| \Gg \|_2 \geq \bar{\lambda}_2 \geq \dots \geq \bar{\lambda}_{n'}$ be the eigenvalues of $\Gg$, and $\lambda_1 \geq \dots \geq \lambda_{n'}$ be the eigenvalues of $\Ggmb$. 
Let $\eta = \min_{k, l \in [n'], \bar{\lambda}_k \neq \bar{\lambda}_l} |\bar{\lambda}_k - \bar{\lambda}_l|$.

Recall that when $\alpha = 1 / 2$, $\hatevte(\mut{0}) = \rhoa \times \mathcal{N}(0, \Gg)$.
By the triangle inequality of $1$-Wasserstein distance, there is
\begin{equation}
\label{eq:triangle_1/2}
\begin{split}
    \mathcal{W}_1(\hatevte(\mumbt{0}), \hatevte(\mut{0})) =&~ \mathcal{W}_1(\hatevte(\mumbt{0}), \rhoa \times \mathcal{N}(0, \Gg)) \\
    \leq &~ \mathcal{W}_1(\rhoa \times \mathcal{N}(0, \Ggmb), \rhoa \times \mathcal{N}(0, \Gg)) \\
    & ~ + \mathcal{W}_1(\hatevte(\mumbt{0}), \rhoa \times \mathcal{N}(0, \Ggmb))~.
\end{split}
\end{equation}
First, we examine the first term on the right-hand side.  By the property of Wasserstein distances on product measures (e.g. \citealt[Lemma 3]{mariucci2018wasserstein}), we have
\begin{equation}
    \begin{split}
        \mathcal{W}_1(\rhoa \times \mathcal{N}(0, \Ggmb), \rhoa \times \mathcal{N}(0, \Gg))
        \leq & ~ \mathcal{W}_1(\rhoa, \rhoa) + \mathcal{W}_1(\mathcal{N}(0, \Ggmb), \mathcal{N}(0, \Gg)) \\
        \leq & ~ \mathcal{W}_1(\mathcal{N}(0, \Ggmb), \mathcal{N}(0, \Gg))~.
    \end{split}
\end{equation}
Before establishing an upper bound on the $1$-Wassertein distance between $\mathcal{N}(0, \Ggmb)$ and  $\mathcal{N}(0, \Gg)$, we first prove that the $\Ggmb$ and $\Gg$ are close in terms of eigen-decomposition.
\begin{lemma}
\label{lem:basis_near}
If $\| \Ggmb - \Gg \|_2 \leq \frac{1}{2} \eta$, then there exist eigen-decompositions of $\Gg$ and $\Ggmb$, $\Gg = \bar{V} \bar{\Lambda} \bar{V}^{\intercal}$ and $\Ggmb = V \Lambda V^{\intercal}$, where $\bar{V} = [\bar{\vb}_1, ..., \bar{\vb}_{n'}] \in \mathbb{R}^{n' \times n'}$ and $V = [\vb_1, ..., \vb_{n'}] \in \mathbb{R}^{n' \times n'}$ are both orthonormal matrices, and $\bar{\Lambda}$ and $\Lambda$ are both diagonal matrices, such that $\forall k \in [n']$, $\bar{\vb}_k^{\intercal} \cdot \vb_k \geq 1 - (\frac{2 \| \Ggmb - \Gg \|_2}{\eta})^2$.
\end{lemma}

\noindent \textit{Proof of Lemma~\ref{lem:basis_near}}:
Let $\Gg = \bar{U} \bar{\Sigma} \bar{U}^{\intercal}$ be any eigen-decomposition of $\Gg$, where the diagonal entries of $\bar{\Sigma}$ are sorted in non-ascending order. To account for the possible multiplicity of the eigenvalues, we can write $\bar{\Sigma}$ as a block-diagonal matrix $\mathrm{diag}(\bar{\Sigma}_1, ..., \bar{\Sigma}_p)$ with $p \leq n'$, where $\forall q \in [p], \bar{\Sigma}_q$ is a $d_q \times d_q$ diagonal matrix with all diagonal entries equal to some value $\zeta_q$, such that $\zeta_1 > ... > \zeta_p > 0$ and moreover, $\sum_{q=1}^p d_q = n'$. We then write $\bar{U} = [\bar{U}_1, ..., \bar{U}_p]$, where $\forall k \in [p], \bar{U}_q \in \mathbb{R}^{n' \times d_q}$.

Meanwhile, let $\Ggmb = U \Sigma U^{\intercal}$ be any eigen-decomposition of $\Ggmb$, where the diagonal entries are sorted in non-ascending order. Like with $\bar{\Sigma}$ and $\bar{U}$, we can also write $\Sigma = \mathrm{diag}(\Sigma_1, ..., \Sigma_p)$ and $U = [U_1, ..., U_p]$, where $\forall q \in [p], \Sigma_q \in \mathbb{R}^{d_q \times d_q}$ and $U_q \in \mathbb{R}^{n' \times d_q}$. Note that unlike in $\bar{\Sigma}_q$, each $\Sigma_q$ does not necessarily have all its diagonal entries equal.

By the definition of $\eta$, we know that $\forall q, q' \in [p]$ such that $q \neq q'$, there is $|\zeta_q - \zeta_q'| \geq \eta$. By Weyl's inequality for the eigenvalues of perturbed symmetric matrices, we know that $\forall p \in [n'], \| \bar{\Sigma}_p - \Sigma_p \|_2 \leq \| \Ggmb - \Gg \|_2$. As a result, if $\| \Ggmb - \Gg \|_2 < \frac{1}{2} \eta$, then $\forall q, q' \in [p]$ such that $q \neq q'$, we know that $\forall r \in [d_q], \forall r' \in [d_{q'}]$, there is $| (\bar{\Sigma}_q)_{rr} - (\Sigma_{q'})_{r'r'}| < \frac{1}{2} \eta$. Then, applying the ``$\sin\theta$ theorem'' of Davis-Kahan \citep{davis1970kahan}, we know that $\forall q \in [p]$, the $d_q \times d_q$ matrix $\bar{U}_q^{\intercal} \cdot U_q$ admits a singular value decomposition $E_q \cdot \mathrm{diag}(\cos(\thetab_q)) \cdot F_q^{\intercal}$, where $E_q, F_q \in \mathbb{R}^{d_q \times d_q}$ are orthonormal matrices and $\thetab \in \mathbb{R}^{d_q}$ with each entry in $[0, \frac{\pi}{2}]$, which satisfies 
\begin{equation}
    \| \sin(\thetab_q) \|_{\infty} \leq \frac{2 \| \Ggmb - \Gg \|_2}{\eta}~,
\end{equation}
where the $\cos$ and $\sin$ functions are applied entry-wise to the vector $\thetab$. Thus, since the entries of $\thetab$ are in $[0, \frac{\pi}{2}]$, we know that $\| 1 - \cos(\thetab_q) \|_{\infty} \leq \| 1 - \cos^2(\thetab_q) \|_{\infty} \leq \| \sin^2(\thetab_q) \|_{\infty} \leq (\frac{2 \| \Ggmb - \Gg \|_2}{\eta})^2$. Defining $\bar{V}_q = \bar{U}_q \cdot E_q$ and $V_q = U_q \cdot F_q$, we then have
\begin{equation}
    \bar{V}_q^{\intercal} \cdot V_q = E_q^{\intercal} \cdot E_q \cdot \mathrm{diag}(\cos(\thetab_q)) \cdot F_q^{\intercal} \cdot F_q = \mathrm{diag}(\cos(\thetab_q))~.
\end{equation}
Thus, writing $\bar{V} = [\bar{V}_1, ..., \bar{V}_p]$ and $V = [V_1, ..., V_p] \in \mathbb{R}^{n' \times n'}$, $\bar{\Lambda} = \mathrm{diag}(E_1^{\intercal} \cdot \bar{\Sigma}_1 \cdot E_1, ..., E_p^{\intercal} \cdot \bar{\Sigma}_p \cdot E_p)$ and $\Lambda = \mathrm{diag}(F_1^{\intercal} \cdot \Sigma_1 \cdot F_1, ..., F_p^{\intercal} \cdot \Sigma_p \cdot F_p)$, we see that 
\begin{equation}
\begin{split}
    \Gg = \bar{U} \cdot \bar{\Sigma} \cdot \bar{U}^{\intercal}
    =&~ \sum_{q=1}^p \bar{U}_q \cdot \bar{\Sigma}_q \cdot \bar{U}_q^{\intercal} \\
    =&~ \sum_{q=1}^p (\bar{U}_q \cdot E_q) \cdot (E_q^{\intercal} \cdot \bar{\Sigma}_q \cdot E_q) \cdot (E_q^{\intercal} \cdot \bar{U}_q^{\intercal}) \\
    =&~ \sum_{q=1}^p \bar{V}_q \cdot (E_q^{\intercal} \cdot \bar{\Sigma}_q \cdot E_q) \cdot \bar{V}_q ~ = \bar{V} \cdot \bar{\Lambda} \cdot \bar{V}^{\intercal}~,
\end{split}
\end{equation}
and similarly, $\Ggmb = V \cdot \Lambda \cdot V^{\intercal}$,
which give eigen-decompositions of $\Gg$ and $\Ggmb$. In particular, $\forall k \in [n']$, if $\bar{\vb}_k$ and $\vb_k$ are the $k$th columns of $\bar{V}$ and $V$, respectively, then we have $|1 - \bar{\vb}_k^{\intercal} \cdot \vb_k| \leq (\frac{2 \| \Ggmb - \Gg \|_2}{\eta})^2$. This proves the lemma.
\BlackBox \\

\noindent
With this lemma, we can prove an upper-bound on the $1$-Wasserstein distance between $\mathcal{N}(0, \Gg)$ and $\mathcal{N}(0, \Ggmb)$:
\begin{lemma}
\label{lem:w_gamma_bargamma}
If $\| \Ggmb - \Gg \|_2 < \frac{1}{2} \eta$, then
\begin{equation}
    \mathcal{W}_1(\mathcal{N}(0, \Ggmb), \mathcal{N}(0, \Gg)) \leq
    \sqrt{n' \left (\| \Ggmb - \Gg \|_2 + \frac{8 \| \Gg \|_2 \| \Ggmb - \Gg \|_2^2}{\eta^2} + \frac{8 \| \Ggmb - \Gg \|_2^3}{\eta^2} \right )}~.
\end{equation}
\end{lemma}
\begin{proof}
Using the eigen-decompositions of $\Gg$ and $\Ggmb$ constructed in the proof of Lemma~\ref{lem:basis_near}, we can apply Lemma~2.4 from \citet{chafai2010fine} to bound the $1$-Wasserstein distance between $\mathcal{N}(0, \Gg)$ and $\mathcal{N}(0, \Ggmb)$:
\begin{equation}
    \begin{split}
        &~ \mathcal{W}_1(\mathcal{N}(0, \Ggmb), \mathcal{N}(0, \Gg)) \\
        \leq &~ \sqrt{ \sum_{k=1}^{n'} (\sqrt{\bar{\lambda}_k} - \sqrt{\lambda_k})^2 + 2 \sqrt{\bar{\lambda}_k \lambda_k} (1 - \bar{\vb}_k^{\intercal} \cdot \vb_k)} \\
        \leq &~ \sqrt{\sum_{k=1}^{n'} |\bar{\lambda}_k - \lambda_k | + 2 \max\{\bar{\lambda}_k, \lambda_k \} (1 - \bar{\vb}_k^{\intercal} \cdot \vb_k)} \\
        \leq &~ \sqrt{n' \left ( \| \Ggmb - \Gg \|_2 + 2 (\| \Gg \|_2 + \| \Ggmb - \Gg \|_2) (\frac{2 \| \Ggmb - \Gg \|_2}{\eta})^2 \right )}
    \end{split}
\end{equation}
\end{proof}

\noindent Next, we look at the second term on the right-hand side of \eqref{eq:triangle_1/2}. We see that, when conditioned on $\zb_1, ..., \zb_{m_1}$, the collection $\{ [a^0, h_i^0(\xbg{1}), ..., h_i^0(\xbg{\ng})] \}_{i \in [m_2]}$ is distributed i.i.d. across $i \in [m_2]$ according to $\rhoa \times \mathcal{N}(0, \Ggmb)$. Hence, when conditioned on $\Ggmb$, which is measurable with respect to $\zb_1, ..., \zb_{m_1}$, $\hatevte(\mumbt{0})$ has the same distribution as the empirical measure of $m_2$ i.i.d. samples from $\rho^0_a \times \mathcal{N}(0, \Ggmb)$, which we denote by $\nu_{(m_2)} \in \mathcal{P}(\mathbb{R} \times \mathbb{R}^{n'})$. Therefore, by conditioning on $\Ggmb$, we can again leverage concentration inequalities of empirical measures of i.i.d. samples in Wasserstein distance, as given by Lemma~\ref{lem:wconcen}.
In particular, we choose $d = n'+1$, $\nu = \rho^0_a \times \mathcal{N}(0, \Ggmb)$ and choose $\alpha = 2$, $\gamma = \frac{1}{2 \lambda_{\max}(\Ggmb)}$. Recalling that $\mathcal{N}(0, \Ggmb)$ is also the distribution $(\Ggmb)^{\frac{1}{2}} \cdot \ub$, where each entry of $\ub \in \Rbb^n$ is independently distributed as $\mathcal{N}(0, 1)$, we can then write
\begin{equation}
    \begin{split}
        \mathcal{E}_{\alpha, \gamma}(\nu) =&~ \int_{\Rbb} \frac{1}{(2 \pi)^{\frac{\ng}{2}}} \int_{\Rbb^{\ng}} e^{\gamma (a^2 + \| (\Ggmb)^{\frac{1}{2}} \cdot \ub \|_2^2)^{\frac{\alpha}{2}}} e^{-\| \ub \|_2^2} d \ub \rhoa(da) \\
        \leq &~ e^{\gamma (a^0_{\max})^2} \frac{1}{(2 \pi)^{\frac{\ng}{2}}} \int_{\Rbb^{\ng}} e^{(\gamma \lambda_{\max}(\Ggmb) - 1) \| \ub \|_2^2} d \ub \\
        \leq &~ e^{\gamma (a^0_{\max})^2} \frac{2^{\frac{\ng}{2}}}{(2 \pi \cdot 2)^{\frac{\ng}{2}}} \int_{\Rbb^{\ng}} e^{- \| \ub \|_2^2 / 2} d \ub \\
        \leq &~ 2^{\frac{\ng}{2}} e^{(a^0_{\max})^2 / (2 \lambda_{\max}(\Ggmb))} < \infty~.
    \end{split}
\end{equation}
Moreover, for $u > 0$, $\log(2 + \frac{1}{u}) < 1 + \frac{1}{u}$, and hence $\frac{u}{\log(2 + \frac{1}{u})} \geq \frac{u^2}{u + 1} \geq u^2$. Therefore, applying Lemma~\ref{lem:wconcen}, we have $\forall u > 0$, $\exists C_1, C_2 > 0$ such that
\begin{equation}
\begin{split}
    \mathbb{P} \left (\mathcal{W}_1(\hatevte(\mumbt{0}), \rhoa \times \mathcal{N}(0, \Ggmb)) \geq u ~ \Big |~ \Ggmb \right ) =& \mathbb{P} \left (\mathcal{W}_1(\nu_{(m_2)}, \nu) \geq u ~\Big | ~ \Ggmb \right ) \\
    \leq & C_1 e^{-C_2 m_2 u^{\max\{n'+1, 4\}}}~,
\end{split}
\end{equation}
where $C_1$ and $C_2$ depend on $n'$, $\atmax{0}$ and $\lambda_{\max}(\Ggmb)$. 

Furthermore, if we condition on the event that $\| \Ggmb - \Gg \|_2 < \Delta$ for any $\Delta \in (0, \lambda_{\max}(\Gg)]$ -- which is measurable with respect to $\Ggmb$ -- then by choosing $\iota = 2$ and $\gamma = \frac{1}{4 \lambda_{\max}(\Gg) )}$, we have $\mathcal{E}_{\alpha, \gamma}(\nu) \leq 
2^{\frac{\ng}{2}} e^{(a^0_{\max})^2 / (2 \lambda_{\max}(\Ggmb))} < \infty$.
Therefore, $\forall u > 0$, $\exists C_1, C_2 > 0$ depending only on $n'$, $\atmax{0}$ and $\lambda_{\max}(\Gg)$ such that,
\begin{equation}
\label{eq:wconcen_}
    \mathbb{P} \left (\mathcal{W}_1(\hatevte(\mumbt{0}), \rhoa \times \mathcal{N}(0, \Ggmb)) \geq u \bigg | \| \Ggmb - \Gg \|_2 < \Delta \right ) \leq C_1 e^{-C_2 m_2 u^{\max\{n'+1, 4\}}}~.
\end{equation}
Thus, our overall strategy is to control the first and second terms on the right-hand side of \eqref{eq:triangle_1/2} via Lemma~\ref{lem:w_gamma_bargamma} and \eqref{eq:wconcen_}, respectively, by restricting to the high-probability event that $\| \Ggmb - \Gg \|_2 < \Delta$ for some $\Delta > 0$. Specifically, we will use the following concentration result of $\Ggmb$:
\begin{lemma}[\citealt{chen2022on}, Lemma 4]
\label{lem:Gm_dev}
Let $\{ \xbg{1}, ..., \xbg{\ng} \}$ be any finite subset of $\mathcal{X}$. Let $\Gg = \Gfun[\xbg{1}, ..., \xbg{\ng}]$ and $\Ggmb = \Gfunmb[\xbg{1}, ..., \xbg{\ng}]$. $\exists C_3, C_4, C_5 > 0$, which depend on   $\mathtt{L}_{\sigma}$ and the sub-Gaussian norm of $\rhoz$ such that, $\forall \Delta > 0$
\begin{equation}
    \Pbb \left ( \| \Ggmb - \Gg \|_2 \geq \Delta \right ) < C_3 (\ng)^2 e^{- C_4 \min \{ \Delta, C_5 \Delta^2 \} m_1}~.
\end{equation}
\end{lemma}

\noindent Fix an $\epsilon > 0$. Define 
\begin{equation}
    \Delta_{\epsilon} = \min \left \{ \frac{1}{2} \eta, \frac{\epsilon^2}{12 n'}, \frac{\epsilon \eta}{(96 n' \lambda_{\max}(\Gg))^{\frac{1}{2}}}, \left ( \frac{\epsilon^2 \eta^2}{96 n'} \right )^{\frac{1}{3}} \right \}~.
\end{equation}
Then, conditioned on the event that $ \| \Ggmb - \Gg \|_2 \leq \Delta$, it holds that $ \| \Ggmb - \Gg \|_2 \leq \frac{1}{2} \eta $ and $\mathcal{W}_1(\mathcal{N}(0, \Ggmb), \mathcal{N}(0, \Gg)) \leq \frac{1}{2} \epsilon$.
Thus, putting things together,
\begin{equation}
\begin{split}
    &~ \Pbb \left ( \mathcal{W}_1 (\hatevte(\mumbt{0}), \hatevte(\mut{0})) > \epsilon \right ) \\
    \leq &~ 
    \Pbb \left ( \mathcal{W}_1(\hatevte(\mumbt{0}), \hatevte(\mut{0})) > \epsilon ~ \Big |~ \| \Ggmb - \Gg \|_2 < \Delta_{\epsilon} \right ) + \Pbb \left ( \| \Ggmb - \Gg \|_2 \geq \Delta_{\epsilon} \right ) \\
    \leq & \mathbb{P} \left (\mathcal{W}_1(\hatevte(\mumbt{0}), \rhoa \times \mathcal{N}(0, \Ggmb)) \geq \frac{1}{2} \epsilon ~ \Big |~ \| \Ggmb - \Gg \|_2 < \Delta_{\epsilon} \right ) + \Pbb \left ( \| \Ggmb - \Gg \|_2 \geq \Delta_{\epsilon} \right ) \\
    \leq & C_1 e^{-C_2 (\epsilon / 2)^{\max\{n'+1, 4\}} m_2} + C_3 (\ng)^2 e^{- C_4 \min \{ \Delta_{\epsilon}, C_5 (\Delta_{\epsilon})^2 \} m_1}~.
\end{split}
\end{equation}
Thus, with any pair of increasing $\Nbb_+$-valued sequences $\{m_{1, k}\}_{k \in \Nbb_+}$ and $\{m_{2, k}\}_{k \in \Nbb_+}$, denoting $\mb_k = (m_{1, k}, m_{2, k})$, there is
\begin{equation}
    \summ{k}{\infty} \Pbb \left ( \mathcal{W}_1 (\hatevte(\mumbkt{k}{0}), \hatevte(\mut{0})) > \epsilon \right ) < \infty~.
\end{equation}
Since this holds for any $\epsilon > 0$, the Borel-Cantelli lemma implies that 
\begin{equation}
    \lim_{k \to \infty} \mathcal{W}_1 (\hatevte(\mumbkt{k}{0}), \hatevte(\mut{0})) = 0~,
\end{equation}
almost surely, and hence $\hatevte(\mumbkt{k}{0})$ converges weakly to $\hatevte(\mut{0})$ almost surely.\\

\noindent \textbf{Part 2: Existence of $\mathcal{GP}(\zerob, \Gfun)$ as a probability measure on $\Con$}

Since the set of all given finite-dimensional distributions clearly satisfy the consistency conditions for a projective family of probability measures, the Kolmogorov extension theorem (e.g. \citet{kallenberg1997foundations}, Theorem 5.16) implies that there exists a random field with $\mathcal{X}$ being the index space, $\{B_{\xb}\}_{\xb \in \mathcal{X}}$, such that $\forall \xb_1, ..., \xb_{n'}$, the random vector $[B_{\xb_1}, ..., B_{\xb_{n'}}]$ is distributed as $\mathcal{N}(\mathbf{0}, \Gfun[\xb_1, ..., \xb_{n'}])$.

It remains to apply the Kolmogorov-Chentsov continuity theorem (e.g. \citealt[Theorem 2.23]{kallenberg1997foundations}) to prove that there exists a continuous version of $B$. Note that $\forall \xb_1, \xb_2 \in \mathcal{X}$, $B_{\xb_1} - B_{\xb_2}$ follows a Gaussian distribution with mean zero and variance
\begin{equation}
\begin{split}
    \text{Var}(B_{\xb_1} - B_{\xb_2}) =& \Gfun(\xb_1, \xb_1) + \Gfun(\xb_2, \xb_2) - \Gfun(\xb_2, \xb_2) \\
    =& \EE_{\zb \in \rhoz} \left [ \sigma_1(\zb^{\intercal} \xb_1) \sigma_1(\zb^{\intercal} \xb_1) + \sigma(\zb^{\intercal} \xb_2) \sigma_1(\zb^{\intercal} \xb_2) - 2 \sigma_1(\zb^{\intercal} \xb_1) \sigma_1(\zb^{\intercal} \xb_2) \right ] \\
    =& \EE_{\zb \in \rhoz} \left [ \left ( \sigma_1(\zb^{\intercal} \xb_1) - \sigma_1(\zb^{\intercal} \xb_2) \right )^2 \right ] \\
    \leq & \mathtt{M}_{\sigma_2} \EE_{\zb \in \rhoz} \left [ \left ( \zb^{\intercal} (\xb_1 - \xb_2) \right )^2 \right ] \\
    \leq & \mathtt{M}_{\sigma_2} \| \rhoz \|_{\text{SG}} \| \xb_1 - \xb_2 \|^2~,
\end{split}
\end{equation}
where $\| \rhoz \|_{\text{SG}} < \infty$ is the sub-Gaussian norm of $\rhoz$ \citep{vershynin_2018}. Thus, $\forall p \in \mathbb{N}_+$, 
\begin{equation}
    \mathbb{E}\left [ \left | B_{\xb_1} - B_{\xb_2} \right |^{2p} \right ] \leq (p-1)!! \left ( \text{Var}(B_{\xb_1} - B_{\xb_2}) \right )^p \\
    C_p \| \xb_1 - \xb_2 \|^{2p}~,
\end{equation}
with some constant $C_p > 0$. Therefore, by the Kolmogorov-Chentsov continuity theorem, there exists a version of $B$ whose samples paths are locally H\"{o}lder continuous with exponent $\frac{2p-d}{2p}$. In fact, since this argument applies to all $p \in \mathbb{N}_+$, we know that $\forall \alpha \in [0, 1)$, there exists a version of $B$ whose samples paths are locally H\"{o}lder continuous with exponent $\alpha$. In particular, there exists a version of $B$ whose samples paths are continuous, since H\"{o}lder continuity with any exponent $\alpha > 0$ implies uniform continuity. Then, the law of sample paths of such a $B$ is indeed supported on $\Con$.

\section{Proof of Lemma~\ref{lem:mf_exist_1/2}}
\label{app:pf_mf_exist_1/2}
The dynamics of $\muhidt{t}$ is a Wasserstein gradient flow on finite-dimensional Euclidean space, whose existence has been proved in prior works such as \citet{braun1977vlasov, sirignano2020mean_lln, mei2018mean}. Below, we prove that the characteristic flow maps $\At{t}$ and $\Ht{t}$ constructed from $\muhidt{t}$ via \eqref{eq:AtfromCt_Hilb} and \eqref{eq:HtfromLambt_Con} indeed satisfy \eqref{eq:dotAt_Con} and \eqref{eq:dotHt_Con}. 

First, as an intermediate step, we construct a candidate for $\hatevtrpf{\mut{t}}$ from $\muhidt{t}$.
For $t \geq 0$, we define two maps, $\Atrit{t}: \Rbb \times \Rbb^n \to \Rbb$ and $\Ubtrit{t} = [\Utritk{t}{1}, ..., \Utritk{t}{n}]: \Rbb \times \Rbb^n \to \Rbb^n$, by
\begin{align}
    \Atrit{t}(a, \ub) =&~ \Ct{t}(a, \Gmhalf \cdot \ub) ~,\label{eq:Atrit_constr}\\
        \Ubtrit{t}(a, \ub) =&~ \Ghalf \cdot \Lambt{t}(a, \Gmhalf \cdot \ub) ~, \label{eq:Ubtrit_constr}
\end{align}
for $(a, \ub) \in \supp(\mutrit{0})$. 
We let $\Thetabtrit{t} = [ \Atrit{t}, \Ubtrit{t} ]: \Rbb \times \Rbb^n \to \Rbb \times \Rbb^n$, and want to show that
\begin{align}
    \Atrit{0}(a, \ub) =&~ a~,~ \quad \Ubtrit{0}(a, \ub) = \ub ~, \\
        \frac{d}{dt} \Atrit{t}(a, \ub) =&~ -\frac{1}{n} \summ{k}{n} \sigtbig{\Utritk{t}{k}(a, \ub)} \zetak{k}{\ftritk{t}{k}}~, \label{eq:dotAtri} \\
    \frac{d}{dt} \Utritk{t}{k}(a, \ub) =& -\frac{1}{n}  \Atrit{t}(a, \ub) \summ{l}{n} \sigtpbig{\Utritk{t}{l}(a, \ub)} \zetak{l}{\ftritk{t}{l}} G_{k, l}~, \label{eq:dotUtri}
\end{align}
if we define $\mutrit{t} = (\Thetabtrit{t})_{\#} \mutrit{0}$ and $\ftritk{t}{k} = \int_{\Rbb \times \Rbb^n} a \sigtbig{ u_k } \mutrit{t}(da, d\ub)$.
First, there is
\begin{equation}
\begin{split}
    \ftritk{t}{k} =&~ \int_{\Rbb \times \Rbb^n} \Atrit{t}(a, \ub) \sigtbig{\Utritk{t}{k}(a, \ub)} \mutrit{0}(da, d\ub) \\
    =&~ \int_{\Rbb \times \Rbb^n}  \Ct{t}(a, \Gmhalf \cdot \ub) \sigtbig{\big (\Ghalf \cdot \Lambt{t}(a, \Gmhalf \cdot \ub) \big )_k} \mutrit{0}(da, d\ub) \\
    =&~ \int_{\Rbb \times \Rbb^n}  \Ct{t}(a, \lambb) \sigtbig{\Lambt{t}(a, \lambb)^{\intercal} \cdot \xtildetrk{k}} \muhidt{0}(da, d\lambb) \\
    =&~ \int_{\Rbb \times \Rbb^n}  a \sigtbig{\lambb^{\intercal} \cdot \xtildetrk{k}} \muhidt{t}(da, d\lambb) ~= \falt_t(\xtildetrk{k})~.
\end{split}
\end{equation}
Recall that $\mutrit{0} = \rhoa \times \mathcal{N}(0, G)$ if $\alpha = 1 / 2$ and $\rhoa \times \delta_{\zerob}$ if $\alpha > \frac{1}{2}$. Hence, in either case, if $(a, \ub) \in \supp(\mutrit{0})$, then $\ub$ belongs to the range of $\Ghalf$, which implies that $\Ghalf \cdot \Gmhalf \cdot \ub = \ub$. Thus, for any $(a, \ub) \in \supp(\mutrit{0})$, there is $\Atrit{0}(a, \ub) = \Ct{0}(a, \Gmhalf \cdot \ub) = a$ and $\Ubtrit{0}(a, \ub) = \Ghalf \cdot \Lambt{0}(a, \Gmhalf \cdot \ub) = \Ghalf \cdot \Gmhalf \cdot \ub = \ub$. Moreover, it holds that
\begin{equation}
    \begin{split}
        \frac{d}{dt} \Atrit{t}(a, \ub) =&~ \frac{d}{dt} \Ct{t}(a, \Gmhalf \cdot \ub) \\
        =&~ -\frac{1}{n} \summ{k}{n} \zetak{k}{g_t(\xtildetrk{k})} \sigtbig{ \Lambt{t}(a, \Gmhalf \cdot \ub)^{\intercal} \cdot \xtildetrk{k} } \\
        =& -\frac{1}{n} \summ{k}{n} \zetak{k}{\ftritk{t}{k}} \sigtbig{\Utritk{t}{k}(a, \ub)}~,
    \end{split}
\end{equation}
and
\begin{equation}
    \begin{split}
        \frac{d}{dt} \Utritk{t}{l}(a, \ub) =&~ \left ( \Ghalf \cdot \frac{d}{dt} \Lambt{t}(a, \Gmhalf \cdot \ub) \right )_l \\
        =&~ -\frac{1}{n} \Ct{t}(a, \Gmhalf \cdot \ub) \summ{k}{n} \zetak{k}{g_t(\xtildetrk{k})} \sigtpbig{\Lambt{t}(a, \Gmhalf \cdot \ub)^{\intercal} \cdot \xtildetrk{k}} (\Ghalf \cdot\xtildetrk{k})_l \\
        =&~ -\frac{1}{n} \Atrit{t}(a, \ub) \summ{k}{n} \zetak{k}{\ftritk{t}{k}} \sigtpbig{\Utritk{t}{k}(a, \ub)} G_{k, l}~,
    \end{split}
\end{equation}
which verify \eqref{eq:dotAtri} and \eqref{eq:dotUtri}. In addition, 
\begin{equation}
    \frac{d}{dt} \Ubtrit{t}(a, \ub) = - \frac{1}{n} \Atrit{t}(a, \ub) ~ G \cdot \left [ \zetak{k}{\ftritk{t}{k}} \sigtpbig{\Utritk{t}{k}} \right ]_{k=1}^n~,
\end{equation}
and hence
\begin{equation}
\label{eq:Ubtrit_int}
    \begin{split}
        \Ubtrit{t}(a, \ub) =&~ \ub - G \cdot \frac{1}{n} \int_0^t \Atrit{s}(a, \ub) \left [ \zetak{k}{\ftritk{s}{k}} \sigtpbig{\Utritk{s}{k}(a, \ub)} \right ]_{k=1}^n ds
    \end{split}
\end{equation}
belongs to the range of $G$ for all $t \geq 0$. We also observe from \eqref{eq:Atrit_constr} and \eqref{eq:Ubtrit_constr} that for $(a, \ub) \in \mutrit{0}$,
\begin{equation}
    \Gmhalf ( \Thetabtrit{t}(a, \ub)) = \Lambt{t}( \Gmhalf (a, \ub))~,
\end{equation}
and therefore,
\begin{equation}
    \muhidt{t} = (\Lambt{t})_{\#} (\Gmhalf)_{\#} \mutrit{0} = (\Gmhalf)_{\#} (\Thetabtrit{t})_{\#} \mutrit{0} = (\Gmhalf)_{\#} \mutrit{t}~.
\end{equation}

Next, we will construct $\mut{t}$ from $\mutrit{t}$, by defining, for $(a, h) \in \supp(\mut{0})$,
\begin{equation}
\label{eq:At_constr}
    \begin{split}
        \Atah{t} =&~ \Atrit{t}(a, \evtr(h)) = \Ct{t}(a, \Gmhalf \cdot \evtr(h)) ~,
    \end{split}
\end{equation}
and
\begin{equation}
\label{eq:Ht_constr}
    \begin{split}
        \Htah{t} =&~ h + \sum_{k=1}^n \left ( \Gm \cdot \left (\Ubtrit{t}(a, \evtr(h)) - \evtr(h) \right ) \right )_k \Gfun(\xbtrk{k}, \cdot) \\
        =&~ h + \sum_{k=1}^n \left ( \Gmhalf \cdot \left (\Lambt{t}(a, \Gmhalf \cdot \evtr(h)) - \Gmhalf \cdot \evtr(h) \right ) \right )_k \Gfun(\xbtrk{k}, \cdot)~.
    \end{split}
\end{equation}
We first check that,
\begin{align}
    \Atah{0} =&~ \Atrit{0}(a, \evtr(h)) = a \\
    \Htah{0} =&~ h + \sum_{k=1}^n \left ( \Gm \cdot \left (\Ubtrit{0}(a, \evtr(h)) - \evtr(h) \right ) \right )_k \Gfun(\xbtrk{k}, \cdot) = h~.
\end{align}
Next, for all $(a, h) \in \supp(\mut{0})$, $\evtr(h)$ belongs to the range of $G$, and thus \eqref{eq:Ubtrit_int} implies that $\Ubtrit{t}(a, \evtr(h))$ belongs to the range of $G$ as well. Therefore, $\forall k \in [n], t \geq 0$,
\begin{equation}
\label{eq:Ht_Ut_on_tr}
\begin{split}
    \Htah{t}(\xbtrk{k}) =&~ h(\xbtrk{k}) + \left (G \cdot \Gm \cdot \left (\Ubtrit{t}(a, \evtr(h)) - \evtr(h) \right ) \right )_k \\
    =&~ \Utritk{t}{k}(a, \evtr(h))~,
\end{split}
\end{equation}
Moreover,
\begin{equation}
    \begin{split}
        \ft{t}(\xbtrk{k}) =&~ \int_{\Rbb \times \Con} \Atah{t} \sigtbig{\Htah{t}(\xbtrk{k})} \mut{0}(da, dh) \\
        =&~ \int_{\Rbb \times \Con} \Atrit{t}(a, \evtr(h)) \sigtbig{\Utritk{t}{k}(a, \evtr(h))} \mut{0}(da, dh) \\
        =&~ \int_{\Rbb \times \Rbb^n} \Atrit{t}(a, \ub) \sigtbig{\Utritk{t}{k}(a, \ub)} \mutrit{0}(da, d\ub) 
        = \ftritk{t}{k}~,
    \end{split}
\end{equation}
and hence,
\begin{equation}
\label{eq:ddtA_app}
\begin{split}
    \frac{d}{dt} \Atah{t} = \frac{d}{dt} \Atrit{t}(a, \evtr(h)) =&~ -\frac{1}{n} \summ{k}{n} \zetak{k}{\ftritk{t}{k}} \sigtbig{\Utritk{t}{k}(a, \evtr(h))} \\
    =&~ -\frac{1}{n} \summ{k}{n} \zetak{k}{\ft{t}(\xbtrk{k})} \sigtbig{\Htah{t}(\xbtrk{k})}~,
\end{split}
\end{equation}
and
\begin{equation}
\label{eq:ddtH_app}
    \begin{split}
        \frac{d}{dt} \Htah{t} =&~ \summ{k}{n} \Gfun(\xbtrk{k}, \cdot ) \left ( \Gm \cdot \frac{d}{dt} \Ubtrit{t}(a, \evtr(h)) \right )_k \\
        =&~ - \summ{k}{n} \Gfun(\xbtrk{k}, \cdot ) \left (\frac{1}{n} \Atrit{t}(a, \evtr(h))~ \Gm \cdot  G \cdot \left [ \zetak{l}{\ftritk{t}{l}} \sigtpbig{\Utritk{t}{l}(a, \evtr(h))} \right ]_{l=1}^n \right )_k \\
        =&~ - \frac{1}{n} \Atrit{t}(a, \evtr(h)) \summ{k}{n} \zetak{k}{\ftritk{t}{k}} \sigtpbig{\Utritk{t}{k}(a, \evtr(h))} \Gfun(\xbtrk{k}, \cdot ) \\
        =&~ - \frac{1}{n} \Atah{t} \summ{k}{n} \zetak{k}{\ft{t}(\xbtrk{k})} \sigtpbig{\Htah{t}(\xbtrk{k})} \Gfun(\xbtrk{k}, \cdot )~,
    \end{split} 
\end{equation}
which verify \eqref{eq:dotAt_Con} and \eqref{eq:dotHt_Con}. This proves the existence of $\mut{t}$.

Furthermore, \eqref{eq:At_constr} and \eqref{eq:Ht_Ut_on_tr} imply that for $(a, h) \in \supp(\mut{0})$, there is
\begin{equation}
\label{eq:proj_evol_commute}
    \hatevtr(\Thetabt{t}(a, h)) = \Thetabtrit{t}(\hatevtr(a, h))~.
\end{equation}
This implies that, $\forall t \geq 0$, $\mutrit{t} = (\Thetabtrit{t})_{\#} \hatevtrpf{\mut{0}} = \hatevtrpf{(\Thetabt{t})_{\#} \mut{0}} = \hatevtrpf{\mut{t}}$,
and hence also $\muhidt{t} = (\Gmhalf)_{\#} \mutrit{t} = (\Gmhalf)_{\#} \hatevtrpf{\mut{t}}$.
Therefore,
\begin{equation}
    \begin{split}
        \ft{t}(\xb)
        =&~ \int_{\Rbb \times \Con} \Atah{t} \sigtbig{\Htah{t}(\xb)} \mut{0}(da, dh) \\
        =&~ \int_{\Rbb \times \Rbb^n \times \Rbb} \Atrit{t}(a, \ub) \sigma \bigg ( v + \sum_{k=1}^n \Big ( \Gmhalf \cdot \big (\Lambt{t}(a, \Gmhalf \cdot \ub) - \Gmhalf \cdot \ub \big ) \Big )_k \Gfun(\xbtrk{k}, \cdot) \bigg ) \\
        & \hspace{270pt} \big ( (\hat{\ev}_{\xbtrk{1}, ..., \xbtrk{n}, \xb})_{\#}\mut{0} \big )(da, d\ub, dv) \\
        =&~ \int_{\Rbb \times \Rbb^n} \Atrit{t}(a, \ub) \EE_{Z \sim \mathcal{N}(0, 1)} \Big [ \sigma \Big ( \tau(\xb) Z + \sum_{k=1}^n \big ( \Gmhalf \cdot \Lambt{t}(a, \Gmhalf \cdot \ub)  \big )_k \Gfun(\xbtrk{k}, \cdot) \Big ) \Big ] \mutrit{0}(da, d\ub) \\
        =&~ \int_{\Rbb \times \Rbb^n} \Ct{t}(a, \ub) \EE_{Z \sim \mathcal{N}(0, 1)} \Big [ \sigma \Big ( \tau(\xb) Z + \sum_{k=1}^n \big ( \Gmhalf \cdot \Lambt{t}(a, \lambb) \big )_k \Gfun(\xbtrk{k}, \cdot) \Big ) \Big ] \muhidt{0}(da, d\lambb) \\
        =&~ \int_{\Rbb \times \Rbb^n} a \EE_{Z \sim \mathcal{N}(0, 1)} \left [ \sigtbig{\tau(\xb) Z + \lambb^{\intercal} \cdot \Xtil(\xb)} \right ] \muhidt{t}(da, d\lambb)~.
    \end{split}
\end{equation}

\section{Proof of Lemma~\ref{lem:lln_t_tr}}
\label{app:pf_lln_t_tr}
We define $\mumbtrit{t} = \hatevtrpf{\mumbt{t}}$. The goal then is to provide an upper bound for $\mathcal{W}_1(\mutrit{t}, \mumbtrit{t})$. Since $\mumbt{t}$ is obtained via the push-forward of $\Thetabtmb{t}$, which satisfies \eqref{eq:dotAtmb_Con} and \eqref{eq:dotHtmb_Con}, we see that $\mumbtrit{t}$ can be written as $\mumbtrit{t} = (\Thetabmbtrit{t})_{\#} \mumbtrit{0}$, where $\Thetabmbtrit{t} = [\Ambtrit{t}, \Ubmbtrit{t}]: \Rbb \times \Rbb^n \to \Rbb \times \Rbb^n$ evolve according to
\begin{align}
        \frac{d}{dt} \Ambtrit{t}(a, \ub) =&~ -\frac{1}{n} \summ{k}{n} \sigtbig{\Umbtritk{t}{k}(a, \ub)} \zetak{k}{\fmbtritk{t}{k}}~, \label{eq:dotAmbtri} \\
    \frac{d}{dt} \Umbtritk{t}{k}(a, \ub) =& -\frac{1}{n}  \Ambtrit{t}(a, \ub) \summ{k}{n} \sigtpbig{\Umbtritk{t}{k}(a, \ub)} \zetak{k}{\fmbtritk{t}{k}} \Gmb_{k, l}~, \label{eq:dotUmbtri}
\end{align}
with $\Ambtrit{0}(a, \ub) = a$ and $\Umbtritk{0}{k}(a, \ub) = u_k$, and where $\fmbtritk{t}{k} = \int_{\Rbb \times \Rbb^n} a \sigtbig{ u_k } \mumbtrit{t}(da, d\ub)$. Thus, our strategy is to use the triangle inequality of $1$-Wasserstein distance to write
\begin{equation}
\label{eq:pc_triangle_tr}
    \begin{split}
        \mathcal{W}_1(\mutrit{t}, \mumbtrit{t}) =&~ \mathcal{W}_1((\Thetabtrit{t})_{\#}\mutrit{0}, (\Thetabmbtrit{t})_{\#}\mumbtrit{0}) \\
        \leq&~ \mathcal{W}_1((\Thetabtrit{t})_{\#}\mutrit{0}, (\Thetabtrit{t})_{\#}\mumbtrit{0}) + \mathcal{W}_1((\Thetabmbtrit{t})_{\#}\mumbtrit{0}, (\Thetabtrit{t})_{\#}\mumbtrit{0}) \\
        =&~ \mathcal{W}_1(\mutrit{t}, \tilmumbtrit{t}) + \mathcal{W}_1(\mumbtrit{t}, \tilmumbtrit{t})~,
    \end{split}
\end{equation}
where we define $\tilmumbtrit{t} = (\Thetabtrit{t})_{\#}\mumbtrit{0}$.

To bound the first term on the right-hand side of \eqref{eq:pc_triangle_tr}, we use the following inequality:
\begin{equation}
    \mathcal{W}_1(\mutrit{t}, \tilmumbtrit{t}) \leq \mathcal{W}_1((\Thetabtrit{t})_{\#}\mutrit{0}, (\Thetabtrit{t})_{\#}\mumbtrit{0}) \leq \mathrm{Lip}(\Thetabtrit{t}) \mathcal{W}_1(\mutrit{0}, \mumbtrit{0})~.
\end{equation}
To bound $\mathrm{Lip}(\Thetabtrit{t})$, we need the following lemma:
\begin{lemma}
\label{lem:blowup_tr}
For $n \in \mathbb{N}_+$ and $t \geq 0$, there exists $C_1(n, t)$ and $C_2(n, t)$ that are non-negative and non-decreasing in $t$ such that $\forall t \geq 0$, 
$\forall a \in \supp(\rhoa), \ub \in \Rbb^n$,
\begin{align}
    |\Atrit{t}(a, \ub)| \leq C_1(n, t) ~, \qquad
    |\Ambtrit{t}(a, \ub)| \leq  C_1(n, t)
\end{align}
and for all $\xb \in \domX$,
\begin{align}
    \sup_{k \in [n]} \left | \ft{t}(\xbtrk{k})\right | \leq  C_2(n, t) ~, \qquad
    \sup_{k \in [n]} \left | \fmbt{t}(\xbtrk{k}) \right | \leq  C_2(n, t)
\end{align}
\end{lemma}
\begin{proof}
There is
\begin{equation}
    \begin{split}
        |\ftritk{t}{k}| \leq&~ \mathtt{M}_{\sigma_2} \int_{\Rbb \times \Rbb^n} \left | \Atrit{t}(a, \ub) \right | \mutrit{t}(da, d\ub) \\
        \leq&~ \mathtt{M}_{\sigma_2} \sup_{a \in \supp(\rhoa), \ub \in \Rbb^n} \left | \Atrit{t}(a, \ub) \right |~.
    \end{split}
\end{equation}
Then,
\begin{equation}
\begin{split}
    \left | \Atrit{t}(a, \ub) \right | \leq &~ |a| + \int_0^t \frac{1}{n} \summ{k}{n} \mathtt{M}_{\sigma_2} (|\ftritk{s}{k}| + |y_k|) ds \\
    \leq &~ |a| + t \mathtt{M}_{\sigma_2} y_k + (\mathtt{M}_{\sigma_2})^2 \int_0^t \sup_{a \in \supp(\rhoa), \ub \in \Rbb^n} \left | \Atrit{s}(a, \ub) \right | ds~.
\end{split}
\end{equation}
Thus, by Gr\"onwall's inequality, there exists $C_1(n, t)$ such that
\begin{equation}
    \sup_{a \in \supp(\rhoa),~ \ub \in \Rbb^n} \left | \Atrit{t}(a, \ub) \right | \leq~ C_1(n, t)~,
\end{equation}
and hence $\forall \xb \in \domX$, $|\ft{t}(\xb)| \leq  \mathtt{M}_{\sigma_2} C_1(n, t) =: C_2(n, t)$.

Similar arguments apply to $\sup_{a \in \supp(\rhoa),~ \ub \in \Rbb^n} \left | \Ambtrit{t}(a, \ub) \right |$ and $\fmbt{t}(\xb)$.
\end{proof}

\noindent 
Define the following ODE for $\zb(t) = [z_0(t), ..., z_n(t)]^{\intercal} \in \mathbb{R}^{n+1}$:
\begin{equation}
    \frac{d}{dt} \zb(t) = F(\zb(t))~,
\end{equation}
where $\forall l \in \{0, ..., n\}$,
\begin{equation}
    (F(\zb))_k =
    \begin{cases}
    - \sum_{l=1}^n \sigma_2(z_l) \zetak{l}{\ftritk{t}{l}} ~,~ k=0 \\
    - z_0 \sum_{l=1}^n {\sigma_2}'(z_l) \zetak{l}{\ftritk{t}{l}} G_{k,l}~,~ k \in [n]~.
    \end{cases}
\end{equation}
Then, $\Thetabtrit{t}: \mathbb{R} \times \Rbb^n \to \mathbb{R} \times \Rbb^n$ can be considered as the map from the initial condition $\zb(0)$ to the solution $\zb(t)$ at time $t$ of this ODE. Recall that the solutions of an ODE with a Lipschitz-continuous function on the right-hand side depends continuously on the initial condition.
Since within the interval $[0, t]$, the function $F$ is Lipschitz-continuous with Lipschitz constant
\begin{equation}
\mathrm{Lip}(F) 
\leq n \mathtt{M}_{\sigma_2} (C_2(n, t) + \| \yb \|_{\infty}) (1 + n C_1(n, t) \mathtt{M}_{\sigma_2} \mathtt{M}_{\sigma_2}) =: C_3'(n, t) < \infty~,
\end{equation}
we know that
\begin{equation}
    \mathrm{Lip}(\Thetabtrit{t}) \leq e^{t C_3'(n, t)} =: C_3(n, t) < \infty~.
\end{equation}
Thus,
\begin{equation}
    \begin{split}
         \mathcal{W}_1(\mutrit{t}, \tilmumbtrit{t}) \leq C_3(n, t) ~ \mathcal{W}_1(\mutrit{0}, \mumbtrit{0})~.
    \end{split}
\end{equation}

Next, we consider the second term on the right-hand side of \eqref{eq:pc_triangle_tr}. Define 
\begin{align}
    \Delta \Ambtrit{t} =&~ \int_{\Rbb \times \Rbb^n} \big |\Ambtrit{t}(a, \ub) - \Atrit{t}(a, \ub) \big | \mumbtrit{0}(da, d\ub)~, \\
    \Delta \Umbtrit{t} =&~ \int_{\Rbb \times \Rbb^n} \big \|\Ubmbtrit{t}(a, \ub) - \Ubtrit{t}(a, \ub) \big \|_1 \mumbtrit{0}(da, d\ub)~.
\end{align}
Note that at initialization, there is $\Delta \Ambtrit{0} = \Delta \Umbtrit{0} = 0$.
For the second term on the right-hand side of \eqref{eq:pc_triangle_tr}, we then see that
\begin{equation}
\begin{split}
    \mathcal{W}_1(\mumbtrit{t}, \tilmumbtrit{t}) \leq &~ \int_{\mathbb{R} \times \mathbb{R}^{n}} \| \Thetabmbtrit{t}(a, \ub) - \Thetabtrit{t}(a, \ub) \|_2 \mumbtrit{0}(da, d \ub) \\
    \leq & \Delta \Ambtrit{t} + \Delta \Umbtrit{t}~.
\end{split}
\end{equation}
Therefore, from \eqref{eq:pc_triangle_tr}, we deduce that that
\begin{equation}
   \mathcal{W}_1(\mutrit{t}, \mumbtrit{t}) \leq C_3(n, t) ~ \mathcal{W}_1(\mutrit{0}, \mumbtrit{0}) + \Delta \Ambtrit{t} + \Delta \Umbtrit{t}~.
\end{equation}
Moreover, \eqref{eq:dotAmbtri} and \eqref{eq:dotAtri} imply that
\begin{equation}
\label{eq:ddt_delta_amtrt}
    \begin{split}
        &~ \frac{d}{dt} \left | \Ambtrit{t}(a, \ub) - \Ambtrit{t}(a, \ub) \right | \\ \leq &~ \sum_{k=1}^n \left |\sigtbig {(\Ubmbtrit{t}(a, \ub))} - \sigtbig {(\Ubtrit{t}(a, \ub))} \right | (\sup_{k \in [n]} \left | \ft{t}(\xbtrk{k})\right | + \| \yb \|_{\infty}) \\
        &~ + \sum_{k=1}^n \left |\sigtbig {(\Ubmbtrit{t}(a, \ub))} \right | \sup_{k \in [n]} \left | \fmbt{t}(\xbtrk{k}) - \ft{t}(\xbtrk{k}) \right | \\
        \leq &~ \mathtt{M}_{\sigma_2} (C_2(n, t) + \| \yb \|_{\infty}) \Delta \Umbtrit{t} + n \mathtt{M}_{\sigma_2} (\mathtt{L}_{\sigma_2} C_1(n, t) + \mathtt{M}_{\sigma_2}) ~ \mathcal{W}_1(\mumbtrit{t}, \mutrit{t}) \\
        \leq &~ n \mathtt{M}_{\sigma_2} \mathtt{L}_{\sigma_2} C_1(n, t) \left | \Ambtrit{t}(a, \ub) - \Ambtrit{t}(a, \ub) \right |  + n (\mathtt{M}_{\sigma_2})^2 \Delta \Ambtrit{t} \\
        &~ + \left ( \mathtt{M}_{\sigma_2} (C_2(n, t) + \| \yb \|_{\infty})+ n \mathtt{M}_{\sigma_2} (\mathtt{L}_{\sigma_2} C_1(n, t) + \mathtt{M}_{\sigma_2}) \right ) \Delta \Umbtrit{t} \\
        &~ + n \mathtt{M}_{\sigma_2} (\mathtt{L}_{\sigma_2} C_1(n, t) + \mathtt{M}_{\sigma_2}) C_3(n, t) \mathcal{W}_1(\mumbtrit{0}, \mutrit{0})~.
    \end{split}
\end{equation}
Meanwhile, \eqref{eq:dotUmbtri} and \eqref{eq:dotUtri} imply that, $\forall k \in [n]$,
\begin{equation}
\label{eq:ddt_delta_umtrt}
    \begin{split}
        &~ \left | \frac{d}{dt} \left (\Umbtritk{t}{l}(a, \ub) - \Utritk{t}{l}(a, \ub) \right ) \right | \\
        \leq &~ \left | \Ambtrit{t}(a, \ub) - \Atrit{t}(a, \ub) \right | \sum_{k=1}^n \left | \sigma'_2(\Utritk{t}{k}(a, \ub)) \right | \left | \zetak{k}{\ftritk{t}{k}} \right | | G_{k, l}| \\
        & + \left | \Ambtrit{t}(a, \ub) \right | \sum_{k=1}^n \left | \sigtpbig{\Utritk{t}{k}(a, \ub)} - \sigtpbig{\Umbtritk{t}{k}(a, \ub)} \right | \left | \zetak{k}{\ftritk{t}{k}} \right | | G_{k, l} |  \\
        & + \left | \Ambtrit{t}(a, \ub) \right | \sum_{k=1}^n \left | \sigtpbig{\Umbtritk{t}{k}(a, \ub)} \right | \left | \zetak{k}{\fmbtritk{t}{k}} \right | |G_{k, l}| \\
        & + \left | \Ambtrit{t}(a, \ub) \right | \sum_{k=1}^n \left | \sigtpbig{\Umbtritk{t}{k}(a, \ub)} \right | \left | \zetak{k}{\ftritk{t}{k}} - \zetak{k}{\fmbtritk{t}{k}} \right | |G_{k, l} - \Gmb_{k, l}| \\
        \leq &~ n \mathtt{M}_{\sigma_2} (\mathtt{M}_{\sigma_2})^2 (C_2(n, t) + \| \yb \|_{\infty}) \Delta \Ambtrit{t} \\
        & +  \mathtt{L}_{\sigma'_2} (\mathtt{M}_{\sigma_2})^2 C_1(n, t) (C_2(n, t) + \| \yb \|_{\infty}) \Delta \Umbtrit{t} \\
        & + n \mathtt{M}_{\sigma_2} (\mathtt{M}_{\sigma_2})^2 C_1(n, t) (\mathtt{M}_{\sigma_2} + C_1(n, t) \mathtt{M}_{\sigma_2}) (C_3(n, t) ~\mathcal{W}_1(\mutrit{0}, \mumbtrit{0}) + \Delta \Ambtrit{t} + \Delta \Umbtrit{t}) \\
        & + \sqrt{n} \mathtt{M}_{\sigma_2} C_1(n, t) (C_2(n, t) + \| \yb \|_{\infty}) \| G - \Gmb \|_2~.
    \end{split}
\end{equation}
where we use the inequality that $\forall k \in [n]$,
\begin{equation}
\begin{split}
    \left | \fmbtritk{t}{k} - \ftritk{t}{k} \right | =& \left | \int_{\mathbb{R} \times \mathbb{R}^n} a \sigma_2(u_k) (\mumbtrit{t} - \mutrit{t})(da, d\ub) \right | \\
    \leq &~ (\mathtt{M}_{\sigma_2} + C_1(n, t) \mathtt{M}_{\sigma_2}) ~ \mathcal{W}_1(\mumbtrit{t}, \mutrit{t}) \\
    \leq &~ (\mathtt{M}_{\sigma_2} + C_1(n, t) \mathtt{M}_{\sigma_2}) (C_3(n, t) \mathcal{W}_1(\mumbtrit{0}, \mutrit{0}) + \Delta \Ambtrit{t} + \Delta \Umbtrit{t})~.
\end{split}
\end{equation}
Together, \eqref{eq:ddt_delta_amtrt} and \eqref{eq:ddt_delta_umtrt} imply that
\begin{equation}
    \begin{split}
        &~ \Delta \Ambtrit{t} + \Delta \Umbtrit{t} \\
        \leq &~ \int_0^t \left ( C_4(n, s) (\Delta \Ambtrit{s} + \Delta \Umbtrit{s} ) + C_5(n, s)~ \mathcal{W}_1(\mumbtrit{0}, \mutrit{0}) + C_6(n, t) \| \Gmb - G \|_2 \right ) ds~.
    \end{split}
\end{equation}
Thus, by Gr\"onwall's inequality, we have
\begin{equation}
\label{eq:DeltaA+DeltaU}
\begin{split}
    \Delta \Ambtrit{t} + \Delta \Umbtrit{t} \leq & C_5(n, t) ~ \mathcal{W}_1(\mumbtrit{0}, \mutrit{0}) + C_6(n, t) \| \Gmb - G \|_2 \\
    & + \int_0^t \left ( C_5(n, s) ~\mathcal{W}_1(\mumbtrit{0}, \mutrit{0}) + C_6(n, s) \| \Gmb - G \|_2 \right ) C_4(n, s) e^{\int_s^t C_4(n, r) dr} ds \\
    \leq & C_5(n, t) \left ( 1 + \int_0^t C_4(n, s) e^{\int_s^t C_4(n, r) dr} ds \right ) \mathcal{W}_1(\mumbtrit{0}, \mutrit{0}) \\
    & + C_6(n, t)\left ( 1 + \int_0^t C_4(n, s) e^{\int_s^t C_4(n, r) dr} ds \right ) \| \Gmb - G \|_2 \\
    =&: C_7(n, t) ~\mathcal{W}_1(\mumbtrit{0}, \mutrit{0}) + C_8(n, t) \| \Gmb - G \|_2~,
\end{split}
\end{equation}
which also implies that
\begin{equation}
    \mathcal{W}_1(\mumbtrit{t}, \mutrit{t}) \leq C_7(n, t) ~\mathcal{W}_1(\mumbtrit{0}, \mutrit{0}) + C_8(n, t) \| \Gmb - G \|_2~,
\end{equation}
and 
\begin{equation}
    \sup_{k \in [n]} \left | \fmbt{t}(\xbtrk{k}) - \ft{t}(\xbtrk{k}) \right | \leq (\mathtt{M}_{\sigma_2} + C_1(n, t) \mathtt{M}_{\sigma_2}) \left ( C_7(n, t)~ \mathcal{W}_1(\mumbtrit{0}, \mutrit{0}) + C_8(n, t) \| \Gmb - G \|_2 \right )~.
\end{equation}

\section{Proof of Lemma~\ref{lem:lln_t_gen}}
\label{app:pf_lln_t_gen}
For each $t \geq 0$, we write $\mudtrit{t} = (\hat{\ev}_{\xbtrk{1}, ..., \xbtrk{n}, \xbg{1}, ..., \xbg{\ng}})_{\#} \mumbt{t}$ and $\mumbdtrit{t} = (\hat{\ev}_{\xbtrk{1}, ..., \xbtrk{n}, \xbg{1}, ..., \xbg{\ng}})_{\#}\mut{t}$.

It is straightforward to show that we can write $\mudtrit{t} = (\Thetabdtrit{t})_{\#}(\mudtrit{0})$ and $\mumbdtrit{t} = (\Thetabmbdtrit{t})_{\#}(\mumbdtrit{0})$, where $\Thetabdtrit{t} = [\Adtrit{t}, \Ubdtrit{t}, \Vbdtrit{t}]: \Rbb \times \Rbb^n \times \Rbb^{\ng} \to \Rbb \times \Rbb^n \times \Rbb^{\ng}$ and $\Thetabmbdtrit{t} = [\Ambdtrit{t}, \Ubmbdtrit{t}, \Vbmbdtrit{t}]: \Rbb \times \Rbb^n \times \Rbb^{\ng} \to \Rbb \times \Rbb^n \times \Rbb^{\ng}$ are defined by, $\forall a \in \Rbb, \ub \in \Rbb^n, \vb \in \Rbb^{n'}$,
\begin{equation}
\label{eq:AUdtrit}
    \begin{split}
        \Adtrit{t}(a, \ub, \vb) =&~ \Atrit{t}(a, \ub)~, \\
    \Ambdtrit{t}(a, \ub, \vb) =&~ \Ambtrit{t}(a, \ub)~, \\
    \Ubdtrit{t}(a, \ub, \vb) =&~ \Ubtrit{t}(a, \ub)~, \\
    \Ubmbdtrit{t}(a, \ub, \vb) =&~ \Ubmbtrit{t}(a, \ub)~,
    \end{split}
\end{equation}
and for all $k' \in [\ng]$,
\begin{equation}
\label{eq:Vdtrit}
    \begin{split}
        \frac{d}{dt} \Vdtritk{t}{k'}(a, \ub, \vb) =&~ - \frac{1}{n} \Atrit{t}(a, \ub) \summ{k}{n} \sigtpbig{\Utritk{t}{k}(a, \ub)} \zetak{k}{\ft{t}(\xbtrk{k})} \Gfun(\xbtrk{k}, \xbg{k'}) ~, \\
    \frac{d}{dt} \Vmbdtritk{t}{k'}(a, \ub, \vb) =&~ - \frac{1}{n} \Ambtrit{t}(a, \ub) \summ{k}{n} \sigtpbig{\Umbtritk{t}{k}(a, \ub)} \zetak{k}{\fmbt{t}(\xbtrk{k})} \Gfun(\xbtrk{k}, \xbg{k'})~,
    \end{split}
\end{equation}
with $\Vbdtrit{0}(a, \ub, \vb) = \Vbmbdtrit{0}(a, \ub, \vb) = \vb$.

Define $\tilmumbdtrit{t} = (\Thetabdtrit{t})_{\#} \mumbdtrit{0}$.
By the triangle inequality,
\begin{equation}
    \label{eq:triangle_nnp}
    \mathcal{W}_1(\mumbdtrit{t}, \mudtrit{t}) \leq \mathcal{W}_1(\mudtrit{t}, \tilmumbdtrit{t}) + \mathcal{W}_1(\mumbdtrit{t}, \tilmumbdtrit{t})~.
\end{equation}
For the first term on the right-hand side,
\begin{equation}
    \mathcal{W}_1(\mudtrit{t}, \tilmumbdtrit{t}) \leq \text{Lip}(\Thetabdtrit{t}) ~ \mathcal{W}_1(\mudtrit{0}, \tilmumbdtrit{0})
    \leq C_9(n, t) ~ \mathcal{W}_1(\mudtrit{0}, \mumbdtrit{0})~.
\end{equation}
For the second term, we observe that
\begin{equation}
    \begin{split}
        \mathcal{W}_1(\mumbdtrit{t}, \tilmumbdtrit{t}) \leq & \int_{\mathbb{R} \times \mathbb{R}^n \times \mathbb{R}^{n'}} \| \Thetabmbdtrit{t}(a, \ub, \vb) - \Thetabdtrit{t}(a, \ub, \vb) \|_2 \mumbdtrit{0}(da, d\ub, d\vb) \\
        \leq & \Delta \Ambdtrit{t} + \Delta \Umbdtrit{t} + \Delta \Vmbdtrit{t}~,
    \end{split}
\end{equation}
where we define 
\begin{align}
    \Delta \Ambdtrit{t} =& \int_{\Rbb \times \Rbb^n} | \Ambdtrit{t}(a, \ub, \vb) - \Adtrit{t}(a, \ub, \vb) | \mumbdtrit{t}(da, d\ub)~, \\
    \Delta \Ubmbdtrit{t} =& \int_{\Rbb \times \Rbb^n} \| \Umbdtrit{t}(a, \ub, \vb) - \Ubdtrit{t}(a, \ub, \vb) \|_1 \mumbdtrit{t}(da, d\ub) \\
    =& \int_{\Rbb \times \Rbb^n} \sum_{k=1}^n | \Umbdtritk{t}{k}(a, \ub, \vb) - \Udtritk{t}{k}(a, \ub, \vb) | \mumbdtrit{t}(da, d\ub)~, \\
    \Delta \Vmbdtrit{t} =& \int_{\Rbb \times \Rbb^n} \| \Vmbdtrit{t}(a, \ub, \vb) - \Vbdtrit{t}(a, \ub, \vb) \|_1  \mumbdtrit{t}(da, d\ub) \\
    =& \int_{\Rbb \times \Rbb^n} \sum_{k=1}^n | \Vmbdtritk{t}{k}(a, \ub, \vb) - \Vdtritk{t}{k}(a, \ub, \vb) | \mumbdtrit{t}(da, d\ub)~.
\end{align}
With the definitions in \eqref{eq:AUdtrit}, we see that
\begin{align}
    \Delta \Ambdtrit{t} = \Delta \Ambtrit{t}~, \qquad
    \Delta \Umbdtrit{t} = \Delta \Umbtrit{t}~,
\end{align}
and hence, \eqref{eq:DeltaA+DeltaU} implies that
\begin{equation}
    \Delta \Ambdtrit{t} + \Delta \Umbdtrit{t} \leq C_7(n, t) ~ \mathcal{W}_1(\mutrit{0}, \mumbtrit{0}) + C_8(n, t) \| \Gmb - G \|_2~.
\end{equation}
Moreover, \eqref{eq:Vdtrit} implies that
\begin{equation}
\label{eq:ddt_delta_vmtrt}
    \begin{split}
        & \int_{\Rbb \times \Rbb^n} \left | \frac{d}{dt} \left (\Vmbdtritk{t}{k'}(a, \ub, \vb) - \Vdtritk{t}{k'}(a, \ub, \vb) \right )\right | \mumbdtrit{t}(da, d\ub) \\
        \leq &~\int_{\Rbb \times \Rbb^n} \bigg ( \left | \Ambtrit{t}(a, \ub) - \Atrit{t}(a, \ub) \right | \sum_{k=1}^n \left | \sigma'_2(\Utritk{t}{k}(a, \ub)) \right | \left | \zetak{k}{\ft{t}(\xbtrk{k})} \right | |\Gfun(\xbtrk{k}, \xbg{k'})| \\
        & + \left | \Ambtrit{t}(a, \ub) \right | \sum_{k=1}^n \left | \sigma'_2(\Utritk{t}{k}(a, \ub)) - \sigma'_2(\Umbtritk{t}{k}(a, \ub)) \right | \left | \zetak{k}{\ft{t}(\xbtrk{k})} \right | |\Gfun(\xbtrk{k}, \xbg{k'})|  \\
        & + \left | \Ambtrit{t}(a, \ub) \right | \sum_{k=1}^n \left | \sigma'_2(\Umbtritk{t}{k}(a, \ub)) \right | \left | \zetak{k}{\ft{t}(\xbtrk{k})} - \zetak{k}{\fmbt{t}(\xbtrk{k})} \right | |\Gfun(\xbtrk{k}, \xbg{k'})| \\
        & + \left | \Ambtrit{t}(a, \ub) \right | \sum_{k=1}^n \left | \sigma'_2(\Umbtritk{t}{k}(a, \ub)) \right | \left | \zetak{k}{\fmbt{t}(\xbtrk{k})} \right |  |\Gfun(\xbtrk{k}, \xbg{k'}) - \Gfun(\xbtrk{k}, \xbg{k'})| \bigg ) \mumbdtrit{t}(da, d\ub) \\
        \leq & n \mathtt{M}_{\sigma_2} (\mathtt{M}_{\sigma_2})^2 (C_2(n, t) + \| \yb \|_{\infty}) \Delta \Ambtrit{t} \\
        & +  \mathtt{L}_{\sigma'_2} (\mathtt{M}_{\sigma_2})^2 C_1(n, t) (C_2(n, t) + \| \yb \|_{\infty}) \Delta \Umbtrit{t} \\
        & + n \mathtt{M}_{\sigma_2} (\mathtt{M}_{\sigma_2})^2 C_1(n, t) (\mathtt{M}_{\sigma_2} + C_1(n, t) \mathtt{M}_{\sigma_2}) (C_3(n, t) ~\mathcal{W}_1(\mutrit{0}, \mumbtrit{0}) + \Delta \Ambtrit{t} + \Delta \Umbtrit{t}) \\
        & + \sqrt{n} \mathtt{M}_{\sigma_2} C_1(n, t) (C_2(n, t) + \| \yb \|_{\infty}) \sum_{k=1}^n |\Gfunmb(\xbtrk{k}, \xbg{k'}) - \Gfun(\xbtrk{k}, \xbg{k'})|~.
    \end{split}
\end{equation}
Thus, together with \eqref{eq:DeltaA+DeltaU}, we see there exists a function $C_9'(n, t)$ that is non-negative and non-decreasing in $t$ such that
\begin{equation}
\begin{split}
    \Delta \Vmbdtrit{t} \leq & \int_0^t C_9(n, s) (\Delta \Ambtrit{s} + \Delta \Umbtrit{s} + ~\mathcal{W}_1(\mutrit{0}, \mumbtrit{0}) + \| \Gmbtrte - \Gtrte \|_2) ds \\
    \leq & \int_0^t C_{10}(n, s) (1 + C_7(n, s) + C_8(n, s)) (\mathcal{W}_1(\mutrit{0}, \mumbtrit{0}) + \| \Gmbtrte - \Gtrte \|_2) ds \\
    \leq & e^{\int_0^t C_{10}(n, s) (1 + C_7(n, s) + C_8(n, s)) ds}(\mathcal{W}_1(\mutrit{0}, \mumbtrit{0}) + \| \Gmbtrte - \Gtrte \|_2) \\
    =&: C_{11}(n, t) (\mathcal{W}_1(\mutrit{0}, \mumbtrit{0}) + \| \Gmbtrte - \Gtrte \|_2)~.
\end{split}
\end{equation}
Therefore,
\begin{equation}
\begin{split}
    \mathcal{W}_1(\mumbdtrit{t}, \mudtrit{t}) \leq & C_9(n, t) \mathcal{W}_9(\mudtrit{0}, \mumbdtrit{0}) + C_{11}(n, t) (\mathcal{W}_1(\mutrit{0}, \mumbtrit{0})) + \| \Gmbtrte - \Gtrte \|_2) \\
    \leq & 2 C_{11}(n, t) (\mathcal{W}_1(\mudtrit{0}, \mumbdtrit{0}) + \| \Gmbtrte - \Gtrte \|_2)~,
\end{split}
\end{equation}
since $\mathcal{W}_1(\mutrit{0}, \mumbtrit{0}) \leq \mathcal{W}_1(\mudtrit{0}, \mumbdtrit{0})$.

\section{Extension to include the bias term}
\label{app:bias}
We can define a more general version of the \ptl NN model with the bias term included in the second hidden layer, as
\begin{equation}
\label{eq:p3l_b}
    \begin{split}
        f^{\mb}_{\alpha}(\xb; \ab, \bb, W) 
        =&~ \frac{1}{m_2} \sum_{i=1}^{m_2} a_i \sigma_{2} \big ( h_i(\xb) \big )~,  \\
    \forall i \in [m_2] \quad : \quad h_i(\xb) =&~ b_i + \frac{1}{m_1^{\alpha}} \sum_{j=1}^{m_1} W_{ij} \sigma_{1} \big (\zb_j^{\intercal} \cdot \xb \big )~, \\
    \end{split}
\end{equation}
where $\bb = [b_1, ..., b_{m_2}] \in \Rbb^{m_2}$. During training, its dynamics is given by
\begin{equation}
    \frac{d}{dt}\bit{t} = -\frac{\betaB \ait{t}}{n} \sum_{k=1}^n \zetak{k}{\fmbt{t}(\xbtrk{k})} {\sigma_2}' \big (\hit{t}(\xbtrk{k}) \big ) ~,
\end{equation}
where $\betaB \geq 0$ denotes its learning rate relative to $W_t$. As $m_1, m_2 \to \infty$, the model can be described by a similar functional-space MF limit, namely, $\mut{t} = (\Thetabt{t})_{\#} \mut{0}$ with $\mut{0} = \rhoa \times \chi$. Compared to the bias-less case, \eqref{eq:dotHt_Con} is replaced by
\begin{equation}
    \frac{d}{dt} \Htah{t} = \frac{1}{n}\Atah{t} \summ{k}{n} \zetak{k}{\ft{t}(\xbtrk{k})} \sigtpbig{\Htah{t}(\xbtrk{k})} \big ( \betaB + \Gfun(\xbtrk{k}, \cdot) \big )~,
\end{equation}
and moreover, $\chi = \int_{\Rbb} \delta_{\bb} \rhob(db)$ if $\alpha > \frac{1}{2}$ and $\chi = \int_{\Rbb} \mathcal{GP}(b, \Gfun) \rhob(db)$ if $\alpha  = \frac{1}{2}$, where for any $b \in \Rbb$, $\delta_{\bb}$ denotes the singular measure at the constant function on $\domX$ with value $b$. The proof for the existence of the MF dynamics and the LLN is similar to the biasless case and can be found in Appendix D.1 of \citet{chen2024nhl}, a follow-up work by the authors.

\section{Proof of Theorem~\ref{prop:gc}}
\label{app:pf_gc}
With the value of $\hat{a} > 0$ to be specified later, we define $\xi_{\max} =  \min\{\frac{1}{2}(I_r - I_l), \frac{1}{2} \hat{a} \}$ if $\alpha = 1 / 2$ and $\min\{\frac{1}{2} I_r, -\frac{1}{2} I_l, \frac{1}{2} \hat{a} \}$ if $\alpha > 1 / 2$ (note the additional condition in Assumption~\ref{ass:itvl} in the latter case). We choose any $\xi \in (0, \xi_{\max})$ 
and define an open interval $I_{\xi} = (I_l + \xi, I_r - \xi)$. For each $k \in [n]$, we define sets $\Xi, \Xi^{\dag}_k \in \mathbb{R} \times \Con$ as
\begin{align}
  \Xi_k =& \left \{ a \in \mathbb{R}, h \in \Con: |a| \geq \frac{1}{2} \hat{a}, h(\xbtrk{k}) \in I \right \} ~, \\
  \Xi^{\dag}_k =& \left \{ a \in \mathbb{R}, h \in \Con: |a| \geq \frac{1}{2} \hat{a} + \xi, h(\xbtrk{k}) \in I_{\xi} \right \}~.
\end{align}
We see that
\begin{equation}
\label{eq:PL_in_lem}
    \begin{split}
        -\frac{d}{dt}\mathcal{L}_t \geq & 
         \int_{\mathbb{R} \times \Con} \frac{\left (\Atah{t} \right )^2}{n^2} \sum_{k, l = 1} \sigtpbig{\Htah{t}(\xbtrk{k})} \sigtpbig{\Htah{t}(\xbtrk{l})} \\
         & \hspace{130pt} \cdot \barzetak \barzetal G_{k, l} \mut{0}(da, dh) \\
         \geq & \int_{\mathbb{R} \times \Con} \frac{\left (\Atah{t} \right )^2 }{n^2}  \lambmin \sum_{k = 1}^n \left (\sigtpbig{\Htah{t}(\xbtrk{k})} \right )^2 \barzetak^2 \mut{0}(da, dh) \\
         \geq & \frac{\lambmin(G)}{n^2} \sum_{k = 1}^n \int_{\Xi_k} \left (\Atah{t} \right )^2 \left (\sigtpbig{\Htah{t}(\xbtrk{k})} \right )^2 \barzetak^2 \mut{0}(da, dh) \\
         \geq & \frac{(\mathtt{K}_{\sigma_2})^2 \lambmin(G)}{2 n} \hat{a}^2 \left ( \min_{k \in [n]} \mut{t}(\Xi_k) \right ) \mathcal{L}_t~.
    \end{split}
\end{equation}
In the following lemma, we provide a lower bound on the term $\min_{k \in [n]} \mut{t}(\Xi_k)$ for $t \geq 0$ via a fine-grained analysis of the dynamics:
\begin{lemma}
\label{lem:etat}
$\forall t \geq 0$, $\forall \hat{a} > 0$,
\begin{equation}
    \min_{k \in [n]} \mut{t}(\Xi_k) \geq \left ( (\min_{k \in [n]} \mut{0}(\Xi_k))^{\frac{2}{3}} - \frac{K_1}{\hat{a}} \right )^{\frac{3}{2}}~,
\end{equation}
where $K_1 = \frac{3 \left ( (\betaA)^{\frac{1}{2}} + (\| \Gfun \|_{\infty})^{\frac{1}{2}} \right ) \| \yb \|_2}{\xi (\lambmin(G))^{\frac{1}{2}} \mathtt{K}_{\sigma_2}}$.
\end{lemma}
\noindent This lemma is proved in Appendix~\ref{app:etat_pf}, and it extends the analogous results proved in \citet{chen2022on} for the non-asymptotic setting restricted to having $\betaA = 0$ and $\alpha = 1 / 2$.

By assumption, $\rhoa((-\infty, \hat{a}] \cup [\hat{a}, \infty)) > 0$. When $\alpha > \frac{1}{2}$, if Assumptions~\ref{ass:init} and \ref{ass:itvl} are satisfied, we know that for any $k \in [n]$, $\mut{0}(\Xi^{\dag}_k) = \rhoa((-\infty, \frac{1}{2}\hat{a}-\xi] \cup [\frac{1}{2}\hat{a} + \xi, \infty)) \geq 2 \rhoa([\hat{a}, \infty)) > 0$. When $\alpha = 1 / 2$, for any $k \in [n]$, since $(\hatev_{\xbtrk{k}})_{\#}\mut{0} = \mathcal{N}(0, G_{kk})$, we know that
\begin{equation}
\begin{split}
    \mut{0}(\Xi^{\dag}_k) =&~ \rhoa((-\infty, \frac{1}{2} \hat{a}-\xi] \cup [\frac{1}{2} \hat{a} + \xi, \infty))  \int_{I_l + \xi}^{I_r - \xi} \frac{1}{\sqrt{2 \pi G_{kk}}} e^{-\frac{u^2}{2 G_{kk}}} du \\
     \geq &~ \sqrt{2} \rhoa( [\hat{a}, \infty)) \frac{I_r - I_l - 2 \xi}{\sqrt{\pi} \| \Gfun \|_{\infty}} e^{-\frac{\max\{ (I_l)^2, (I_r)^2\}}{2 G_{\min}}} > 0~.
\end{split}
\end{equation}
Thus, defining 
\begin{equation}
    K_2 = \begin{cases}
        2 \rhoa([\hat{a}, \infty))~,~ &\text{if } \alpha > \frac{1}{2} \\
        \sqrt{2} \rhoa( [\hat{a}, \infty)) \frac{I_r - I_l - 2 \xi}{\sqrt{\pi} \| \Gfun \|_{\infty}} e^{-\frac{\max\{ (I_l)^2, (I_r)^2\}}{2 G_{\min}}}~,~ &\text{if } \alpha = 1 / 2~, 
    \end{cases}
\end{equation}
it holds that $\min_{k \in [n]} \mut{0}(\Xi^{\dag}_k) > K_2 > 0$.
Hence, if we choose $\hat{a} \geq {4 K_1}/{(3 (K_2)^{\frac{2}{3}})}$, then $\forall t \geq 0$,
\begin{equation}
    \begin{split}
        \min_{k \in [n]} \mut{t}(\Xi^{\dag}_k) \geq (\frac{1}{4}(K_2)^{\frac{2}{3}})^{\frac{3}{2}} = \frac{1}{8} K_2 > 0~.
    \end{split}
\end{equation}
This allows us to conclude that
\begin{equation}
    - \frac{d}{dt} \mathcal{L}_t \geq \frac{\lambmin(G) (\mathtt{K}_{\sigma_2})^2 \hat{a}^2}{2 n} K_2 \mathcal{L}_t~,
\end{equation}
and hence $\mathcal{L}_t \leq \mathcal{L}_0 e^{-r \lambmin \hat{a}^2 t}$,
where $r = (\mathtt{K}_{\sigma_2})^2 K_2 / (2n)$.

\subsection{Proof of Lemma~\ref{lem:etat}}
\label{app:etat_pf}
We first prove a relevant lemma about the dynamics of $\At{t}$ and $\Ht{t}$.
\begin{lemma}
\label{lem:ddtAH_int_upper}
$\forall t \geq 0$, 
\begin{equation}
\label{eq:ddtA_bdd}
     \int_{\mathbb{R} \times \Con} \left |\frac{d}{dt} \Atah{t} \right |^2 \mut{0}(da, dh) \leq - \beta_a \frac{d}{dt} \mathcal{L}_t~,
\end{equation}
and $\forall \xb \in \mathcal{X}$,
\begin{equation}
    \int_{\mathbb{R} \times \Con} \left |\frac{d}{dt} \Htah{t}(\xb) \right |^2 \mut{0}(da, dh) \leq - \| \Gfun \|_{\infty} \frac{d}{dt} \mathcal{L}_t ~.
\end{equation}
\end{lemma}
\begin{proof}
For $t \geq 0, a \in \mathbb{R}, h \in \Con$, define a function $g_t(\cdot; a, h)$ on $\mathbb{R}^d$ by, $\forall \zb \in \mathbb{R}^d$,
\begin{equation}
\label{eq:g_defn}
    g_t(\zb; a, h) = -\frac{1}{n} \Atah{t} \sum_{k=1}^n {\sigma_2}' \left (\Htah{t}(\xbtrk{k}) \right ) \barzetak \sigma_1(\zb^{\intercal} \xbtrk{k}) ~.
\end{equation}
On one hand, there is
\begin{equation}
\label{eq:ddH_duality}
    \frac{d}{dt} \Htah{t}(\xb) = \int_{\mathbb{R}^d} g_t(\zb; a, h) \sigobig{\zb^{\intercal} \xb} \rhoz(d \zb)~,
\end{equation}
and so $\forall \xb \in \mathcal{X}$, by the Cauchy-Schwarz inequality,
\begin{equation}
\label{eq:ddtH_g}
\begin{split}
    \left | \frac{d}{dt} \Htah{t}(\xb) \right | \leq &~ \left ( \int_{\mathbb{R}^d} \left ( g_t(\zb; a, h) \right )^2 \rhoz(d \zb) \right )^{\frac{1}{2}} \left ( \int_{\mathbb{R}^d} \left ( \sigma_1(\zb^{\intercal} \xb) \right )^2 \rhoz(d \zb) \right )^{\frac{1}{2}} \\
    \leq &~ \left ( \| \Gfun \|_{\infty} \int_{\mathbb{R}^d} \left ( g_t(\zb; a, h) \right )^2 \rhoz(d \zb) \right )^{\frac{1}{2}} ~.
\end{split}
\end{equation}
On the other hand, we see that
\begin{equation}
\label{eq:gt_L2}
\begin{split}
    & \int_{\mathbb{R}^d} \left | g_t(\zb; a, h) \right |^2 \rhoz(d \zb)\\
    =&~ \int_{\mathbb{R}^d} \frac{1}{n^2} |\Atah{t}|^2 \sum_{k, l = 1} \bigg ( {\sigma_2}' \left ( \Htah{t}(\xbtrk{k}) \right ) {\sigma_2}' \left ( \Htah{t}(\xbtrk{l}) \right ) \\
    & \hspace{100pt} \cdot \barzetak \barzetal \sigma_1(\zb^{\intercal} \xbtrk{k}) \sigma_1(\zb^{\intercal} \xbtrk{l}) \rhoz(d \zb) \bigg ) \\
    =&~ \frac{1}{n^2} |\Atah{t}|^2 \sum_{k, l = 1} {\sigma_2}' \left ( \Htah{t}(\xbtrk{k}) \right ) {\sigma_2}' \left ( \Htah{t}(\xbtrk{l}) \right ) \barzetak \barzetal G_{k, l}~.
\end{split}
\end{equation}
and hence
\begin{equation}
\begin{split}
    - \frac{d}{dt} \mathcal{L}_t =&~ \int_{\mathbb{R} \times \Con} \frac{\beta}{n^2} \sum_{k, l=1} \barzetak \barzetal \sigma_2 \left ( \Htah{t}(\xbtrk{k}) \right ) \sigma_2 \left ( \Htah{t}(\xbtrk{l}) \right ) \mut{0}(da, dh)\\
    & + \int_{\mathbb{R} \times \Con} \frac{1}{n^2} |\Atah{t}|^2 \sum_{k, l = 1} \bigg ( {\sigma_2}' \left ( \Htah{t}(\xbtrk{k}) \right ) {\sigma_2}' \left ( \Htah{t}(\xbtrk{l}) \right ) \\
    & \hspace{100pt} \cdot \barzetak \barzetal G_{k, l} \mut{0}(da, dh) \bigg ) \\
    =&~ \beta^{-1} \int_{\mathbb{R} \times \Con} \left | \frac{d}{dt} \Atah{t} \right |^2 \mut{0}(da, dh) + \int_{\mathbb{R} \times \Con} \int_{\mathbb{R}^d} \left | g_t(\zb; a, h) \right |^2 \rhoz(d\zb) \mut{0}(da, dh)~.
\end{split}
\end{equation}
Thus,
\begin{align}
    \int_{\mathbb{R} \times \Con} \left | \frac{d}{dt} \Atah{t} \right |^2 \mut{0}(da, dh) \leq & -\beta \frac{d}{dt} \mathcal{L}_t~,~  \\
    \int_{\mathbb{R} \times \Con} \int_{\mathbb{R}^d} \left | g_t(\zb; a, h) \right |^2 \rhoz(d\zb) \mut{0}(da, dh) \leq & - \frac{d}{dt} \mathcal{L}_t~,~  \label{eq:g}
\end{align}
and by \eqref{eq:ddtH_g}, we know that $\forall \xb \in \mathcal{X}$,
\begin{equation}
    \begin{split}
        \int_{\mathbb{R} \times \Con} \left | \frac{d}{dt} \Htah{t}(\xb) \right |^2 \mut{0}(da, dh) \leq 
        &~ \| \Gfun \|_{\infty} \int_{\mathbb{R} \times \Con} \int_{\mathbb{R}^d} \left | g_t(\zb; a, h) \right |^2 \rhoz(d \zb) \mut{0}(da, dh) \\
        \leq &~ - \| \Gfun \|_{\infty} \frac{d}{dt} \mathcal{L}_t~.
    \end{split}
\end{equation}
\end{proof}
\noindent Next, we will prove Lemma~\ref{lem:etat}.
Since $\forall k \in [n]$, $\forall t \geq 0$, there is
\begin{equation}
\begin{split}
    \Xi^{\dag}_k \subseteq (\Thetabt{t})^{-1} (\Xi_k) ~\cup~ & \{ a \in \mathbb{R}, h \in \Con: \left | \Atah{t} - a \right | > \xi \} \\ ~\cup~ & \{ a \in \mathbb{R}, h \in \Con: \left | \Htah{t}(\xbtrk{k}) - h(\xbtrk{k}) \right | > \xi \}~,
\end{split}
\end{equation}
we know that
\begin{equation}
\begin{split}
    \mut{0}(\Xi^{\dag}_k) \leq 
    \mut{t} \left ( \Xi_k \right ) 
    + & \mut{0} \left ( \{ a \in \mathbb{R}, h \in \Con: \left | \Atah{t} - a \right | > \xi \} \right ) \\ + & \mut{0} \left ( \{ a \in \mathbb{R}, h \in \Con: \left | \Htah{t}(\xbtrk{k}) - h(\xbtrk{k}) \right | > \xi \} \right )~.
\end{split}
\end{equation}
Meanwhile, we know that
\begin{equation}
\begin{split}
    \int_{\mathbb{R} \times \Con} \left | \Atah{t} - a \right | \mut{0}(da, dh) \leq &~ \int_{\mathbb{R} \times \Con} \int_0^t \left | \frac{d}{ds} \Atah{s} \right | ds \mut{0}(da, dh) \\
    \leq &~ \int_0^t \left ( \int_{\mathbb{R} \times \Con} \left | \frac{d}{ds} \Atah{s} \right |^2 \mut{0}(da, dh) \right )^{\frac{1}{2}} ds \\
    \leq &~ (\betaA)^{\frac{1}{2}} \int_0^t \left ( - \frac{d}{ds} \mathcal{L}_s \right )^{\frac{1}{2}} ds~,
\end{split}
\end{equation}
and $\forall k \in [n]$,
\begin{equation}
\begin{split}
    \int_{\mathbb{R} \times \Con} \left | \Htah{t}(\xbtrk{k}) - h(\xbtrk{k}) \right | \mut{0}(da, dh) \leq &~ \int_{\mathbb{R} \times \Con} \int_0^t \left | \frac{d}{ds} \Htah{s}(\xbtrk{k}) \right | ds \mut{0}(da, dh) \\
    \leq &~ \int_0^t \left ( \int_{\mathbb{R} \times \Con} \left | \frac{d}{ds} \Htah{s}(\xbtrk{k}) \right |^2 \mut{0}(da, dh) \right )^{\frac{1}{2}} ds \\
    \leq &~ (\| \Gfun \|_{\infty})^{\frac{1}{2}} \int_0^t \left ( - \frac{d}{ds} \mathcal{L}_s \right )^{\frac{1}{2}} ds~.
\end{split}
\end{equation}
Thus, by Markov's inequality,
\begin{equation}
    \begin{split}
        \mut{0} \left ( \{ a \in \mathbb{R}, h \in \Con: \left | \Atah{t} - a \right | > \xi \} \right ) \leq & ~ \xi^{-1} \int_{\mathbb{R} \times \Con} \left | \Atah{t} - a \right | \mut{0}(da, dh) \\
        \leq & ~ \frac{(\betaA)^{\frac{1}{2}}}{\xi} \int_0^t \left ( - \frac{d}{ds} \mathcal{L}_s \right )^{\frac{1}{2}} ds~,
    \end{split}
\end{equation}
and $\forall k \in [n]$,
\begin{equation}
    \begin{split}
        \mut{0} \left ( \{ a \in \mathbb{R}, h \in \Con: \left | \Htah{t}(\xbtrk{k}) - h(\xbtrk{k}) \right | > \xi \} \right ) \leq &~ \xi^{-1} \int_{\mathbb{R} \times \Con} \left | \Htah{t}(\xbtrk{k}) - h(\xbtrk{k}) \right | \mut{0}(da, dh) \\
        \leq &~ \frac{(\| \Gfun \|_{\infty})^{\frac{1}{2}}}{\xi} \int_0^t \left ( - \frac{d}{ds} \mathcal{L}_s \right )^{\frac{1}{2}} ds~.
    \end{split}
\end{equation}
Hence, $\forall k \in [n]$,
\begin{equation}
    \begin{split}
        \mut{t}(\Xi_k) 
        \geq &~ \mut{0}(\Xi^{\dag}_k) - \frac{(\betaA)^{\frac{1}{2}} + (\| \Gfun \|_{\infty})^{\frac{1}{2}}}{\xi} \int_0^t \left ( - \frac{d}{ds} \mathcal{L}_s \right )^{\frac{1}{2}} ds~. \\
    \end{split}
\end{equation}
Thus, defining $\eta_t = \min_{k \in [n]} \mut{0}(\Xi^{\dag}_k) - \frac{(\betaA)^{\frac{1}{2}} + (\| \Gfun \|_{\infty})^{\frac{1}{2}}}{\xi} \int_0^t \left ( - \frac{d}{ds} \mathcal{L}_s \right )^{\frac{1}{2}} ds$, we have $\min_{k \in [n]} \mut{t}(\Xi_k) \geq \eta_t$.
Therefore, via \eqref{eq:PL_in_lem}, we deduce that
\begin{equation}
    - \frac{d}{dt} \mathcal{L}_t \geq \frac{\lambmin (\mathtt{K}_{\sigma_2})^2 \hat{a}^2}{2 n} \eta_t \mathcal{L}_t~.
\end{equation}
On the other hand, the definition of $\eta_t$ implies that
\begin{equation}
    - \frac{d}{dt} \eta_t = \frac{(\betaA)^{\frac{1}{2}} + (\| \Gfun \|_{\infty})^{\frac{1}{2}}}{\xi} \left ( - \frac{d}{dt} \mathcal{L}_t \right )^{\frac{1}{2}}~.
\end{equation}
Combined together, they imply that
\begin{equation}
    \begin{split}
         - \frac{d}{dt} \eta_t =&~ \frac{(\betaA)^{\frac{1}{2}} + (\| \Gfun \|_{\infty})^{\frac{1}{2}}}{\xi} \left ( - \frac{d}{dt} \mathcal{L}_t \right ) \left ( - \frac{d}{dt} \mathcal{L}_t \right )^{-\frac{1}{2}} \\
         \leq &~ \frac{(\betaA)^{\frac{1}{2}} + (\| \Gfun \|_{\infty})^{\frac{1}{2}}}{\xi} \left ( - \frac{d}{dt} \mathcal{L}_t \right ) \left ( \frac{\lambmin (\mathtt{K}_{\sigma_2})^2 \hat{a}^2}{2 n} \eta_t \mathcal{L}_t \right )^{-\frac{1}{2}} \\
         \leq &~ \frac{\left ( (\betaA)^{\frac{1}{2}} + (\| \Gfun \|_{\infty})^{\frac{1}{2}} \right ) (2n)^{\frac{1}{2}}}{\xi (\lambmin(G))^{\frac{1}{2}} \mathtt{K}_{\sigma_2} \hat{a}} (\eta_t)^{-\frac{1}{2}} (\mathcal{L}_t)^{-\frac{1}{2}} \left ( - \frac{d}{dt} \mathcal{L}_t \right )~.
    \end{split}
\end{equation}
Therefore,
\begin{equation}
    \begin{split}
        \frac{d}{dt} \left ( \frac{2}{3} (\eta_t)^{\frac{3}{2}} \right ) = (\eta_t)^{\frac{1}{2}} \frac{d}{dt} \eta_t 
        \geq &~ \frac{\left ( (\betaA)^{\frac{1}{2}} + (\| \Gfun \|_{\infty})^{\frac{1}{2}} \right ) (2n)^{\frac{1}{2}}}{\xi (\lambmin(G))^{\frac{1}{2}} \mathtt{K}_{\sigma_2} \hat{a}} (\mathcal{L}_t)^{-\frac{1}{2}} \frac{d}{dt} \mathcal{L}_t \\
        =&~ \frac{\left ( (\betaA)^{\frac{1}{2}} + (\| \Gfun \|_{\infty})^{\frac{1}{2}} \right ) (2n)^{\frac{1}{2}}}{\xi (\lambmin(G))^{\frac{1}{2}} \mathtt{K}_{\sigma_2} \hat{a}} \frac{d}{dt} \left ( 2 (\mathcal{L}_t)^{\frac{1}{2}} \right )~,
    \end{split}
\end{equation}
which implies that
\begin{equation}
\begin{split}
    \frac{2}{3} (\eta_t)^{\frac{2}{3}} \geq &~ \frac{2}{3} (\eta^0)^{\frac{2}{3}} + \frac{2\sqrt{2} \left ( (\betaA)^{\frac{1}{2}} + (\| \Gfun \|_{\infty})^{\frac{1}{2}} \right ) n^{\frac{1}{2}}}{\xi (\lambmin(G))^{\frac{1}{2}} \mathtt{K}_{\sigma_2} \hat{a}} \left ( (\mathcal{L}_t)^{\frac{1}{2}} - (\mathcal{L}_0)^{\frac{1}{2}} \right ) \\
    \geq &~ \frac{2}{3} (\min_{k \in [n]} \mut{0}(\Xi_k))^{\frac{2}{3}} - \frac{2\sqrt{2} \left ( (\betaA)^{\frac{1}{2}} + (\| \Gfun \|_{\infty})^{\frac{1}{2}} \right ) n^{\frac{1}{2}}}{\xi (\lambmin(G))^{\frac{1}{2}} \mathtt{K}_{\sigma_2} \hat{a}}(\mathcal{L}_0)^{\frac{1}{2}}~,
\end{split}
\end{equation}
and hence
\begin{equation}
\begin{split}
    \min_{k \in [n]} \mut{t}(\Xi_k) \geq \eta_t 
    \geq \left ( (\min_{k \in [n]} \mut{0}(\Xi^{\dag}_k))^{\frac{2}{3}} - \frac{C}{ \hat{a}} \right )^{\frac{3}{2}} ~,
\end{split}
\end{equation}
where we define
\begin{equation}
    C = \frac{3 \sqrt{2} \left ( (\betaA)^{\frac{1}{2}} + (\| \Gfun \|_{\infty})^{\frac{1}{2}} \right ) n^{\frac{1}{2}}}{\xi (\lambmin(G))^{\frac{1}{2}} \mathtt{K}_{\sigma_2}} (\mathcal{L}_0)^{\frac{1}{2}} = \frac{3 \left ( (\betaA)^{\frac{1}{2}} + (\| \Gfun \|_{\infty})^{\frac{1}{2}} \right ) \| \yb \|_2}{\xi (\lambmin(G))^{\frac{1}{2}} \mathtt{K}_{\sigma_2}}~.
\end{equation}

\section{Proof of Lemma~\ref{lem:AH_movement_bound}}
\label{app:pf_lem_AH_movement_bound}
We will prove an extension of Lemma~\ref{lem:AH_movement_bound} to the case of $\betaA > 0$, where the only change is to replace \eqref{eq:AH_movement_bound_2} by
\begin{equation}
\label{eq:AH_movement_bound_2_betaA}
    \sup_{(a, h) \in \supp(\mu_0)}\left \| \Ht{t}(a, h) - h\right \|_{\Hilb} \leq \sqrt{2} (\| \Gfun \|_{\infty})^{\frac{1}{2}} \mathtt{L}_{\sigma_2} \int_0^t \left ( a_{\max} + \sqrt{2} \betaA \mathtt{M}_{\sigma_2} \int_0^s (\Lossb_r)^{\frac{1}{2}} dr \right ) \big (\mathcal{L}_s \big )^{\frac{1}{2}} ds~.
\end{equation}
Note that $\| \Gfun \|_{\infty} < \infty$ by the assumptions on $\sigma_1$ and $\rhoz$ and the compactness of $\domX$.

We first consider \eqref{eq:AH_movement_bound_1}. From the results in \citet{bach2017equivalence} on the duality between integral transforms and RKHS, it follows from \eqref{eq:ddtH_g} that 
\begin{equation}
\label{eq:ddH_duality_norm}
    \left \| \tfrac{d}{dt} \Htah{t} \right \|_{\Hilb}^2 = \int_{\zb} |g_t(\zb; a, h)|^2 \rhoz(d \zb)~.
\end{equation}
Thus,
\begin{equation}
\begin{split}
    &~ \int_{\Rbb \times \Con} \| \Htah{t} - h \|_{\Hilb} \mut{0}(da, dh) \\
    \leq &~ \int_{\Rbb \times \Con} \int_0^t \left \| \tfrac{d}{ds} \Htah{s} \right \|_{\Hilb} ds \mut{0}(da, dh) \\
    \leq&~ \int_0^t \int_{\Rbb \times \Con} \left ( \int_{\zb} |g_t(\zb; a, h)|^2 \rhoz(d \zb) \right )^{\frac{1}{2}} \mut{0}(da, dh) ds \\
    \leq &~ \int_0^t \left ( \int_{\Rbb \times \Con} \int_{\zb} |g_t(\zb; a, h)|^2 \rhoz(d \zb) \mut{0}(da, dh) \right )^{\frac{1}{2}} ds~,
\end{split}
\end{equation}
and then \eqref{eq:AH_movement_bound_1} follows from \eqref{eq:g}.

To obtain an ``$L^{\infty}$-type'' bound for the second part of the lemma, we start from \eqref{eq:gt_L2} and see that
\begin{equation}
    \begin{split}
        \int_{\zb} |g_t(\zb; a, h)|^2 \rhoz(d \zb) \leq &~ \frac{1}{n^2} \left | \Atah{t} \right |^2 \summ{k, l}{n} (\mathtt{L}_{\sigma_2})^2 (f(\xbtrk{k}) - y_k) (f(\xbtrk{l}) - y_l) \Gmat_{k, l} \\
        \leq &~ (a_{\max, t})^2 (\mathtt{L}_{\sigma_2})^2 \| \Gfun \|_{\infty} \frac{1}{n^2} \summ{k, l}{n}(f(\xbtrk{k}) - y_k) (f(\xbtrk{l}) - y_l) \\
        \leq &~ (a_{\max, t})^2 (\mathtt{L}_{\sigma_2})^2 \| \Gfun \|_{\infty} \cdot 2 \Lossb_t~,
    \end{split}
\end{equation}
where we write $a_{\max, t} \coloneqq \text{ess} \sup_{(a, h) \in \supp(\mut{0})} \Atah{t}$. Therefore, from \eqref{eq:ddH_duality_norm} we derive that
\begin{equation}
    \begin{split}
        \| \Htah{t} - h \|_{\Hilb} \leq&~ \int_0^t \| \tfrac{d}{ds} \Htah{s} - h \|_{\Hilb} ds \\
        =&~ \int_0^t \left ( \int_{\zb} |g_t(\zb; a, h)|^2 \rhoz(d \zb) \right )^{\frac{1}{2}} ds \\
        \leq &~ \sqrt{2} \mathtt{L}_{\sigma_2} (\| \Gfun \|_{\infty})^{\frac{1}{2}} \int_0^t a_{\max, s} (\Lossb_s)^{\frac{1}{2}} ds~,
    \end{split}
\end{equation}
and hence it only remains to bound $a_{\max, t}$. From \eqref{eq:dotAt_Con}, we have that
\begin{equation}
    \begin{split}
        \left | \tfrac{d}{dt} \Atah{t} \right | \leq &~ \frac{\betaA}{n} \mathtt{M}_{\sigma_2} \summ{k}{n} |f_t(\xbtrk{k}) - y_k| \\
        \leq &~ \betaA \mathtt{M}_{\sigma_2} \bigg ( \frac{1}{n} \summ{k}{n} |f_t(\xbtrk{k}) - y_k|^2 \bigg )^{\frac{1}{2}} \\
        =&~ \betaA (2 \Lossb_t)^{\frac{1}{2}} \mathtt{M}_{\sigma_2}~.
    \end{split}
\end{equation}
Therefore, we have
\begin{equation}
\label{eq:A_movement_bound}
\begin{split}
    \left | \Atah{t} - a \right | \leq &~ \int_0^t \left | \tfrac{d}{ds} \Atah{s} \right | ds \\
    \leq &~ \sqrt{2} \betaA \mathtt{M}_{\sigma_2} \int_0^t (\Lossb_s)^{\frac{1}{2}} ds~.
\end{split}
\end{equation}
from which $a_{\max, t}$ can be bounded and hence \eqref{eq:AH_movement_bound_2_betaA} is derived.

\section{Proof of Lemma~\ref{lem:rad_1}}
\label{app:pf_rad_1}
Using ``$\sup_{\mu}$'' as a shorthand for taking the supremum over all $\mu \in \mathcal{P}(\mathbb{R} \times \Fbase)$ such that $\int_{\mathbb{R} \times \Fbase} |a| \| h \|_{\Fbase} \mu(da, dh) \leq c$, we have
\begin{equation}
    \begin{split}
        \widehat{\Rad}_S(\Fa(\Fbase, c)) =&~ \frac{1}{n} \EE_{\tau} \left [ \sup_{\mu} \sum_{k=1}^n \tau_k \int_{\mathbb{R} \times \Fbase} a \sigma_2(h(\xb_k)) \mu(da, dh) \right ] \\
        =&~ \frac{1}{n} \EE_{\tau} \left [ \sup_{\mu} \int_{\mathbb{R} \times \Fbase} \sum_{k=1}^n \tau_k  \frac{a}{|a|} \frac{\sigma_2(h(\xb_k))}{\| h \|_{\Fbase}} a \| h \|_{\Fbase} \mu(da, dh) \right ] \\
        \leq &~ \frac{c }{n} \EE_{\tau} \left [ \sup_{a \in \mathbb{R}, h \in \Fbase} \sum_{k=1}^n \tau_k  \frac{a}{|a|} \frac{\sigma_2(h(\xb_k))}{\| h \|_{\Fbase}} \right ] \\
        \leq &~ \frac{c }{n} \EE_{\tau} \left [ \left | \sup_{\hat{h} \in \ball{\Fbase}{1}} \sum_{k=1}^n \tau_k  \sigma_2(\hat{h}(\xb_k)) \right | \right ]~,
    \end{split}
\end{equation}
where for the last line, we use the $1$-homogeneity of $\sigma$, which implies that for any $h \in \Fbase \setminus \{0\}$, $h / \| h \|_{\Fbase}$ belongs to $\ball{\Fbase}{1}$ and satisfies $\forall \xb \in \mathcal{X}$, $(h / \| h \|_{\Fbase})(\xb) = h(\xb) / \| h \|_{\Fbase}$.

Moreover, the $1$-homogeneity of $\sigma$ also implies that $\sigma_2(0) = 0$. Thus, since $0 \in \ball{\Fbase}{1}$, we have $\sup_{\hat{h} \in \ball{\Fbase}{1}} \sum_{k=1}^n \tau_k  \sigma_2(\hat{h}(\xb_k)) = \big |\sup_{\hat{h} \in \ball{\Fbase}{1}} \sum_{k=1}^n \tau_k  \sigma_2(\hat{h}(\xb_k)) \big | \geq 0$. Therefore,
\begin{equation}
    \begin{split}
        \widehat{\Rad}_S(\Fa(\Fbase, c)) =&~ \frac{c }{n} \EE_{\tau} \left [ \sup_{\hat{h} \in \ball{\Fbase}{1}} \sum_{k=1}^n \tau_k  \sigma_2(\hat{h}(\xb_k)) \right ] \\
        \leq &~ \frac{\mathtt{L}_{\sigma} c}{n} \EE_{\tau} \left [ \sup_{\hat{h} \in \ball{\Fbase}{1}} \sum_{k=1}^n \tau_k \hat{h}(\xb_k) \right ] \\
        =&~ \mathtt{L}_{\sigma} c \widehat{\Rad}_S(\ball{\Fbase}{1})~,
    \end{split}
\end{equation}
where for the second line, we use Lemma~\ref{lem:talagrand} with $\Phi_k(u) = \sigma_2(u)$, $\forall k \in [n]$.
\begin{lemma}[Ledoux-Talagrand contraction lemma]
\label{lem:talagrand}
Suppose $\mathcal{F}$ is any function class and for each $k \in [n]$, $\Phi_k$ is an $L$-Lipschitz function. Then
\begin{equation}
    \frac{1}{n} \EE_{\tau} \left [ \sup_{h \in \mathcal{F}} \sum_{k=1}^n \tau_k (\Phi_{k} \circ h)(\xb_k) \right ] \leq \frac{L}{n} \EE_{\tau} \left [ \sup_{h \in \mathcal{F}} \sum_{k=1}^n \tau_k h (\xb_k) \right ]~.
\end{equation}
\end{lemma}
A proof can be found in \citet{mohri2018foundations}, while a similar result appears in \citet{ledoux1991probability}. 

Thus, 
\begin{equation}
\begin{split}
    \Rad_n(\Fa(\Fbase, c)) = \EE_{S \sim \mathcal{D}^n} \left [ \widehat{\Rad}_S(\Fa(\Fbase, c)) \right ] \leq &~ \mathtt{L}_{\sigma} c \EE_{S \sim \mathcal{D}^n} \left [ \widehat{\Rad}_S(\ball{\Fbase}{1}) \right ] \\
    =&~ \mathtt{L}_{\sigma} c \Rad_n(\ball{\Fbase}{1})~.
\end{split}
\end{equation}

\section{Wasserstein-type metric with $p \in [1, \infty)$}
\label{app:Wp}
For $p \in [1, \infty)$, we can define
\begin{equation}
\label{eq:Wp}
    \mathcal{W}_p(\mu, \mu'; \Fbase, \mathcal{V}) = \left ( \inf_{\pi \in \joint(\mu, \mu')} |a| \| h - h' \|_{\mathcal{V}} \pi(da, dh, dh') \right )^{\frac{1}{p}}~,
\end{equation}
in place of \eqref{eq:Winfty}, and 
\begin{equation}
\label{eq:varn_dag_p}
    \varnp{p}{\Fbase}{\sigma}{\mubase}(f) :=~ \inf_{\mu} \mathcal{W}_{p}(\mu, \mubase; \Con, \rkhsh)~,
\end{equation}
in place of \eqref{eq:varn_dag}.
Then, for any $c \geq 0$, we can use $\FbpUsigc{p}{\Fbase}{\sigma}{\mubase}{c}$ to denote the space of all functions $f$ on $\domX$ such that $\varnp{p}{\Fbase}{\sigma}{\mubase}(f) \leq c$, and further define $\FbpUsig{p}{\Fbase}{\sigma}{\mubase} = \cup_{c > 0} \FbpUsigc{p}{\Fbase}{\sigma}{\mubase}{c}$. 

It is clear that for $1 \leq p \leq p' \leq \infty$, there is $\varnp{p}{\Fbase}{\sigma}{\mubase}(f) \leq \varnp{p'}{\Fbase}{\sigma}{\mubase}(f)$ for any function $f$. 

\section{Proof of Lemma~\ref{lem:rad_1/2}}
\label{app:pf_rad_1/2}
We will state and prove a more general version of Lemma~\ref{lem:rad_1/2} that is also applicable when $\betaA > 0$. First, we extend the definition of the norm $ \varninfty{\Fbase}{\sigma}{\mubase}$ to include the case $\betaA > 0$ as follows. For a Banach space $\Fbase$,
we define the following norm
on $\mathbb{R} \times \Fbase$: 
\begin{equation}
\label{eq:prod_norm}
    \| (a, h) \|_{\Fbase} = \max \{ C_{\betaA} |a|, \| h \|_{\Fbase} \}~,
\end{equation}
where $C_{\betaA} \in [0, \infty]$ is a constant to be specified that depends on $\betaA$.
This norm induces a metric on $\mathbb{R} \times \Fbase$: $\forall a_1, a_2 \in \mathbb{R}$ and $\forall h_1, h_2 \in \Fbase$, 
\begin{equation}
    d_{\Fbase}((a_1, h_1), (a_2, h_2)) = \| (a_1-a_2, h_1-h_2)\|_{\Fbase}~.
\end{equation}
Let $\Fbase$ and $\mathcal{V}$ be two Banach spaces  with norms $\| \cdot \|_{\Fbase}$ and $\| \cdot \|_{\mathcal{V}}$ such that $\Fbase \subseteq \mathcal{V}$.
Let $\mu, \mu'$ be two probability measures on $\mathbb{R} \times \Fbase$, and let $\tilde{\joint}(\mu, \mu')$ denote the space of probability measures on $(\mathbb{R} \times \mathcal{V}) \times (\mathbb{R} \times \mathcal{V})$ with marginals equal to $\mu$ and $\mu'$, respectively. 
For $p \in [1, \infty)$, we define
\begin{equation}
\label{eq:Wp_betaA}
    \mathcal{W}_p(\mu, \mu'; \Fbase, \mathcal{V}) = \left ( \inf_{\pi \in \tilde{\joint}(\mu, \mu')} d_{\Fbase}((a_1, h_1), (a_2, h_2))^p \pi(da_1, dh_1, da_2, dh_2) \right )^{\frac{1}{p}}~,
\end{equation}
and,
\begin{equation}
\label{eq:Winfty_betaA}
    \mathcal{W}_{\infty}(\mu, \mu'; \Fbase, \mathcal{V}) = \inf_{\pi \in \tilde{\joint}(\mu, \mu')} \hspace{10pt} \text{ess} \hspace{-22pt} \sup_{\pi(da, dh, da', dh') \hspace{15pt}} d_{\Fbase}((a_1, h_1), (a_2, h_2))~.
\end{equation}
When $\betaA = 0$, we set $C_{\betaA} = \infty$. Thus, under the convention ``$0 \cdot \infty = 0$'', we see that \eqref{eq:Wp_betaA} and \eqref{eq:Winfty_betaA} are equivalent to the definitions \eqref{eq:Wp} and \eqref{eq:Winfty}. We then also define $\varninfty{\Fbase}{\sigma}{\mubase}$ and $\varnp{p}{\Fbase}{\sigma}{\mubase}$ through \eqref{eq:varn_dag} and \eqref{eq:varn_dag_p}, as well as $\FbUsigc{\Fbase}{\sigma}{\mubase}{c}$ and $\FbpUsigc{p}{\Fbase}{\sigma}{\mubase}{c}$ in the same way as before.

Under the generalized definitions, we state the following lemma, which extends Lemma~\ref{lem:rad_1/2}:
\begin{lemma}
\label{lem:rad_1/2_betaA}
Assume that $\sigma$ is $\mathtt{L}_{\sigma}$-Lipschitz and $\int_{\mathbb{R} \times \Con} |a| \mubase(da, dh) = \bar{a} < \infty$. 
If $\betaA > 0$, we further assume that $|\sigma(u)| < \mathtt{M}_{\sigma_2}$, $\forall u \in \Rbb$. Then it holds that,
\begin{equation}
    \widehat{\Rad}_S(\FbUsigc{\Fbase}{\sigma}{\mubase}{c}) \leq \mathtt{L}_{\sigma} \left ( \bar{a} + \frac{c}{C_{\betaA}} \right )  \widehat{\Rad}_S(\ball{\Fbase}{c}) + \frac{\mathtt{M}_{\sigma} c}{\sqrt{n} C_{\betaA}}~.
\end{equation}
\end{lemma}
\begin{proof}
Given any $f \in \FbUsigc{\rkhsh}{\sigma}{\mubase}{c}$, let $\mu$ denote its corresponding measure. Define the function $f_{\text{base}}(\xb) = \int_{\Rbb \times \Con} a \sigbig{h(\xb)} \mubase(da, dh)$ on $\domX$. Since $\mathcal{W}_{\infty}(\mubase, \mu; \Con, \rkhsh) \leq c$, $\exists \pi \in \tilde{\joint}(\mubase, \mu)$ such that 
almost surely with respect to $\pi(da_1, dh_1, da_2, dh_2)$,
\begin{equation}
\label{eq:d_on_supp}
    d((a_1, h_1), (a_2, h_2) \leq c~.
\end{equation}
We then see that
\begin{equation}
    \begin{split}
        f(\xb) =& \int_{\mathbb{R} \times \Con} a_{\star} \sigbig {h_{\star}(\xb)} \mu(da_{\star}, dh_{\star}) \\
        =& \int_{\mathbb{R} \times \Con \times \mathbb{R} \times \Con} a_{\star} \sigbig { h_{\star}(\xb)} \pi(da, dh, da_{\star}, dh_{\star}) \\
        = & \int_{\mathbb{R} \times \Con \times \mathbb{R} \times \Con} (a+\tilde{a}) \sigbig {h(\xb) + \tilde{h}(\xb)} \tilde{\pi}(da, dh, d \tilde{a}, d \tilde{h})~,
    \end{split}
\end{equation}
where $\tilde{\pi}$ is the push-forward of $\pi$ under the map $(a, h, a', h') \mapsto (a, h, a'-a, h'-h)$. Let $\xi(d\tilde{a}, d\tilde{h}; a, h)$ denote the Radon-Nikodym derivative of $\tilde{\pi}$ with respect to $\mubase$ (or in other words, the conditional probability measure of $\tilde{a}$ and $\tilde{h}$ with respect to $a$ and $h$). Then, \eqref{eq:d_on_supp} implies that $\mubase(da, dh)$-almost surely, 
$\xi(\cdot, \cdot; a, h)$ has probability mass $0$ outside of $\ball{\Rbb \times \Fbase}{c} $.
Thus,
\begin{equation}
\begin{split}
    f(\xb) =& \int_{\mathbb{R} \times \Con} \int_{\mathbb{R} \times \Con}  (a+\tilde{a}) \sigbig { h(\xb) + \tilde{h}(\xb)} \xi(d\tilde{a}, d\tilde{h}; a, h) \mubase(da, dh) \\
    =& \int_{\mathbb{R} \times \Con} \int_{\ball{\Rbb \times \Fbase}{c}} (a+\tilde{a}) \sigbig { h(\xb) + \tilde{h}(\xb)} \xi(d\tilde{a}, d\tilde{h}; a, h) \mubase(da, dh) ~,
\end{split}
\end{equation}
and
\begin{equation}
\begin{split}
    &~ f(\xb) - f_{\text{base}}(\xb) \\
    =&~ \int_{\mathbb{R} \times \Con} \Big ( \int_{\ball{\Rbb \times \Fbase}{c}} (a+\tilde{a}) \sigma_2(h(\xb) + \tilde{h}(\xb)) \xi(d\tilde{a}, d\tilde{h}; a, h) - a \sigbig{h(\xb)} \Big ) \mubase(da, dh) \\
    =&~ \int_{\mathbb{R} \times \Con} \int_{\ball{\Rbb \times \Fbase}{c}} \Big ((a+\tilde{a}) \sigma_2(h(\xb) + \tilde{h}(\xb)) - a \sigbig{h(\xb)} \Big ) \xi(d\tilde{a}, d\tilde{h}; a, h)  \mubase(da, dh) ~.
\end{split}
\end{equation}
Given $S = \{ \xb_1, ..., \xb_n \} \subseteq \mathcal{X}$, the empirical Rademacher complexity of $\FbUsigc{\rkhsh}{\sigma}{\mubase}{c}$ is
\begin{equation}
    \begin{split}
        & \widehat{\Rad}_S(\FbUsigc{\rkhsh}{\sigma}{\mubase}{c}) \\
        =& \frac{1}{n} \EE_{\tau} \left [ \sup_{f \in \FbUsigc{\rkhsh}{\sigma}{\mubase}{c}} \sum_{k=1}^n \tau_k f(\xb_k) \right ] \\
        =& \frac{1}{n} \EE_{\tau} \left [ \sup_{f \in \FbUsigc{\rkhsh}{\sigma}{\mubase}{c}} \sum_{k=1}^n \tau_k \Big ( f(\xb_k) - f_{\text{base}}(\xbtrk{k}) \Big ) \right ] \\
        =& \frac{1}{n} \EE_{\tau} \left [ \sup_{\xi } \int_{\mathbb{R} \times \Con} \int_{\mathbb{R} \times \Fbase} \sum_{k=1}^n \tau_k \Big ((a+\tilde{a}) \sigbig{h(\xbtrk{k}) + \tilde{h}(\xbtrk{k})} - a \sigbig{h(\xbtrk{k})} \Big ) \xi(d\tilde{a}, d\tilde{h}; a, h) \mubase(da, dh)  \right ] \\
        =& \frac{1}{n} \EE_{\tau} \left [ \int_{\mathbb{R} \times \Con} \sup_{\xi(\cdot, \cdot; a, h)} \left ( \int_{\mathbb{R} \times \Fbase} \sum_{k=1}^n \tau_k \Big ((a+\tilde{a}) \sigbig{h(\xbtrk{k}) + \tilde{h}(\xbtrk{k})} - a \sigbig{h(\xbtrk{k})} \Big ) \xi(d\tilde{a}, d\tilde{h}; a, h) \right )\mubase(da, dh)  \right ] \\
        = & \int_{\mathbb{R} \times \Con}\frac{1}{n} \EE_{\tau} \left [ \sup_{\xi(\cdot, \cdot; a, h)} \int_{\mathbb{R} \times \Fbase} \sum_{k=1}^n \tau_k \Big ((a+\tilde{a}) \sigbig{h(\xbtrk{k}) + \tilde{h}(\xbtrk{k})} - a \sigbig{h(\xbtrk{k})} \Big ) \xi(d\tilde{a}, d\tilde{h}; a, h) \right ] \mubase(da, dh) \\
        \leq & \int_{\mathbb{R} \times \Con}\frac{1}{n} \EE_{\tau} \left [ \sup_{(\tilde{a}, \tilde{h}) \in \ball{\Rbb \times \Fbase}{c}} \sum_{k=1}^n \tau_k \Big ((a+\tilde{a}) \sigbig{h(\xbtrk{k}) + \tilde{h}(\xbtrk{k})} - a \sigbig{h(\xbtrk{k})} \Big ) \right ] \mubase(da, dh) ~,
    \end{split}
\end{equation}
where in lines $4$ - $6$, the supremum is taken over all $\xi$ such that $\mubase(da, dh)$-almost surely, $\xi(\cdot, \cdot~; a, h)$ has probability mass $0$ outside of $\mathcal{P}(\ball{\Rbb \times \Fbase}{c} )$.
For any $a \in \mathbb{R}$ and $h \in \Con$, we see that
\begin{equation}
\label{eq:3terms}
    \begin{split}
        & \frac{1}{n} \EE_{\tau} \left [ \sup_{(\tilde{a}, \tilde{h}) \in \ball{\Rbb \times \Fbase}{c}} \sum_{k=1}^n \tau_k \Big ((a+\tilde{a}) \sigbig{h(\xbtrk{k}) + \tilde{h}(\xbtrk{k})} - a \sigbig{h(\xbtrk{k})} \Big ) \right ] \\ 
        \leq & \frac{1}{n} \EE_{\tau} \left [ \sup_{\| \tilde{h} \|_{\Fbase} \leq c} \sum_{k=1}^n \tau_k a \Big ( \sigbig{h(\xb_k) + \tilde{h}(\xb_k)} - \sigbig{h(\xbtrk{k})} \Big )  \right ] \\
        & + \frac{1}{n} \EE_{\tau} \left [ \sup_{(\tilde{a}, \tilde{h}) \in \ball{\Rbb \times \Fbase}{c}} \sum_{k=1}^n \tau_k \tilde{a} \left (\sigbig{h(\xb_k) + \tilde{h}(\xb_k)} - \sigbig{h(\xb_k)} \right ) \right ] \\
        & + \frac{1}{n} \EE_{\tau} \left [ \sup_{|\tilde{a}| \leq c / C_{\betaA}} \sum_{k=1}^n \tau_k \tilde{a} \sigbig{h(\xb_k)} \right ]~.
    \end{split}
\end{equation}
We bound the three terms on the right-hand side separately.
For the first term,
\begin{equation}
    \begin{split}
        & \frac{1}{n} \EE_{\tau} \left [ \sup_{\| \tilde{h} \|_{\Fbase} \leq c} \sum_{k=1}^n \tau_k a \Big ( \sigma(h(\xb_k) + \tilde{h}(\xb_k)) - \sigbig{h(\xbtrk{k})} \Big ) \right ] \\
        \leq & \frac{|a|}{n} \EE_{\tau} \Bigg [ \sup_{\| \tilde{h} \|_{\Fbase} \leq c} \left | \sum_{k=1}^n \tau_k \Big ( \sigbig{h(\xb_k) + \tilde{h}(\xb_k)} - \sigma\big(h(\xbtrk{k}) \big ) \Big ) \right | \Bigg ] \\
        \leq & \frac{|a|}{n} \Bigg ( \EE_{\tau} \left [ \sup_{\| \tilde{h} \|_{\Fbase} \leq c} \sum_{k=1}^n \tau_k \Big ( \sigma  \big ( h(\xb_k) + \tilde{h}(\xb_k) \big ) - \sigbig{h(\xbtrk{k})} \Big ) \right ] \\
        & \hspace{15pt} + \EE_{\tau} \left [ \sup_{\| \tilde{h} \|_{\Fbase} \leq c} \sum_{k=1}^n (-\tau_k) \Big ( \sigbig{h(\xb_k) + \tilde{h}(\xb_k)} - \sigbig{h(\xbtrk{k})} \Big ) \right ] \Bigg )\\
        \leq & \frac{2 |a|}{n} \EE_{\tau} \left [ \sup_{\| \tilde{h} \|_{\Fbase} \leq c} \sum_{k=1}^n \tau_k \Big ( \sigbig{h(\xb_k) + \tilde{h}(\xb_k)} - \sigbig{h(\xbtrk{k})} \Big ) \right ] \\
        \leq & \frac{\mathtt{L}_{\sigma} |a|}{n} \EE_{\tau} \left [ \sup_{\| \tilde{h} \|_{\Fbase} \leq c} \sum_{k=1}^n \tau_k \tilde{h}(\xb_k) \right ] \\
        =& \mathtt{L}_{\sigma} |a| \widehat{\Rad}_S(\ball{\Fbase}{c})~,
    \end{split}
\end{equation}
where the second inequality uses the fact that $\ball{\Fbase}{c}$ contains the zero function for any $c \geq 0$, which implies that for any $\tau$, $\sup_{\| \tilde{h} \|_{\Fbase} \leq c} \sum_{k=1}^n \tau_k\left (\sigma(h(\xb_k) + \tilde{h}(\xb_k)) - \sigma(h(\xb_k)) \right )
    \geq \sum_{k=1}^n \tau_k\left (\sigma(h(\xb_k) + 0) - \sigma(h(\xb_k)) \right )
    = 0$;
the third inequality uses the symmetry of the Rademacher distribution; and the fourth inequality uses Lemma~\ref{lem:talagrand}, with each $\Phi_{k}(u)$ defined to be $\sigbig{h(\xb_k) + u} - \sigbig{h(\xbtrk{k})}$.

For the second term,
\begin{equation}
    \begin{split}
        & \frac{1}{n} \EE_{\tau} \left [ \sup_{(\tilde{a}, \tilde{h}) \in \ball{\Rbb \times \Fbase}{c}} \sum_{k=1}^n \tau_k \tilde{a} \Big ( \sigma(h(\xb_k) + \tilde{h}(\xb_k)) - \sigbig{h(\xbtrk{k})} \Big )  \right ] \\
        \leq & \frac{1}{n} \EE_{\tau} \left [ \sup_{|\tilde{a}| \leq c / C_{\betaA}} \sup_{\| \tilde{h} \|_{\Fbase} \leq c} \sum_{k=1}^n \tau_k \tilde{a} \Big ( \sigma(h(\xb_k) + \tilde{h}(\xb_k)) - \sigbig{h(\xbtrk{k})} \Big )  \right ]  \\
        \leq & \frac{c}{C_{\betaA} n} \EE_{\tau} \left [ \sup_{\| \tilde{h} \|_{\Fbase} \leq c} \left | \sum_{k=1}^n \tau_k \Big ( \sigma(h(\xb_k) + \tilde{h}(\xb_k)) - \sigbig{h(\xbtrk{k})} \Big )  \right | \right ] \\
        \leq & \frac{2 c}{C_{\betaA} n} \EE_{\tau} \left [ \sup_{\| \tilde{h} \|_{\Fbase} \leq c} \sum_{k=1}^n \tau_k \Big ( \sigma(h(\xb_k) + \tilde{h}(\xb_k)) - \sigbig{h(\xbtrk{k})} \Big )  \right ] \\
        \leq & \frac{\mathtt{L}_{\sigma} c}{C_{\betaA} n} \EE_{\tau} \left [ \sup_{\| \tilde{h} \|_{\Fbase} \leq c} \sum_{k=1}^n \tau_k \tilde{h}(\xb_k) \right ] \\
        \leq & \frac{\mathtt{L}_{\sigma} c}{C_{\betaA}} \widehat{\Rad}_S(\ball{\Fbase}{c})~,
    \end{split}
\end{equation}
where the third and fourth inequalities again use the fact that $\ball{\Fbase}{c}$ contains the zero function for any $c \geq 0$ and Lemma~\ref{lem:talagrand}, respectively.

For the third term,
\begin{equation}
    \begin{split}
        \frac{1}{n} \EE_{\tau} \left [ \sup_{|\tilde{a}| \leq c / C_{\betaA}} \sum_{k=1}^n \tau_k \tilde{a} \sigbig { h(\xb_k) } \right ] =&~ \frac{c}{n C_{\betaA}} \EE_{\tau} \left [ \left | \sum_{k=1}^n \tau_k \sigbig { h(\xb_k) } \right | \right ] \\
        \leq &~ \frac{c}{n C_{\betaA}} \left ( \EE_{\tau} \left [ \left | \sum_{k=1}^n \tau_k \sigbig { h(\xb_k) } \right |^2 \right ] \right )^{\frac{1}{2}} 
        \leq \frac{\mathtt{M}_{\sigma} c}{\sqrt{n} C_{\betaA}}~.
    \end{split}
\end{equation}

Therefore, from \eqref{eq:3terms} we deduce that
\begin{equation}
    \frac{1}{n} \EE_{\tau} \left [ \sup_{(\tilde{a}, \tilde{h}) \in \ball{\Rbb \times \Fbase}{c}} \sum_{k=1}^n \tau_k (a+\tilde{a}) \sigbig { h(\xb_k) + \tilde{h}(\xb_k) } \right ] \leq \left ( \mathtt{L}_{\sigma}|a| + \frac{\mathtt{L}_{\sigma} c}{C_{\betaA}} \right )  \widehat{\Rad}_S(\ball{\Fbase}{c}) + \frac{\mathtt{M}_{\sigma} c}{\sqrt{n} C_{\betaA}}~.
\end{equation}
Hence,
\begin{equation}
\label{eq:rad_bdd_betaA}
\begin{split}
    \widehat{\Rad}_S(\FbUsigc{\rkhsh}{\sigma}{\mubase}{c}) \leq &~ \int_{\mathbb{R} \times \Con} \left ( \mathtt{L}_{\sigma}|a| + \frac{\mathtt{L}_{\sigma} c}{C_{\betaA}} \right )  \widehat{\Rad}_S(\ball{\Fbase}{c}) + \frac{\mathtt{M}_{\sigma} c}{\sqrt{n} C_{\betaA}} ~\mubase(da, dh) \\
    \leq &~ \mathtt{L}_{\sigma} \left ( \bar{a} + \frac{c}{C_{\betaA}} \right )  \widehat{\Rad}_S(\ball{\Fbase}{c}) + \frac{\mathtt{M}_{\sigma} c}{\sqrt{n} C_{\betaA}}~.
\end{split}
\end{equation}
In particular, when $\betaA = 0$, the results above reduce to Lemma~\ref{lem:rad_1/2} and Corollary~\ref{cor:rad_1/2_hilb} (and does not require $\sigma$ to be bounded).
\end{proof}

\section{Additional Experiment Results}
\label{app:exp}
\begin{figure}[h]
    \centering
    \includegraphics[scale=0.35]{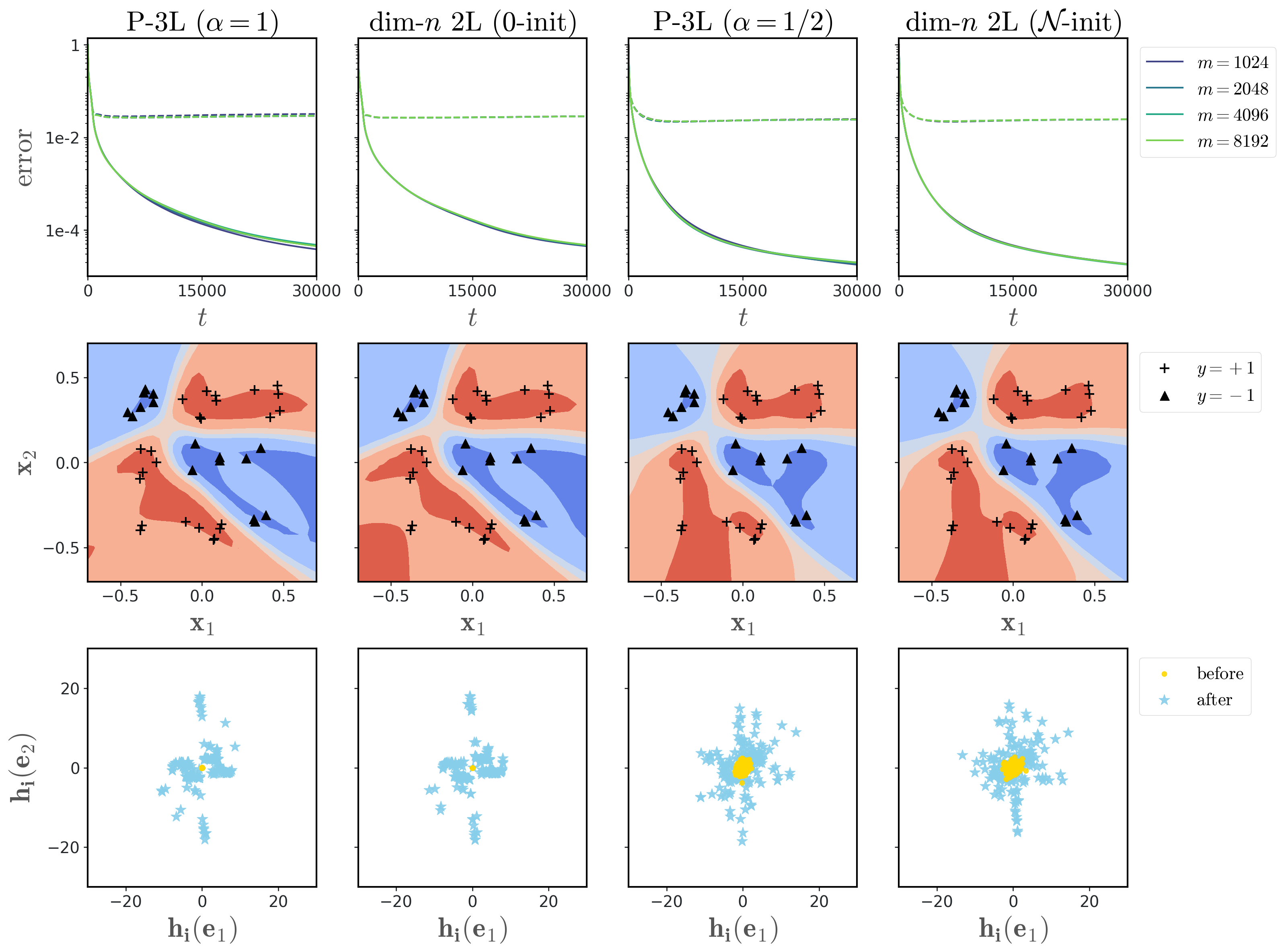}
    \caption{Comparison between \ptl NNs and their corresponding $n$-dimensional shallow NNs on Task I. The plots are defined in the same way as in Figure~\ref{fig:chizat2D}.}
    \label{fig:chizat_tanh_dim-n}
\end{figure}
\begin{figure}
    \centering
    \includegraphics[scale=0.35]{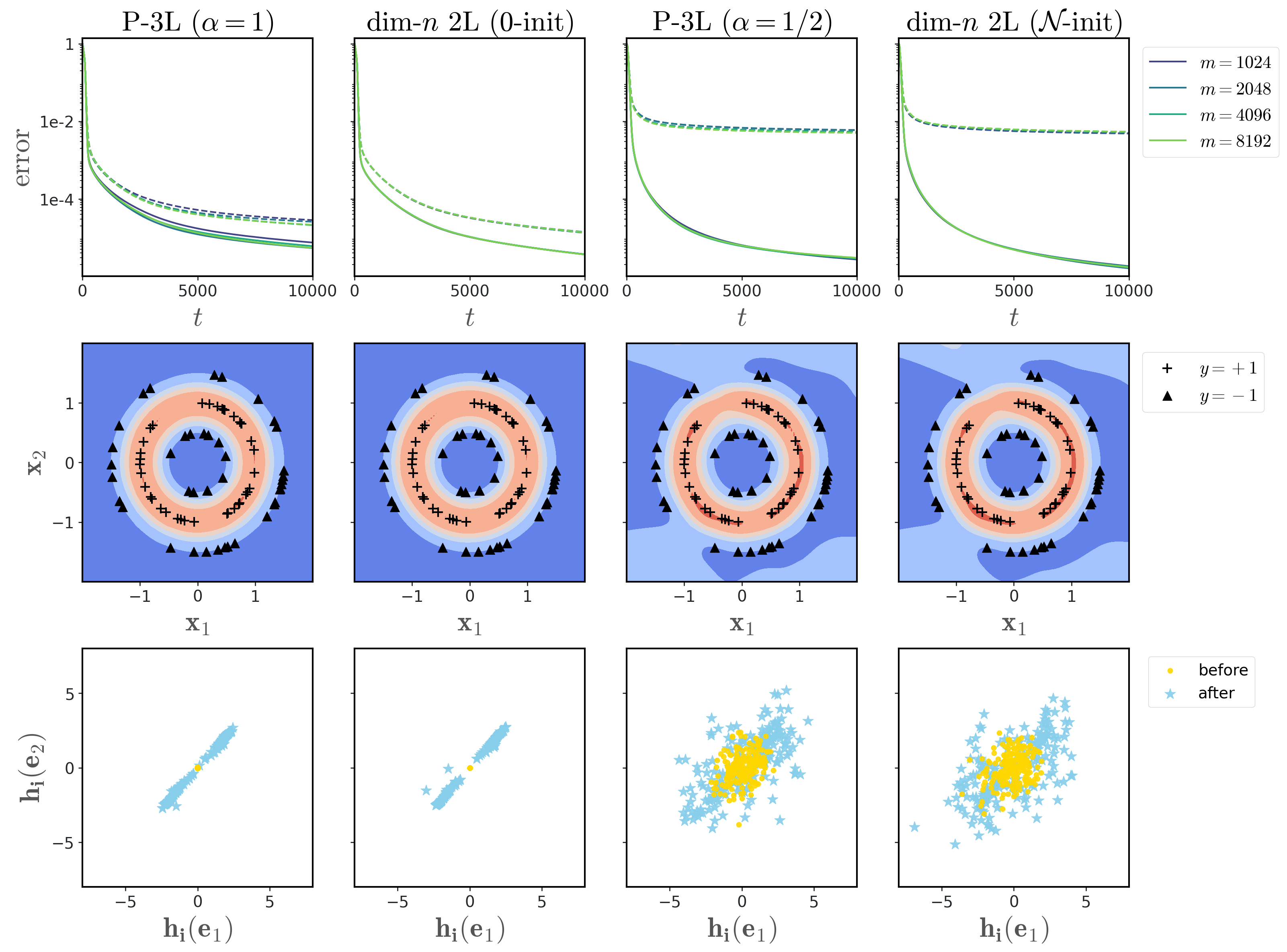}
    \caption{Comparison between \ptl NNs and their corresponding $n$-dimensional shallow NNs on Task II with $\sigma_2$ as tanh. The plots are defined in the same way as in Figure~\ref{fig:chizat2D}.}
    \label{fig:radial_tanh_dim-n}
\end{figure}
\begin{figure}
    \centering
    \includegraphics[scale=0.35]{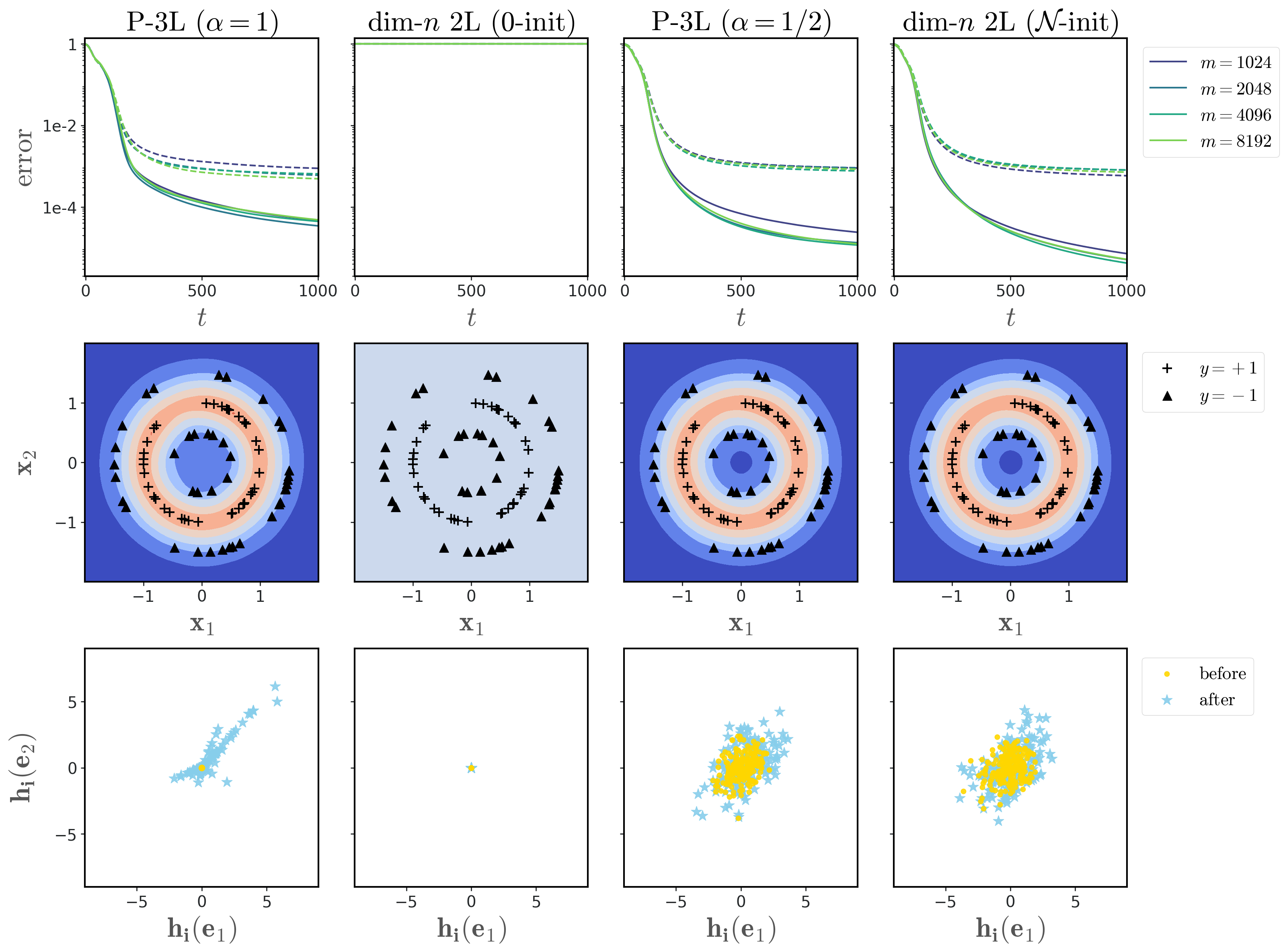}
    \caption{Comparison between \ptl NNs and their corresponding $n$-dimensional shallow NNs on Task II with $\sigma_2$ as ReLU. The plots are defined in the same way as in Figure~\ref{fig:chizat2D}.}
    \label{fig:radial_relu_dim-n}
\end{figure}
\begin{figure}
    \centering
    \includegraphics[scale=0.35]{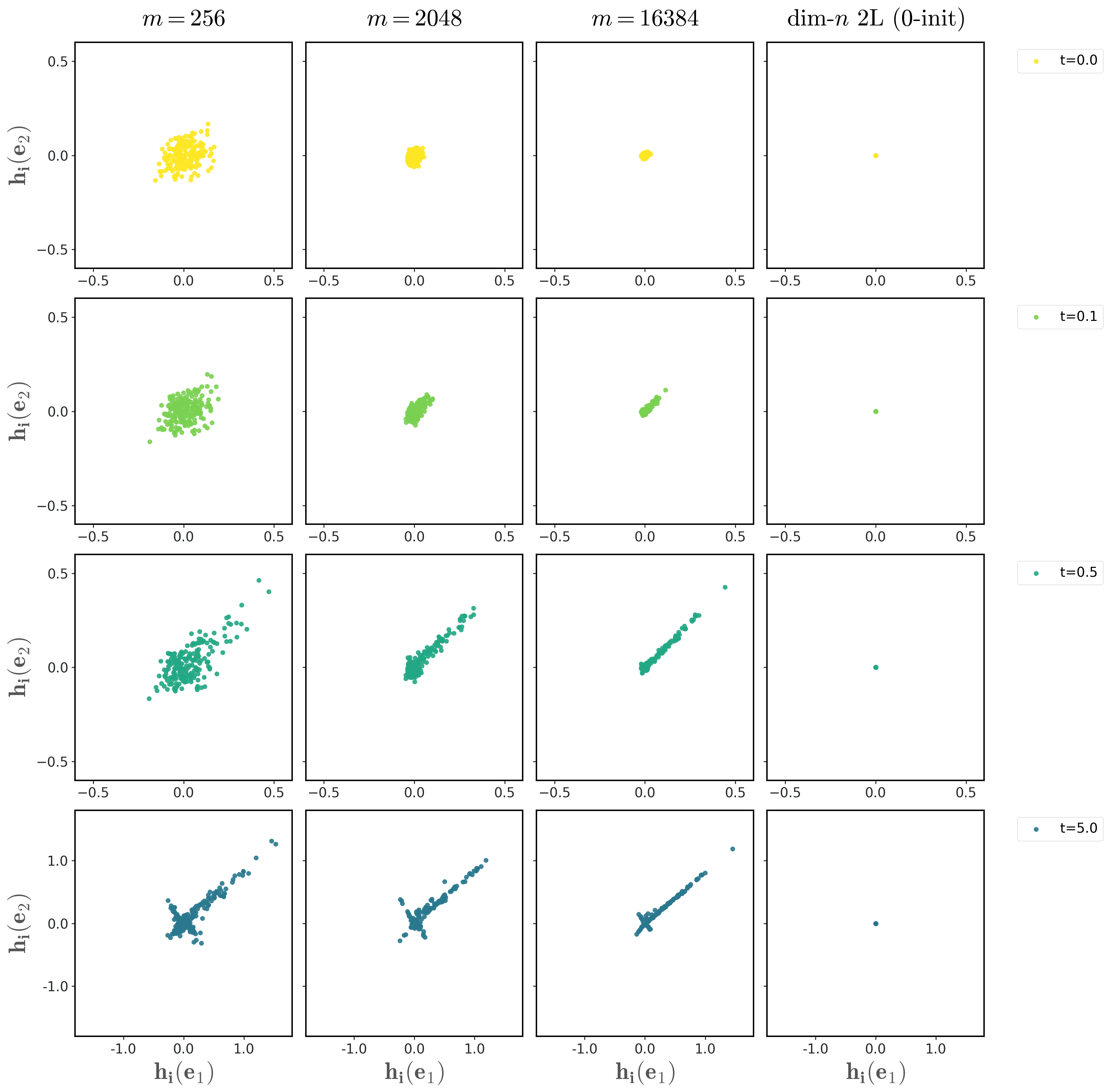}
    \caption{Comparison between \textbf{\ptbl ($\boldsymbol{\alpha = 1}$)} with various $m$ and \textbf{dim-$\boldsymbol{n}$ $\boldsymbol{2}$L ($\boldsymbol{0}$-init)} in terms of pre-activation values of second-hidden-layer neurons during early training.
    At initial time, as $m$ increases in \textbf{\ptbl ($\boldsymbol{\alpha = 1}$)}, the neurons' pre-activation values shrink in their magnitude and converge to the zero due to the LLN. But because they are not \emph{exactly} zero, gradients can be back-propagated through the ReLU activation and weights are able to evolve during training. In contrast, those in  \textbf{dim-$\boldsymbol{n}$ $\boldsymbol{2}$L ($\boldsymbol{0}$-init)} are exactly zero at initialization and therefore at all times as well due to ReLU being not differentiable at zero.}
    \label{fig:early}
\end{figure}
\begin{figure}
    \centering
    \includegraphics[scale=0.35]{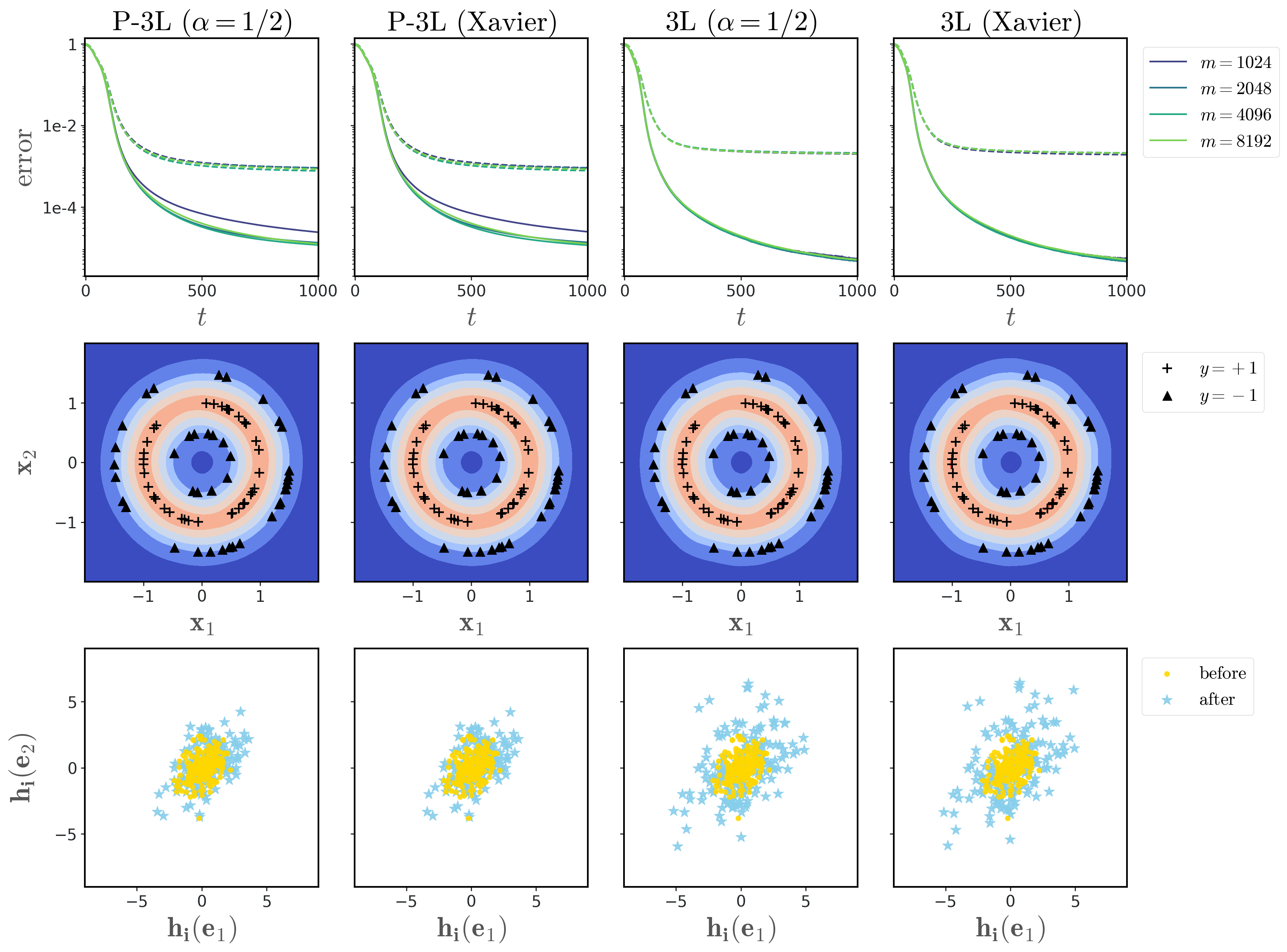}
    \caption{Comparison between $\alpha = 1 / 2$ and the Xavier-initialized standard parameterization for both \ptl and $3$-L NN on Task II with ReLU as $\sigma_2$.}
    \label{fig:xavier}
\end{figure}
\begin{figure}
    \centering
    \includegraphics[width=0.8\linewidth]{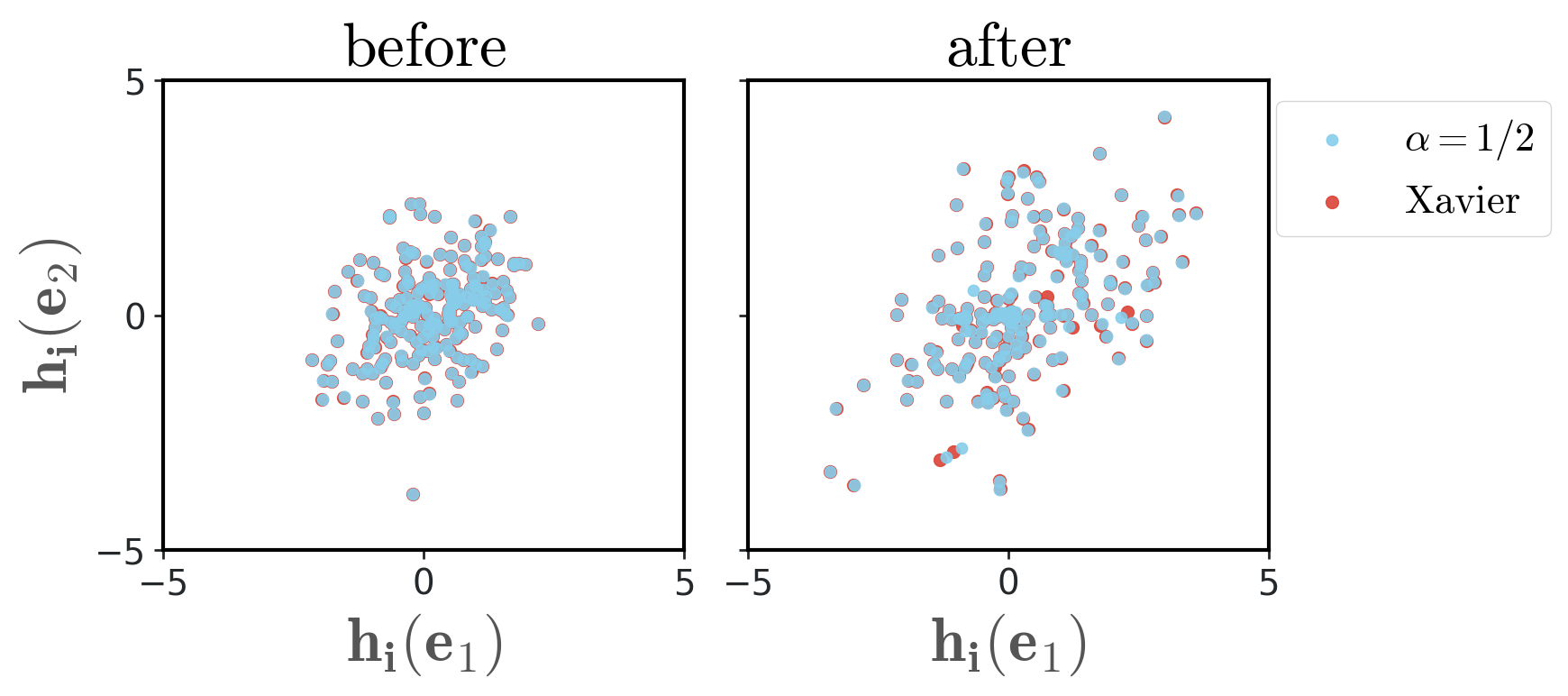}
    \caption{Finer comparison between $\alpha = 1 / 2$ and the Xavier-initialized standard parameterization for \ptl NN on Task II with ReLU as $\sigma_2$. We plot the pre-activation values of the second-hidden-layer neurons in the \ptl NN before and after training.}
    \label{fig:xavier_vs_1/2}
\end{figure}

\clearpage

\bibliography{ref}

\end{document}